\documentclass[13pt]{article}
\usepackage{latexsym}
\usepackage{geometry}
\usepackage{graphicx}
\usepackage{amsmath, amssymb, amsthm}
\usepackage{booktabs}
\usepackage{algorithm}
\usepackage{algorithmic}

\usepackage{url}
\usepackage{appendix}

\usepackage{amsfonts}
\usepackage{multirow}
\usepackage{multicol}

\usepackage{color}
\usepackage{algorithm}
\usepackage{algorithmic}

\usepackage{xcolor,colortbl}
\definecolor{LightBlue}{rgb}{0.78,1,1}

\usepackage{graphicx}
\usepackage{subfig}
\usepackage{hyperref}
\usepackage{tcolorbox}

\newtheorem{theorem}{Theorem}
\newtheorem{lemma}{Lemma}
\newtheorem{proposition}{Proposition}

\newtheorem{assumption}{Assumption}
\newtheorem{remark}{Remark}

\usepackage{microtype}
\usepackage{graphicx}
\usepackage{booktabs} 

\usepackage{hyperref}

\usepackage{amsmath}
\usepackage{amssymb}
\usepackage{mathtools}
\usepackage{amsthm}
\usepackage{multirow}
\usepackage{bbm}
\usepackage{boldline}

\begin{document}
\title{ Local Stochastic Bilevel Optimization with Momentum-Based Variance Reduction  }

\author{ Junyi Li\thanks{Department of Electrical and Computer Engineering, University of Pittsburgh,
Pittsburgh, USA. Email: junyili.ai@gmail.com}, \ Feihu Huang\thanks{Department of Electrical and Computer Engineering, University of Pittsburgh,
Pittsburgh, USA. Email: huangfeihu2018@gmail.com}, \ Heng Huang \thanks{Department of Electrical and Computer Engineering, University of Pittsburgh,
Pittsburgh, USA.  Email: henghuanghh@gmail.com}
}

\date{}
\maketitle

\begin{abstract}
Bilevel Optimization has witnessed notable progress recently with new emerging efficient algorithms and has been applied to many machine learning tasks such as data cleaning, few-shot learning, and neural architecture search. However, little attention has been paid to solve the bilevel problems under distributed setting. Federated learning (FL) is an emerging paradigm which solves machine learning tasks over distributed-located data. FL problems are challenging to solve due to the heterogeneity and communication bottleneck. However, it is unclear how these challenges will affect the convergence of Bilevel Optimization algorithms. In this paper, we study Federated Bilevel Optimization problems. Specifically, we first propose the FedBiO, a deterministic gradient-based algorithm and we show it requires $O(\epsilon^{-2})$ number of iterations to reach an $\epsilon$-stationary point. Then we propose FedBiOAcc to accelerate FedBiO with the momentum-based variance-reduction technique under the stochastic scenario. We show FedBiOAcc has complexity of $O(\epsilon^{-1.5})$. Finally, we validate our proposed algorithms via the important Fair Federated Learning task. More specifically, we define a bilevel-based group fair FL objective. Our algorithms show superior performances compared to other baselines in numerical experiments.
\end{abstract}

\section{Introduction}
Bilevel optimization problems~\cite{willoughby1979solutions, solodov2007explicit} involve two levels of problems: an outer problem and an inner problem. The two problems are entangled where the outer problem is a function of the minimizer of the inner problem. Recently, great progress has been made to solving this type of problems, especially, efficient single loop algorithms have been developed based on various gradient approximation techniques~\cite{ji2020provably, huang2021enhanced}. Bilevel optimization problems also frequently emerge in machine learning tasks, such as hyper-parameter optimization, meta learning, neural architecture search \emph{etc.} However, most existing Bilevel Optimization work focuses on the standard non-distributed setting, and how to solve the Bilevel optimization problems under distributed settings is under-explored. Federated learning is a recently promising distributed learning paradigm. In Federated Learning~\cite{mcmahan2017communication}, a set of clients jointly solve a machine learning task under the coordination of a central server. To protect user privacy and reduce communication burden, clients only exchange their models every a few epochs, but this slows down the convergence. Various algorithms~\cite{wang2019adaptive, yu2019parallel,haddadpour2019convergence, karimireddy2019scaffold, bayoumi2020tighter, xingbig} were proposed to accelerate its training. However,most of these algorithms focus on the standard single level optimization problems. \cite{xingbig} considered one type of bilevel formulation, but their algorithm needs  Hessian matrix communication every iteration, which is impractical in practice. So efficient algorithms designed for Federated Bilevel Optimization are still missing. In this work, we propose two novel algorithms for Federated Bilevel Optimization and aim to make one step forward to mitigate this gap.

In this work, we propose the FedBiO and FedBiOAcc algorithms. The FedBiO algorithm adapts the single loop gradient-based bilevel algorithm to the federated learning setting. More precisely, each client optimizes its local bilevel problem with a single loop algorithm and for every a few epochs, clients average their local states. Then we utilize the momentum-based variance reduction technique in the local updates of FedBiO, which can control the stochastic noise and accelerate the convergence. We denote the algorithm based on this idea as FedBiOAcc. We theoretically analyze the convergence of two algorithms, which involves careful balance between the distributed-related errors and bilevel-related errors. The first type of error is `consensus error'.  In Federated Learning, clients do not communicate the model state at every epoch. As a result, the local model states drift away from each other. So the gradient directions queried at these states may not represent the true descent directions. The associated errors are often named as the `consensus error'. The consensus error also exists in Federated Bilevel Optimization problems where both inner and outer variables will drift. Another type of error is the `inner variable estimation error' or `(hyper)-gradient bias'. For a single level optimization problem, we often assume access of an unbiased gradient oracle. However, this is infeasible in bilevel optimization problems due to high computational complexity. In bilevel optimization, we need to solve the inner problem exactly to get an unbiased estimation of the (hyper-)gradient. A practical way is to solve the inner problem approximately and use the biased (hyper)-gradient. In summary, it is challenging to balance these two types of errors. However, we show that our algorithms FedBiO and FedBiOAcc converge with rate $O(\epsilon^{-2})$ and $O(\epsilon^{-1.5})$ respectively.

Finally, we apply our algorithms to solve an important federated learning task: Improve group fairness in federated learning. Fairness over sensitive groups is one of the most important desiderata in developing machine learning models. However, Federated Learning by design does not learn group-fair models. Meanwhile, due to the fact that sensitive groups often spread across different clients and clients are not allowed to share data with each other. Fair algorithms developed in non-distributed setting can not be applied directly. Recently, several research works focus on the group fairness in Federated Learning: \cite{papadaki2021federating} exploited the minimax fairness notion to learn group fair models, but it required access of the global statistics of sensitive groups;~\cite{cui2021addressing} enforced the local group fairness with linear constraints, but a local fair model may not be global group fair as clients often have heterogeneous distributions. On top of these limitations, we propose a bilevel formulation to develop group fair models. More precisely, we use a small set of samples which are balanced group-wise to tune the groups weights, in other words, we find the optimal group weights such that the learned weighted model can perform well over the validation set. We solve this problem with our two proposed algorithms and validate over real-world datasets. We summarize our contribution as follows:
\begin{enumerate}
\setlength{\itemsep}{-2pt}
    \item We propose two novel federated bilevel learning algorithms, \emph{i.e.} the FedBiO and FedBiOAcc. We show the convergence of both algorithms theoretically: FedBiO has convergence rate $O(\epsilon^{-2})$ and FedBiOACC has convergence rate $O(\epsilon^{-1.5})$. 
    \item We propose a Bilevel Optimization Formulation to improve the group fairness in Federated Learning. We compare our algorithms with various baselines and experimental results show superior performance of our new algorithms.
\end{enumerate}

\noindent \textbf{Organization.} The remainder of this paper is organized as follows: In Section 2, we introduce the related works; In Section 3, we introduce preliminaries about Bilevel Optimization and Federated Learning; In Section 4, we formally define the Federated Bilevel Optimization problem we study and introduce two novel algorithms to solve it: FedBiO and FedBiOAcc; In Section 5, we provide convergence analysis for our proposed algorithms; In section 6, we apply our algorithms to solve the Group Fairness in Federated Learning Problem; In section 7, we make discussions and summarize our approaches.

\noindent \textbf{Notations} We use $\nabla$ to denote the full gradient, use $\nabla_{x}$ to denote the partial derivative for variable x, higher order derivatives follow similar rules. $||\cdot||$ represents $l_2$ norm for vectors and spectral norm for matrices. $[K]$ represents the sequence of integers from 1 to $K$.

\section{Related Works}
\label{sec:related-work}
\noindent\textbf{Bilevel Optimization.}Bilevel optimization dates back to at least 1960s when ~\cite{willoughby1979solutions} proposed a regularization method, and then followed by many research works~\cite{ferris1991finite,solodov2007explicit,yamada2011minimizing,sabach2017first}, while in machine learning community, similar ideas in the name of implicit differentiation were also used in Hyper-parameter Optimization~\cite{larsen1996design,chen1999optimal,bengio2000gradient, do2007efficient}. Early algorithms for Bilevel Optimization solves the accurate inner problem solution for each outer variable. Recently, researchers develop algorithms which solve the inner problem with a fix number of steps, and use `back-propagation through time' technique to compute the hyper-gradient~\cite{domke2012generic, maclaurin2015gradient, franceschi2017forward, pedregosa2016hyperparameter, shaban2018truncated}. 
Very Recently, it witnessed a surge of interest in using implicit differentiation to derive single loop algorithms, \emph{i.e.} solve the inner problem with one step per hyper-iteration.\cite{ghadimi2018approximation, hong2020two, ji2020provably, ji2021lower, khanduri2021near, chen2021single,yang2021provably,huang2021enhanced, li2021fully}. 
Meanwhile, there are also works utilizing other strategies like penalty methods~\cite{mehra2019penalty}, and also other formulations like the case where the inner problem has non-unique minimizers~\cite{li2020improved}. 
The bilevel optimization has been widely applied to various machine learning applications, such as Hyper-parameter optimization~\cite{lorraine2018stochastic, okuno2018hyperparameter, franceschi2018bilevel}, meta learning~\cite{zintgraf2019fast, song2019maml, soh2020meta}, neural architecture search~\cite{liu2018darts, wong2018transfer, xu2019pc}, adversarial learning~\cite{tian2020alphagan, yin2020meta, gao2020adversarialnas}, deep reinforcement learning~\cite{yang2018convergent, tschiatschek2019learner}, \emph{etc.} Please refer to the Table~2 of the survey paper by~\cite{liu2021investigating} for a more thorough review of these applications.

\noindent\textbf{Federated Learning.} 
Federated learning~\cite{mcmahan2017communication} is a promising privacy preserving learning paradigm over distributed data. In FL, a server coordinates a set of clients to learn a model with the constraint that private client data never leaves the local device. A basic algorithm for FL is the FedAvg~\cite{mcmahan2017communication} algorithm, where clients receive the current model from the server at each epoch and then update the model locally for several steps and finally upload the new model back to the server. Compared to the traditional data-center distributed learning, Federated Learning poses new challenges including data heterogeneity, privacy concerns ,high communication cost and unfairness. To deal with these challenges, some variants of FedAvg~\cite{karimireddy2019scaffold, li2019convergence, sahu2018convergence, zhao2018federated, mohri2019agnostic, li2021ditto} are proposed.
For example, \cite{li2018federated} added regularization terms over the client objective to reduce the client drift. \cite{hsu2019measuring,karimireddy2019scaffold,wang2019slowmo} used variance reduction techniques to control variates. 

Fairness in Federated Learning has also drawn more attention recently. Some researchers~\cite{mohri2019agnostic, deng2020distributionally, li2019fair, li2021ditto} focus on making models exhibit similar performance across different clients. More recently, group fairness is also studied in federated learning. One possible approach is to learn optimal group weights by formulating it as a minimax optimization problem~\cite{du2021fairness,papadaki2021federating}. Another approach is to re-weight the sensitive groups based on local or global statistics~\cite{abay2020mitigating,ezzeldin2021fairfed}, this approach often involves the transfer of sensitive information. Then a recent work~\cite{cui2021addressing} proposes FCFL which improves both client fairness and group fairness with multi-objective optimization approach. In our work, we formulate the group fairness as a bilevel optimization problem and use it as an application of our algorithms.


\section{Preliminaries}
\textbf{Bilevel Optimization} Bilevel Optimization problems are composed of two levels of entangled problems as defined in Eq.~\eqref{eq:cen-bi}:
\begin{equation}
\label{eq:cen-bi}
\begin{split}
     h(x) &\coloneqq f(x , y_{x})\ \emph{s.t.}\ y_{x} = \underset{y\in \mathbb{R}^{d}}{\arg\min}\ g(x,y)
\end{split}
\end{equation}
As shown in Eq.~\eqref{eq:cen-bi}, the outer problem ($f(x, y_x)$) depends on the solution of the inner problem ($g(x, y)$.  In machine learning, we usually consider the following stochastic formulation as shown in Eq.~\eqref{eq:cen-bi-stoc}:
\begin{equation}
\label{eq:cen-bi-stoc}
\begin{split}
     h(x) &\coloneqq \mathbb{E} [f(x , y_{x}; \mathcal{B}_f)]\ \emph{s.t.}\ y_{x} = \underset{y\in \mathbb{R}^{d}}{\arg\min}\ \mathbb{E} [g(x,y; \mathcal{B}_g)]
\end{split}
\end{equation}
where both the outer and inner problems are defined as expectation of some random variables $\mathcal{B}_f$ (outer) and $\mathcal{B}_g$ (inner). In this work, we study the non-convex-strongly-convex bilevel optimization problems of Eq.~\eqref{eq:cen-bi-stoc}. More formally, we have Assumption~\ref{assumption:function}:
\begin{assumption}
	\label{assumption:function}
	Function $f(x,y) \coloneqq \mathbb{E}[f(x , y_{x}; \mathcal{B}_f)]$ is possibly non-convex and $g(x,y) \coloneqq \mathbb{E}[g(x , y_{x}; \mathcal{B}_g)]$ is $\mu$-strongly convex \emph{w.r.t} $y$ for any given $x$.
\end{assumption}
Based on the above Assumption~\ref{assumption:function}, we have the Hessian matrix $\nabla_{y^2}^2 g(x,y_x)$ is positive definite, so we can derive the following expression for hyper-gradient, \emph{i.e.} $\nabla h(x)$:
\begin{align}
    \nabla h(x) = &\nabla_x f(x, y_x) -  \nabla_{xy}^2 g(x, y_x)\times [\nabla_{y^2}^2 g(x, y_x)]^{-1} \nabla_y f(x, y_x)
\label{eq:outer_grad}
\end{align}
For the proof of Eq~\eqref{eq:outer_grad}, we refer the readers to related work of bilevel optimization like~\cite{ghadimi2018approximation}. Various algorithms are developed in Bilevel Optimization literature to evaluate Eq.~\eqref{eq:outer_grad} approximately. Suppose we denote $\Phi(x, y)$ as:
\begin{align}
    \Phi(x,y) = &\nabla_x f(x, y) -  \nabla_{xy}^2 g(x, y)\times [\nabla_{y^2}^2 g(x, y)]^{-1} \nabla_y f(x, y)
\label{eq:outer_grad_other}
\end{align}
Note that $\Phi(x,y_x) = \nabla h(x)$. In practice, we estimate the hyper-gradient by randomly sampling from $f$ and $g$ and also uses Neumann series to estimate the Hessian Inverse ($\eta\sum_{i=0}^{\infty}(I-\eta H)^i=H^{-1}$). More precisely, suppose we have independent minibatches of samples $\mathcal{B}_x=\{\mathcal{B}_j(j=1,\ldots,Q), \mathcal{B}_f,\mathcal{B}_g\}$, then we estimate $\Phi(x,y)$ as:
\begin{align}
\label{eq:outer_grad_est}
	 & \Phi(x, y;\mathcal{B}_x) = \nabla_x f(x, y; \mathcal{B}_{f}) - \nabla_{xy}g(x, y; \mathcal{B}_{g}) \eta \sum_{q=-1}^{Q-1} \prod_{j=Q-q}^{Q} (I-\eta \nabla_{y^2}^2 g(x, y; \mathcal{B}_j))\nabla_y f(x,y;\mathcal{B}_{f}),
\end{align}
Then with the smoothness assumption about the outer function $f$/inner function $g$ and unbiased and bounded variance oracle to their gradient/second gradient , $\Phi(x, y;\mathcal{B}_x)$ has bounded variance. 
\begin{assumption}
	\label{assumption:f_smoothness}
	Function $f(x,y)$ is $L$-Lipschitz and has $M$-bounded gradient;
\end{assumption}
\begin{assumption}
	\label{assumption:g_smoothness}
	Function $g(x,y)$ is  $L$-Lipschitz. For higher-order derivatives, we have:
	\begin{itemize}
		\item[a)] $\|\nabla_{xy}^2 g(x,y)\| \le C_{g,xy}$ for some constant $C_{g,xy}$
		\item[b)] $\nabla_{xy}^2 g(x,y)$ and $\nabla_{y^2}^2 g(x,y)$ are Lipschitz continuous with constant $L_{g,xy}$ and $L_{g,y^2}$ respectively
	\end{itemize}
\end{assumption}
\begin{assumption}
	\label{assumption:noise_assumption}
	We have an unbiased stochastic oracle with bounded variance for estimating the related properties (gradient and Hessian), \emph{e.g.} $\mathbb{E}[\nabla_x f(x, y; \xi)] = \nabla_x f(x, y)$ and $var(\nabla_x f(x, y; \xi)) \le \sigma^2$
\end{assumption}
Note for Assumptions~\ref{assumption:f_smoothness} and~\ref{assumption:g_smoothness}, we assume the Lipschitz condition also holds for the stochastic query, \emph{i.e.} $f(x,y;\mathcal{B}_f)$ and $g(x,y;\mathcal{B}_g)$. Furthermore, we require stronger conditions in the above assumptions than single level optimization problems: bounded gradients (for $f$) and second order smoothness (for $g$), but these conditions are necessary to derive the smoothness of $h(x)$ and some other basic properties. Firstly,  we have the following Proposition about $\Phi(x, y;\mathcal{B}_x)$:
\begin{proposition} (Combine Lemma 4 and Lemma 7 in~\cite{yang2021provably} )
\label{prop:hg_var}
Suppose Assumptions \ref{assumption:f_smoothness}, \ref{assumption:g_smoothness} and \ref{assumption:noise_assumption} hold and $\eta < \frac{1}{L}$, the hypergradient estimator $\Phi(x;\mathcal{B}_x)$ w.r.t.\ x based on a minibatch $\mathcal{B}_x$ has bounded variance and bias:
\begin{itemize}
    \item [a)] $\mathbb{E} [\|\mathbb{E}[\Phi(x, y;\mathcal{B}_x)] - \Phi(x, y)\|^2] \leq G_1^2$, where $G_1 = (1 - \eta\mu)^{Q+1}ML/\mu$
    \item [b)] $\mathbb{E} \|\Phi(x, y;\mathcal{B}_x)- \mathbb{E}[\Phi(x, y;\mathcal{B}_x)]\|^2 \leq G_2^2$, where $G_2 = 2M^2+ 12M^2L^2\eta^2(Q+1)^2+ 4M^2L^2(Q+2)(Q+1)^2\eta^4\sigma^2$ 
\end{itemize}
\end{proposition}
Furthermore, we have the following useful propositions based on the smoothness assumptions. Cases a), b), c) are rephrased from Lemma 2.2 in~\cite{ghadimi2018approximation}. Proof for Case d) is included in Appendix A.
\begin{proposition} \label{some smoothness}
Suppose Assumptions~\ref{assumption:f_smoothness} and ~\ref{assumption:g_smoothness} hold, the following statements hold:
\begin{itemize}
\item [a)] $\|\Phi(x; y) - \nabla h(x)\| \leq C \|y_x- y\|$,
where $C=L+LC_{g,xy}/\mu + M(L_{g,xy}/\mu + L_{g, y^2}C_{g, xy}/\mu^2)$.

\item [b)] $y_x$ is Lipschitz continuous in $x$ with constant $\rho = C_{g,xy}/\mu$.

\item [c)] $h(x)$ is Lipschitz continuous in $x$ with constant $L_h$ i.e., for any given $x_1, x_2 \in X$, we have
$\|\nabla h(x_2) - \nabla h(x_1)\| \le \bar{L} \|x_2 - x_1\|$
where $\bar{L} =(L+C)C_{g,xy}\mu+L+M(L_{g, xy}M\mu + L_{g,y^2}C_{g,xy}\mu^2)$.

\item [d)] $\|\Phi(x_1; y_1) - \Phi(x_2; y_2)\|^2 \leq \Gamma^2 (\|x_1- x_2\|^2 + \|y_1- y_2\|^2)$,
where $\Gamma = L + ML_{g,xy}/\mu + C_{g,xy}(L/\mu+ ML_{g,yy}/\mu^2)$.
\end{itemize}
We denote $L_h = max(\bar{L}, \Gamma, C)$ for convenience.
\end{proposition}

\textbf{Federated Learning} A general FL problem studies the following problem:
\begin{equation}
\label{eq:fed-stoc}
\begin{split}
    \underset{x \in X }{\min}\ h(x) &\coloneqq \frac{1}{M}\sum_{m=1}^{M} \mathbb{E}_{\xi}[f^{(m)}(x; \xi)]
\end{split}
\end{equation}
Where there is one server and multiple workers. To protect user privacy, only models are transfer ed among workers and the server. A basic algorithm to solve this problem is the FedAvg~\cite{mcmahan2017communication} and we show its procedure in Algorithm~\ref{alg:fedavg}.
\begin{algorithm}[t]
\caption{FedAvg}
\label{alg:fedavg}
\begin{algorithmic}[1]
\STATE {\bfseries Input:}; Initial states $x_1$ learning rates $\{\eta_t\}$
\STATE Set $x^{(m)}_1 = x_1$;
\FOR{$t=1$ \textbf{to} $T$}
\STATE Randomly sample minibatches $\mathcal{B}_{x}$;
\STATE Compute $\nu_{t}^{(m)} = \nabla_x f^{(m)} (x^{(m)}_{t}, \mathcal{B}_{x})$
\STATE $\hat{x}^{(m)}_{t+1} = x^{(m)}_{t} - \eta_{t} \nu_{t}^{(m)}$;
\IF{$t + 1$ mod I $ = 0$}
\STATE $x^{(m)}_{t+1} = \bar{x}_{t+1} = 1/M \sum_{m=1}^{M}  \hat{x}^{(m)}_{t+1} $
\ELSE
\STATE $x^{(m)}_{t+1} = \hat{x}^{(m)}_{t+1}$;
\ENDIF
\ENDFOR
\end{algorithmic}
\end{algorithm}

\section{Federated Bilevel Optimization}
In this section, we discuss Federated Bilevel Optimization problems. Following standard FL setting, we assume there is one server and multiple clients. Specifically, the optimization problem solved by each client is a Bilevel Optimization Problem. More formally, a Federated Bilevel Optimization problem $h(x)$ has the following form:
\begin{equation}
\label{eq:fed-bi}
\begin{split}
    \underset{x \in X }{\min}\ h(x) &\coloneqq \frac{1}{M}\sum_{m=1}^{M} f^{(m)}(x , y_{x}^{(m)}) \emph{s.t.}\ y_{x}^{(m)} = \underset{y\in \mathbb{R}^{d}}{\arg\min}\ g^{(m)}(x,y)
\end{split}
\end{equation}
where $M$ is the number of clients and $f^{(m)}(x,y)$ and $g^{(m)}(x,y)$ are the upper and lower problem over the client $m$ respectively. $h(x)$ denotes the overall objective, and for ease of discussion, we also denote $h^{(m)} (x) = f^{(m)}(x , y_{x}^{(m)})$, while $\nabla h^{(m)} (x)$ denotes the gradient \emph{w.r.t} $x$. Note that it is possible that both $f^{(m)}(x,y) \neq f^{(k)}(x,y)$ and $g^{(m)} (x,y) \neq g^{(k)}(x,y)$ for $m \neq k, m,k \in [M]$. In other words, we consider the heterogeneous case. In machine learning, we often consider the stochastic case of Eq.~\eqref{eq:fed-bi} as follows:
\begin{equation}
\label{eq:fed-bi-stoc}
\begin{split}
    \underset{x \in X }{\min}\ h(x) &\coloneqq \frac{1}{M}\sum_{m=1}^{M} \mathbb{E}[f^{(m)}(x , y_{x}^{(m)}; \mathcal{B}_f)]\ \emph{s.t.}\ y_{x}^{(m)} = \underset{y\in \mathbb{R}^{d}}{\arg\min}\ \mathbb{E} [g^{(m)}(x,y; \mathcal{B}_g)]
\end{split}
\end{equation}
Federated Bilevel Optimization problems are more complicated than general Federated Learning problems. As shown in Algorithm~\ref{alg:fedavg}, in FedAvg, we perform local gradient descent and then average local models every a few iterations. For the deterministic case, each client evaluates the exact gradient $\nabla f^{(m)}$ in Eq.~\eqref{eq:fed-stoc} in each local iteration. However, in Federated Bilevel Optimization, according to Eq.~\eqref{eq:outer_grad} the hypergradient $\nabla h^{(m)}(x)$ has the following form:
\begin{align*}
    \nabla h^{(m)}(x) = &\nabla_x f^{(m)}(x, y_x) -  \nabla_{xy}^2 g^{(m)}(x, y_x)\times [\nabla_{y^2}^2 g^{(m)}(x, y_x)]^{-1} \nabla_y f^{(m)}(x, y_x)
\end{align*}
where $y_x$ is the minimizer of the lower objective which is defined in Eq~\eqref{eq:fed-bi}. Note $y_x$ is a function of the variable $x$, which means we need to solve the inner problem for each new state of $x$, \emph{i.e.} each local iteration, this is very computational expensive and is infeasible in practice. In other words, we can not evaluate the exact hypergradient, as a result, the FedAvg algorithm is not suitable for solving Federated Bilevel Optimization problems.

\begin{algorithm}[t]
\caption{FedBiO}
\label{alg:FedBiO}
\begin{algorithmic}[1]
\STATE {\bfseries Input:}; Initial states $x_1$ and $y_1$; learning rates $\{\gamma_t\}$ and $\{\eta_t\}$
\STATE Set $x^{(m)}_1 = x_1$ and $y^{(m)}_1 = y_1$;
\FOR{$t=1$ \textbf{to} $T$}
\STATE Compute $\omega_{t}^{(m)} = \nabla_y g^{(m)} (x^{(m)}_{t}, y^{(m)}_{t})$ and compute $\nu_{t}^{(m)} = \Phi^{(m)} (x^{(m)}_{t}, y^{(m)}_{t})$ with Eq.~(\ref{eq:outer_grad_other});
\STATE $y^{(m)}_{t+1} = y^{(m)}_{t} - \gamma_{t}  \omega_{t}^{(m)}$;
\STATE $\hat{x}^{(m)}_{t+1} = x^{(m)}_{t} - \eta_{t} \nu_{t}^{(m)}$;
\IF{$t + 1$ mod I $ = 0$}
\STATE $x^{(m)}_{t+1} = \bar{x}_{t+1} = 1/M\sum_{m=1}^{M} \hat{x}^{(m)}_{t+1} $
\ELSE
\STATE $x^{(m)}_{t+1} = \hat{x}^{(m)}_{t+1}$;
\ENDIF
\ENDFOR
\end{algorithmic}
\end{algorithm}

Following the recent progress in Bilevel Optimization~\cite{ji2020provably}, we know that it is not necessary to get the exact $y_x$, instead we could update the inner variable and outer variable alternatively. As a result, we propose our first algorithm named FedBiO whose procedure is shown in Algorithm~\ref{alg:FedBiO}. In the algorithm, we start from two random states $x^{(m)}_1$ and $y^{(m)}_1$. For each local iteration, we update $x^{(m)}_t$ and $y^{(m)}_t$ with gradient-like step where the gradients are defined in Line 5 of Algorithm~\ref{alg:FedBiO}. For every $I$ iterations, we average the $x$ states over clients. Note that we do not average over the $y$ state, this is due to the definition Eq~\eqref{eq:fed-bi} where $y^{(m)}_x$ only depends on the state $x$ and $g^{(m)}(x,y)$. 

The design of FedBiO (Algorithm~\ref{alg:FedBiO}) is natural in the sense that local updates and global average follows the classic idea of FedAvg, and the alternative updates of local variables (local inner variables and outer variables) is used in non-distributed Bilevel Optimization. However, these two designs together will bring extra complexity for the convergence analysis. More specifically, suppose we denote  the virtual average $\bar{x}_t = \frac{1}{M}\sum_{m=1}^M \hat{x}^{(m)}_t$, and we use $\|\nabla h(\bar{x}_t) \|^2$ as the convergence measure, there are two sources of errors. The first one is the outer variable consensus error defined as $\frac{1}{M } \sum_{m=1}^M \|\hat{x}_t^{(m)} - \bar{x}_t \|^2$, and the other one is the inner variable estimation error $\frac{1}{M} \sum_{m=1}^M \| y^{(m)}_t - y^{(m)}_{x^{(m)}_{t}} \|^2$. Note that outer variable consensus error is often seen in the analysis in FedAvg-type algorithms and is the main error due to local updates, as for the inner variable estimation error, it measures the imperfection of inner variable. In FedBiO, these two types of errors are entangled with each other. To see that, for $\bar{t}_s$ which satisfies $\bar{t}_s + 1 = s\times I $, we have:
\begin{align*}
     \|y^{(m)}_{\bar{t}_{s}} - y^{(m)}_{x^{(m)}_{\bar{t}_{s}}} \|^2 &= \|y^{(m)}_{\bar{t}_{s}} - y^{(m)}_{\bar{x}_{\bar{t}_{s}}} \|^2 \leq 2\|y^{(m)}_{\bar{t}_{s}} - y^{(m)}_{\hat{x}^{(m)}_{\bar{t}_{s}}} \|^2 + 2 \|y^{(m)}_{\hat{x}^{(m)}_{\bar{t}_{s}}} - y^{(m)}_{\bar{x}_{\bar{t}_{s}}}\|^2 \\
     & \leq 2\|y^{(m)}_{\bar{t}_{s}} - y^{(m)}_{\hat{x}^{(m)}_{\bar{t}_{s}}} \|^2 + 2\rho^2\|\hat{x}^{(m)}_{\bar{t}_{s}} - \bar{x}_{\bar{t}_{s}}\|^2
\end{align*}
The equality is because that we average the state $x^{(m)}$ at the step $\bar{t}_s$, the first inequality follows the triangle inequality and the second inequality follows Proposition~\ref{some smoothness}. The inequality shows that the inner variable estimation error can be decomposed to two parts: estimation error to $y^{(m)}_{\hat{x}^{(m)}_t}$ (denoted by local variable $\hat{x}^{(m)}_t$) and $\|y^{(m)}_{\hat{x}^{(m)}_{\bar{t}_{s}}} - y^{(m)}_{\bar{x}_{\bar{t}_{s}}}\|^2$ which is related to outer variable consensus error. The first error can be bounded following standard argument of gradient descent step (Line 5 in Algorithm~\ref{alg:FedBiO}). While for the outer variable consensus error, we have the following Lemma:
\begin{lemma}
\label{lemma:consensus_error}
With Assumption~\ref{assumption:function},~\ref{assumption:f_smoothness},~\ref{assumption:g_smoothness} hold, and for $t\in [\bar{t}_{s-1}+1, \bar{t}_s]$, we have:
\begin{align*}
&\| \hat{x}_t^{(m)}-  \bar{x}_t \|^2 \leq 2L_h^2 I \eta^2\sum_{\ell  = \bar{t}_{s-1}}^{t-1} \mathbb{E} \bigg\| y^{(m)}_\ell - y^{(m)}_{x^{(m)}_{\ell}} \bigg\|^2 + 12L_h^2 I \eta^2\sum_{\ell = \bar{t}_{s-1}}^{t-1}      \mathbb{E}\big\|\hat{x}_\ell^{(m)}  -  \bar{x}_\ell \big\|^2  +  6I^2\eta^2  \zeta^2, \nonumber\\
\end{align*}
where $\zeta$ is a constant that measures the heterogeneity among clients.
\end{lemma}
Lemma~\ref{lemma:consensus_error} is an intermediate result of the Lemma~\ref{lem: ErrorAccumulation_Iterates_FedAvg}. As shown in Lemma~\ref{lemma:consensus_error}, the last two terms are commonly seen in FedAvg type of methods (the accumulation of past consensus error and client heterogeneity), while the first term is the accumulation of inner variable estimation error. Due to the imperfect inner variable estimation, the consensus error is increased further. Although the two types of errors increase the analysis complexity with entanglement, we show in the convergence analysis section that our FedBiO converges with rate $O(\epsilon^{-2})$ by carefully balancing the two types of errors.

\begin{algorithm}[t]
\caption{FedBiOAcc}
\label{alg:FedBiOAcc}
\begin{algorithmic}[1]
\STATE {\bfseries Input:} constants $c_{\omega}$, $c_{\nu}$, $\gamma$, $\eta$, $\delta$, u, $\sigma$, initial state ($x_1$, $y_1$);
\STATE Set $y^{(m)}_{1} = y_{1}$, $x^{(m)}_{1} = x_{1}$ for $m \in [M]$
\FOR{$t=1$ \textbf{to} $T$}
\STATE Randomly sample minibatches $\mathcal{B}_{y}$ and $\mathcal{B}_{x}$
\IF{$t = 1$}
\STATE $\omega_{t}^{(m)} = \nabla_y g^{(m)} (x^{(m)}_{t}, y^{(m)}_{t}, \mathcal{B}_{y})$
\STATE $\hat{\nu}_{t}^{(m)} = \Phi^{(m)} (x^{(m)}_{t}, y^{(m)}_{t}; \mathcal{B}_{x})$
\ELSE
\STATE $\omega_{t}^{(m)} = \nabla_y g^{(m)} (x^{(m)}_{t}, y^{(m)}_{t}, \mathcal{B}_{y}) + (1 - c_{\omega}\alpha_{t-1}^2) (\omega_{t-1}^{(m)} - \nabla_y g^{(m)} (x^{(m)}_{t-1}, y^{(m)}_{t-1}, \mathcal{B}_{y}))$
\STATE $\mu_{t}^{(m)} = \Phi^{(m)} (x^{(m)}_{t}, y^{(m)}_{t}; \mathcal{B}_{x})$
\STATE $\mu_{t-1}^{(m)} = \Phi^{(m)} (x^{(m)}_{t-1}, y^{(m)}_{t-1}; \mathcal{B}_{x})$
\STATE $\hat{\nu}_{t}^{(m)} = \mu_{t}^{(m)} + (1 - c_{\nu}\alpha_{t-1}^2) (\nu_{t-1}^{(m)} - \mu_{t-1}^{(m)})$
\ENDIF
\STATE Evaluate $\alpha_t = \frac{\delta}{(u +\sigma^2\times t)^{1/3}}$
\STATE $y^{(m)}_{t+1} = y^{(m)}_{t} - \gamma\alpha_{t}  \omega_{t}^{(m)}$, $\hat{x}^{(m)}_{t+1} = x^{(m)}_{t} - \eta\alpha_{t} \hat{\nu}_{t}^{(m)}$
\IF{$t + 1$ mod I $ = 0$}
\STATE $x^{(m)}_{t} = \bar{x}_{t} = \frac{1}{M}\sum_{j=1}^{M} x^{(j)}_{t} $, $\nu^{(m)}_{t} = \bar{\nu}_{t} = \frac{1}{M}\sum_{j=1}^{M} \hat{\nu}^{(j)}_{t} $,\  $x^{(m)}_{t+1} = \bar{x}_{t+1} = \frac{1}{M}\sum_{j=1}^{M} \hat{x}^{(j)}_{t+1} $
\ELSE
\STATE  $\nu^{(m)}_{t} = \hat{\nu}^{(m)}_{t}$, $x^{(m)}_{t+1} = \hat{x}^{(m)}_{t+1}$;
\ENDIF
\ENDFOR
\end{algorithmic}
\end{algorithm}

Next we consider the Federated Stochastic Bilevel Optimization as defined in Eq.~\eqref{eq:fed-bi-stoc}. To control the stochastic noise, we apply the idea of momentum-based variance reduction~\cite{cutkosky2019momentum}. The procedure of the algorithm is summarized in Algorithm~\ref{alg:FedBiOAcc}. The main step of the algorithm is as follows:
\begin{align*}
    \omega_{t}^{(m)} &= \nabla_y g^{(m)} (x^{(m)}_{t}, y^{(m)}_{t}, \mathcal{B}_{y}) + (1 - c_{\omega}\alpha_{t-1}^2) (\omega_{t-1}^{(m)} - \nabla_y g^{(m)} (x^{(m)}_{t-1}, y^{(m)}_{t-1}, \mathcal{B}_{y})) \nonumber\\
    \hat{\nu}_{t}^{(m)} & = \Phi^{(m)} (x^{(m)}_{t}, y^{(m)}_{t}; \mathcal{B}_{x}) + (1 - c_{\nu}\alpha_{t-1}^2) (\nu_{t-1}^{(m)} - \Phi^{(m)} (x^{(m)}_{t-1}, y^{(m)}_{t-1}; \mathcal{B}_{x}))
\end{align*}
where $\Phi^{(m)}$ follows the definition in Eq~\eqref{eq:outer_grad_est} by replacing $f$ and $g$ with $f^{(m)}$ and $g^{(m)}$, respectively. If $t\ \text{mod}\ I = 0$ we average $x^{(m)}_t$, $x^{(m)}_{t-1}$ and the momentum state $\hat{\nu}^{(m)}_t$ as in Line 17 of Algorithm~\ref{alg:FedBiOAcc}. The analysis of FedBiOACC is more complicated than that of FedBiO. There are several types of errors we need bound to get the convergence, which includes the entangled inner variable estimation error and the outer variable consensus error as in FedBiO, but also the biases from the momentum terms, \emph{i.e.} the outer momentum bias $\| \nu^{(m)}_{t} -  \nabla h(x^{(m)}_t) \|^2$ and the inner momentum bias $\|\omega^{(m)}_t - \nabla_y g^{(m)}(x^{(m)}_{t}, y^{(m)}_{t} ) \|^2$. However, we still see the favorable $O(\epsilon^{-1.5})$ convergence rate of FedBiOACC by balancing different sources of errors. In fact, For both types of momentum biases, we can derive similar recursive equations as its non-distributed counterpart~\cite{yang2021provably} but with additional terms related to the outer variable consensus error, and for the consensus error, we can bound it by carefully choosing the related hyper-parameters in Algorithm~\ref{alg:FedBiOAcc}.

\section{Convergence Analysis}
In this section, we provide formal analysis to the convergence of our two algorithms, \emph{i.e.} FedBiO and FedBiOAcc. 

\subsection{Additional Assumptions}
We first state some mild assumptions needed in our analysis. We assume $f^{(m)} (x,y)$ and $g^{(m)} (x, y)$ for $m \in [m]$ satisfy Assumption~\ref{assumption:function},Assumption~\ref{assumption:f_smoothness}, Assumption~\ref{assumption:g_smoothness} and Assumption~\ref{assumption:noise_assumption} as Defined in Section~3. Next we also need to bound the differences among clients to get convergence results. More precisely, we assume Assumption~\ref{assumption:hetero} holds. Similar assumptions have been used in previous Federated Learning literature~\cite{khanduri2021near, woodworth2021minimax}.
\begin{assumption}
For any $m, j \in [M]$ and $x$, we have:
    \begin{itemize}
        \item [a)] $ \| \nabla_x f^{(m)} (x, y) -  \nabla_x f^{(j)} f (x, y) \| \leq \zeta_f$
        \item [b)] $ \| \nabla_y f^{(m)} (x, y) -  \nabla_y f^{(j)} f (x, y) \| \leq \zeta_f$
        \item [c)] $ \| \nabla_{xy} g^{(m)} (x, y) -  \nabla_{xy} g^{(j)} (x, y) \| \leq \zeta_{g,xy}$
        \item [d)] $ \| \nabla_{y^2} g^{(m)} (x, y) -  \nabla_{y^2} g^{(j)} (x, y) \| \leq \zeta_{g,yy}$
        \item [e)] $ \| y^{(m)}_x -  y^{(j)}_x\| \leq \zeta_{g^{\ast}}$
    \end{itemize}
    where $\zeta_f$, $\zeta_{g,xy}$, $\zeta_{g,yy}$, $\zeta_{g^{\ast}}$ are constants.
	\label{assumption:hetero}
\end{assumption}

Based on the above Assumption, we have the following Proposition to bound the overall heterogeneity of the function $h^{(m)}(x)$, $m \in [M]$:
\begin{proposition}\label{prop:heterogeneity}
With Assumption~\ref{assumption:function},~\ref{assumption:f_smoothness},~\ref{assumption:g_smoothness} and Assumption~\ref{assumption:hetero} hold, we have:
\begin{align*}
    \|  \nabla h^{(m)}(x) - \nabla h^{(j)}(x)  \|  \leq \zeta
\end{align*}
where $\zeta = (1 + C_{g,xy}/\mu)\zeta_{f} + \mu\zeta_{g,xy}/M + MC_{g,xy}\zeta_{g, yy}/\mu^2 + (L + L_{g, xy}\mu/M + C_{g,xy}L/\mu + MC_{g,xy}L_{g,y^2}/\mu^2)\zeta_{g^{\ast}}$.
\end{proposition}
Next, in addition to the bounded noise Assumption~\ref{assumption:noise_assumption}. We make the following assumption:
\begin{assumption}
The bias and variance of the stochastic hyper-gradient is bounded, \emph{i.e.} $\mathbb{E}[\|\mu^{(m)}_t - \mathbb{E}[\mu^{(m)}_t]\|^2] \leq \sigma^2$ and $\mathbb{E}[\|\mathbb{E}[\mu^{(m)}_t] - \Phi(x^{(m)}_t, y^{(m)}_t)\|^2] \leq G^2$ for $m \in [M]$ and $t \in [T]$, where $\mu^{(m)}_t$ is the stochastic hyper-gradient denoted in Line 10 of Algorithm~\ref{alg:FedBiOAcc}.
\label{assumption:outer_noise}
\end{assumption}
The assumption is reasonable due to Proposition~\ref{prop:hg_var}, and we can choose $\sigma = G_1$ and $G = G_2$. 

\subsection{Convergence Analysis for FedBiO and FedBiOAcc}
In this subsection, we provide the convergence result for our FedBiO algorithm~\ref{alg:FedBiO} and FedBiOAcc algorithm~\ref{alg:FedBiOAcc}. Firstly, for FedBiO, we have the following Theorem:
\begin{theorem}
\label{theorem:fedbio}
Suppose Assumption~\ref{assumption:function}-~\ref{assumption:g_smoothness},~\ref{assumption:hetero} hold$, \delta < min\bigg(\frac{\sqrt{(1-q)(1-q_1q^I)}}{2\Gamma\rho I\sqrt{\bar{q}_1q_1q^I }}, \frac{1}{12L_hI}, \frac{\mu\gamma}{2}, 1\bigg)$, $\gamma < \frac{1}{L}$ and $\eta = \frac{\delta}{\sqrt{T}}$, we have:
\begin{align*}
  \frac{1}{T}  \sum_{t = 1}^T \|\nabla h(\bar{x}_t)\|^2  & \leq \frac{2(h(\bar{x}_t) - h^\ast)}{\delta\sqrt{T}} + \frac{L_h^2 B_{\bar{t}_{0}}}{(1 - q)T} + \frac{2L_h^2 B_{\bar{t}_{0}}}{(1 - q)(1 -q_1q^I)T} +  \frac{M^{'}\delta^2}{T}
\end{align*}
where $B_{\bar{t}_0} = \frac{1}{M}\sum_{m=1}^M \|y^{(m)}_{1} - y^{(m)}_{x^{(m)}_{1}} \|^2$, $q = (1 - \frac{\mu\gamma}{2})$, $q_1 = 1 + \frac{\mu\gamma}{4}$ and $\bar{q}_1 = 1 + \frac{4}{\mu\gamma}$, $h^{\ast}$ the optimal value, $m^{'}$ is some constant.
\end{theorem}
We omit the exact form of some constants in Theorem~\ref{theorem:fedbio} and the full version can be found in Theorem~\ref{theorem:FedBiO}. As shown by the Theorem, our FedBiO converge with rate $O(\epsilon^{-2})$. Next we provide the convergence result for the FedBiOAcc algorithm. To prove the convergence of FedBiOAcc, we denote the  potential function $\mathcal{G}_t$ as follows:
\begin{align*}
    \mathcal{G}_t &= h(\bar{x}_{t}) + \frac{\eta}{320L_h^2\alpha_{t}}\Big\| \bar{\nu}_{t} - \frac{1}{M} \sum_{m=1}^M  \nabla h(x^{(m)}_t)  \Big\|^2  + \frac{1}{M} \sum_{m=1}^M \bigg\|y^{(m)}_t - y^{(m)}_{x^{(m)}_{t}} \bigg\|^2 \nonumber \\
    & \qquad  + \frac{\gamma}{32L^2\alpha_{t}} \sum_{m=1}^M \bigg\|\omega^{(m)}_t - \nabla_y g^{(m)}(x^{(m)}_{t}, y^{(m)}_{t} ) \bigg\|^2
\end{align*}
Then we have the following result for FdBiOAcc:
\begin{theorem}
\label{theorem:fedbioacc}
Suppose Assumption~\ref{assumption:function}-~\ref{assumption:noise_assumption},~\ref{assumption:hetero},~\ref{assumption:outer_noise} hold and the hyper-parameter $c_{\nu}$, $c_{\omega}$, $\eta$, $\gamma$, $\delta$ and $u$ are chosen according to Theorem~\ref{theorem:FedBiOAcc} and the learning rate $\alpha_t$ is chosen as in Algorithm~\ref{alg:FedBiOAcc}, then we have:
\begin{align*}
    \frac{1}{T}\sum_{t = 1}^{T-1} \mathbb{E} \bigg[ \|\nabla h(\bar{x}_t) \|^2 \bigg]  & \leq M^{'}\bigg(\frac{u^{1/3}}{\delta T} + \frac{\sigma^{2/3}}{\delta T^{2/3}}\bigg)  \nonumber\\
\end{align*}
where $M^{'}$ is some constant and the expectation is w.r.t the stochasticity of the algorithm.
\end{theorem}
The full version of Theorem~\ref{theorem:fedbioacc} is in shown in Theorem~\ref{theorem:FedBiOAcc}.

\begin{remark}
Recall that $T$ is the total number of running steps, so our FedBiOAcc has convergence rate of $O(\epsilon^{-1.5})$, but note that the above theorem does not show the linear speedup \emph{w.r.t} the number of clients $M$ as in the standard Federated Learning. In fact, there are three sources of errors. Firstly, we use the oracle $\Phi(x,y; \mathcal{B}_{x})$ to estimate the hyper-gradient, it has both bias and variance. Its variance does enjoy the linear speedup \emph{w.r.t} $M$, but the bias does not. Then for the inner gradient, we use the oracle $\nabla_y g^{(m)} (x^{(m)}_{t-1}, y^{(m)}_{t-1}, \mathcal{B}_{y})$, which is an unbiased estimator. But because that the inner problem is minimized locally and not averaged over all clients, the stochastic noise related to it does not decrease linearly \emph{w.r.t} $M$ neither.
\end{remark}

\section{Fair Federated Bilevel Learning}
In this section, we apply FedBiO and FedBiOAcc to solve the Fair Federated Learning tasks. The code of all experiments is written in Pytorch and the Federated Learning environment is simulated via Pytorch.Distributed Package. We use servers with AMD EPYC 7763 64-Core CPU.

\subsection{Group Fair Federated Learning}

\begin{table*}[ht]
\centering
\small
\setlength\extrarowheight{0.1pt}
\setlength{\tabcolsep}{0.05pt}
\caption{Performance comparison of FedBiO, FedBiOACC and other baselines}
\begin{tabular}{c|c|c|c|c|c|c|c}
\toprule
\multirow{8}{*}{\textbf{Adult}} & \textbf{Distribution} & \multicolumn{3}{c}{\textbf{I.I.D}} & \multicolumn{3}{c}{\textbf{Non-I.I.D}}\\ \cline{2-8}
&\textbf{Metrics}  & \textbf{Test Acc.}    & \textbf{Train EqOpp.}   & \textbf{Test EqOpp.}    & \textbf{Test Acc.}  & \textbf{Train EqOpp.}    & \textbf{Test EqOpp.} \\ \clineB{2-8}{2.5}
&FedAvg   & .8239$\pm$.0167       & .0391$\pm$ .0061      & .0420$\pm$.0034  & .8283$\pm$.0080     & .0261$\pm$.0022       & .0507$\pm$.0011        \\ \cline{2-8}
&FedReg    & .8240$\pm$.0159       & .0361$\pm$.0047       & .0425$\pm$.0029        & .8271$\pm$.0077  & \textbf{.0244$\pm$.0013}        & .0498$\pm$.0010       \\ \cline{2-8}
&FedMinMax    & .8228$\pm$.0163       & .0220 $\pm$.0057      & .0366$\pm$.0049        & .8272$\pm$.0077  & .0274$\pm$.0029        & .0363$\pm$.0013       \\ \cline{2-8}
&FCFL     & .8238$\pm$.0159       & .0356$\pm$.0029       & .0452$\pm$.0012        & .8273$\pm$.0074  & .0249$\pm$.0032        & .0501$\pm$.0014       \\ \clineB{2-8}{2.5}
&\textbf{FedBiO} & .8228$\pm$.0163       & .0238$\pm$.0058      & .0337$\pm$.0012        & \textbf{.8331$\pm$.0019}  & .0263$\pm$.0010        & \textbf{.0338$\pm$.0008}       \\ \cline{2-8}
&\textbf{FedBiOAcc} & \textbf{.8391$\pm$.0163}       & \textbf{.0222$\pm$.0064}      & \textbf{.0335$\pm$.0006}        & .8204$\pm$.0013  & .0289$\pm$.0005        & .0356$\pm$.0055      \\ \midrule\midrule
\multirow{8}{*}{\textbf{Credit}}  &FedAvg&  .6873$\pm$.0314     &  .0788$\pm$.0136    &  .0599$\pm$.0122       &  \textbf{.7386$\pm$.0011}   & .0832$\pm$.0248        & .1354$\pm$.0128 \\ \cline{2-8}
&FedReg & .6870$\pm$.0374      &  .0836$\pm$.0015    &  .0575$\pm$.0114      &   .7303$\pm$.0097  & .0735$\pm$.0216        & .1341$\pm$.0088 \\ \cline{2-8}
&FedMinMax &  .6759$\pm$.0757     & .0857$\pm$.0042      &  .0722$\pm$.0013     & .6966$\pm$.0104   & \textbf{.0477$\pm$.0155}        & .1222$\pm$.0024 \\\cline{2-8}
&FCFL &  .6864$\pm$.0237     & .0727$\pm$.0073      &  \textbf{.0375$\pm$.0028}      &  .7266$\pm$.0026  & .0777$\pm$.0162        & .1463$\pm$.0014 \\\clineB{2-8}{2.5}
&\textbf{FedBiO}  & .7015$\pm$.0169       & \textbf{.0548$\pm$.0072}       & .0513$\pm$.0059        & .7339$\pm$.0033 & .0782$\pm$.0116        & .1260$\pm$.0013  \\ \cline{2-8}
&\textbf{FedBiOAcc}  & \textbf{.7067$\pm$.0121}       & .0665$\pm$.0034       & .0501$\pm$.0051        & .7312$\pm$.0023 & .0799$\pm$.0152        & \textbf{.1021$\pm$.0011} \\ \bottomrule
\end{tabular}
\label{tb:1}
\end{table*}



In this task, we investigate the group fairness in Federated Learning from the Bilevel Optimization's perspective. Suppose $D^{(m)}_{t} = \{x^{(m,t)}_{i}, y^{(m,t)}_{i}, a^{(m,t)}_{i}\}$ is the training set at the $m_{th}$ client, and $D^{(m)}_{v} = \{x^{(m,v)}_{i}, y^{(m,v)}_{i}, a^{(m,v)}_{i}\}$ denotes the validation set at $m_{th}$ client. $x \in \mathbb{R}^d$, $y \in [I]$, $a \in [K]$ are input attributes, predictive attributes (we use classification as an demonstration) and sensitive attributes, respectively. Furthermore, we denote $n^{(m,v)}_{y,a}$ as the number of samples which has label $y$ and sensitive attribute label $a$ over $m_{th}$ client. Similarly, we denote $n^{(m,v)}_{y,*}$ as the number of samples with label $y$ and $n^{(m,v)}_{*,a}$ as the number of samples with sensitive attributes $a$. We can define similar notations for the training set.

We assume the validation sets have the follow properties: $n^{(m,v)}_{*,a_1} = n^{(m,v)}_{*,a_2}$, in other words, we assume the validation sets are group balanced. Then we optimize the following objective to learn a group fair model:
\begin{align*}
    \underset{\omega \in \Omega}{\min}\ & \frac{1}{M} \sum_{m=1}^{M} \frac{1}{n^{(m,v)}} \sum_{i=1}^{n^{(m,v)}} f(\theta^{(m)}_{\omega}; x^{(m,v)}_{i}, y^{(m,v)}_{i})
    \emph{s.t.}\ \theta^{(m)}_{\omega} = \underset{\theta \in \mathbb{R}^{d}}{\arg\min} \frac{1}{n^{(m,t)}} \sum_{i=1}^{n^{(m,t)}} \omega_{a}f(\theta; x^{(m,t)}_{i}, y^{(m,t)}_{i})
\end{align*}
where $f$ denotes the model to fit and $\omega = \{\omega_a\}, a \in [K]$ are weights for sensitive groups. Intuitively, the group weights are tuned such that the learned model $\theta^{(m)}_{\omega}$ performs well over the validation set $D^{(m)}_t$. Since $D^{(m)}_t$ have balanced samples from all sensitive groups, the model $\theta^{(m)}_{\omega}$ has to perform equally well for all different groups to get low loss over the validation set. One advantage of this formulation is that it does not rely on a specific group fairness metric such as Equal Opportunity (EqOpp)~\cite{hardt2016equality} or Equalized Odds (EqOdds)~\cite{hardt2016equality}. Furthermore, it also does not need access to the global statistics of groups which is hard to acquire in the Federated Learning setting.

We then solve the above bilevel problem with our FedBiO and FedBiOAcc algorithms.  We also compare with the following baselines: FedAvg~\cite{mcmahan2017communication}, FedReg and two recent works FedMinMax~\cite{papadaki2021federating}, FCFL~\cite{cui2021addressing}. Methods proposed in~\cite{zhang2021unified} are similar to FedMinMax, we do not include it in the results.  Our focus is group fairness, so we do not include client fairness (robustness) focused models such as AFL~\cite{mohri2019agnostic} and q-FedAvg~\cite{li2019fair}. The FedReg baseline is to add a regularization term over the FedAvg objective, and the regularization term could be any fairness metrics such as EqOpp. Note that FedReg evaluates the metric with local statistics only.

We test over real-world benchmark datasets 
Credit~\cite{Asuncion+Newman:2007} and Adult~\cite{kohavi1996scaling}. We pre-process the datasets with code provided by~\cite{diana2021minimax}. For each dataset, we first split it into train and test splits with ratio 7:3, and we keep the group distribution the same for the train and test splits. Then for the train set, we consider both I.I.D and Non-I.I.D cases. For the I.I.D case, we uniformly randomly split the train-set into three subsets and distribute each subset to a client. For the Non-I.I.D case, we split the train-set by sensitive attributes and for each attribute, we split its data into three shares with ratio $2:2:6$ and then randomly distribute each share to one client. Finally, for our FedBiO and FedBiOAcc, we select a small subset of the local train set to create the group-fair validation set. We fit a logistic regression model over the benchmark datasets. For our methods,
we perform a two stage training procedure: we first estimate optimal group weights with the bilevel formulation, then we use the learned weight to fit a weighted logistic regression model with FedAvg. For FedReg and FCFL, we choose its regularization term as the EqOpp metric. The definition of EqOpp metric is included in the Appendix~\ref{more-experiments}. Finally, we perform grid search for the hyper-parameters of all methods and hyper-parameter choices are introduced in the Appendix~\ref{more-experiments}.

We summarize results in Table~\ref{tb:1}, where we use the Test accuracy and EqOpp as metrics, we run 10 runs for each case and report the mean and standard deviation in the table. The best result for each metric is highlighted. As shown by the table,  either FedBiO or FedBiOAcc gets the best result for most cases. FedReg/FCFL are based on local group statistics to achieve fairness, and they tend to perform worse in the Non-I.I.D case, \emph{e.g.} for the Adult dataset, FCFL gets a much lower Train EqOpp in the Non-I.I.D case compared to the I.I.D one, but its Test EqOpp. is worse. FedMinMax is a strong baseline and can get good performance under both settings. However, our algorithms have two advantages compared to FedMinMax. Firstly, we don't  query global statistics, furthermore, our algorithms communicate every $I$ iterations, while FedMinMax collects the model states from clients at every iteration. 



\section{Conclusion}
In this paper, we studied a class of novel Federated Bilevel Optimization problems, and proposed two efficient algorithms, \emph{i.e.}, FedBiO and FedBiOAcc, to solve these problems. Moreover, we provided a rigorous convergence analysis framework for our proposed methods. Specifically, we proved that our FedBiO converges with $O(\epsilon^{-2})$ and our FedBiOAcc converges with $O(\epsilon^{-1.5})$. Meanwhile, we apply our new algorithms to solve the important Fair Federated Learning problem with using a new bilevel optimization formulation. The experimental results validate the efficacy of our algorithms.




\small
\bibliography{BiO}

\newpage
\appendix
\onecolumn

\begin{onecolumn}

\begin{appendices}

\section{Preliminaries}

Before we start the proof, we first define some notations. We define $\bar{t}_s \coloneqq sI + 1$ with $s \in [S]$. Note at $\bar{t}_s$ iteration, we have $x_t^{(m)} = \bar{x}_t$ for $m \in [M]$.  For all proofs, we assume that Assumptions~\ref{assumption:function}-~\ref{assumption:g_smoothness},~\ref{assumption:hetero} hold, and there is stochasity, we assume Assumption~\ref{assumption:noise_assumption} and~\ref{assumption:outer_noise} hold.

Then we state some propositions useful in the proof:
\begin{proposition} (generalized triangle inequality)
\label{prop:generali_tri}
Let $\{x_k\}, k\in{K}$ be $K$ vectors. Then the following are true:
\begin{enumerate}
	\item $||x_i + x_j||^2 \le (1 + a)||x_i||^2 + (1 + \frac{1}{a})||x_j||^2$ for any $a > 0$, and
	\item  $||\sum_{k=1}^K x_k||^2 \le K\sum_{k=1}^{K} ||x_k||^2$
\end{enumerate}
\end{proposition}

\begin{proposition}
\label{prop: Sum_Mean_Kron}
For a finite sequence $x^{(k)} \in \mathbb{R}^d$ for $k \in [K]$ define $\bar{x} \coloneqq \frac{1}{K} \sum_{k = 1}^K x^{(k)}$, we then have
\begin{align*}
\sum_{k=1}^K    \| x^{(k)} - \bar{x} \|^2 \leq \sum_{k=1}^K    \| x^{(k)} \|^2. 
\end{align*}
\end{proposition}
Proof for Proposition~\ref{prop:generali_tri} can be found in ~\cite{li2021fully} and the proof for Propostion~\cite{khanduri2021near}.

\begin{proposition}[\cite{cutkosky2019momentum}]
\label{Lem: AD_Sum_1overT}
Let $a_0 > 0$ and $a_1,a_2, \ldots, a_T \geq 0$. We have
$$\sum_{t=1}^T \frac{a_t}{a_0 + \sum_{i=t}^t a_i} \leq \ln \bigg(1 + \frac{\sum_{i=1}^t a_i}{a_0} \bigg).$$
\end{proposition}

\begin{proposition} (Restate of Propositon~\ref{some smoothness})
Suppose Assumptions~\ref{assumption:f_smoothness} and ~\ref{assumption:g_smoothness} hold, the following statements hold:
\begin{itemize}
\item [a)] $\|\Phi(x; y) - \nabla h(x)\| \leq C \|y_x- y\|$,
where $C=L+LC_{g,xy}/\mu + M(L_{g,xy}/\mu + L_{g, y^2}C_{g, xy}/\mu^2)$.

\item [b)] $y_x$ is Lipschitz continuous in $x$ with constant $\rho = C_{g,xy}/\mu$.

\item [c)] $h(x)$ is Lipschitz continuous in $x$ with constant $L_h$ i.e., for any given $x_1, x_2 \in X$, we have
$\|\nabla h(x_2) - \nabla h(x_1)\| \le \bar{L} \|x_2 - x_1\|$
where $\bar{L} =(L+C)C_{g,xy}\mu+L+M(L_{g, xy}M\mu + L_{g,y^2}C_{g,xy}\mu^2)$.

\item [d)] $\|\Phi(x_1; y_1) - \Phi(x_2; y_2)\|^2 \leq \Gamma^2 (\|x_1- x_2\|^2 + \|y_1- y_2\|^2)$,
where $\Gamma = L + ML_{g,xy}/\mu + C_{g,xy}(L/\mu+ ML_{g,yy}/\mu^2)$.
\end{itemize}
We denote $L_h = max(\bar{L}, \Gamma, C)$ for convenience.
\end{proposition}
Note if Case d) holds, it is straightforward to also get the stochastic version, \emph{i.e.} $\|\Phi(x_1; y_1; \mathcal{B}) - \Phi(x_2; y_2, \mathcal{B})\|^2 \leq \Gamma^2 (\|x_1- x_2\|^2 + \|y_1- y_2\|^2)$
\begin{proof}
We only prove the Case d here. Proof of other cases can be found in Lemma 2.2 of~\cite{ghadimi2018approximation}.
\begin{align*}
     & \|\Phi(x_1; y_1) - \Phi(x_2; y_2)\| \\
     &= \bigg\|\nabla_x f (x_1, y_1) - \nabla_{xy} g(x_1, y_1) \Big(\nabla_{yy} g(x_1, y_1)\Big)^{-1}\nabla_{y} f(x_1, y_1) \nonumber \\
     &\quad \quad \quad - \nabla_x f (x_2, y_2) - \nabla_{xy} g(x_2, y_2) \Big(\nabla_{yy} g(x_2, y_2)\Big)^{-1}\nabla_{y} f(x_2, y_2) \bigg\|  \nonumber \\
     & \leq  \bigg\|\nabla_x f (x_1, y_1) - \nabla_x f (x_2, y_2) \bigg\| \nonumber  + \bigg\| \nabla_{xy} g(x_1, y_1) \\
     & \quad - \nabla_{xy} g(x_2, y_2) \bigg\| \bigg\|\Big(\nabla_{yy} g(x_2, y_2)\Big)^{-1}\nabla_{y} f(x_1, y_1)\bigg\| \nonumber \\
     & \quad + \bigg\| \nabla_{xy} g(x_2, y_2) \bigg\|\bigg\| \Big(\nabla_{yy} g(x_1, y_1)\Big)^{-1}\nabla_{y} f(x_1, y_1)  - \Big(\nabla_{yy} g(x_2, y_2)\Big)^{-1}\nabla_{y} f(x_2, y_2) \bigg\| \nonumber \\
     & \leq  \bigg(L + \frac{M L_{g,xy}}{\mu} + C_{g.xy}\bigg(\frac{L}{\mu} + \frac{ML_{g,yy}}{\mu^2}\bigg) \bigg)\bigg( \bigg\|x_1 - x_2 \bigg\|^2 +  \bigg\| y_1 - y_2 \bigg\|^2\bigg)^{1/2}   \\
\end{align*}
which finishes the proof.
\end{proof}

\begin{proposition}
With Assumption~\ref{assumption:function},~\ref{assumption:f_smoothness},~\ref{assumption:g_smoothness} and Assumption~\ref{assumption:hetero} hold, we have:
\begin{align*}
    \|  \nabla h^{(m)}(x) - \nabla h^{(j)}(x)  \|  \leq \bigg(1 + \frac{C_{g,xy}}{\mu}\bigg)\zeta_{f} + \frac{\mu}{M}\zeta_{g,xy} + \frac{MC_{g,xy}\zeta_{g, yy}}{\mu^2} + \bigg(L + \frac{L_{g, xy}\mu}{M} + \frac{C_{g,xy}L}{\mu} + \frac{MC_{g,xy}L_{g,y^2}}{\mu^2}\bigg)\zeta_{g^{\ast}}
\end{align*}
\end{proposition}

\begin{proof}
Follow the formulation shown in Eq.~\eqref{eq:outer_grad}, we have:
\begin{align*}
\|  \nabla h^{(m)}(x) - \nabla h^{(j)}(x)  \| &= \|\nabla_x f^{(m)}(x, y^{(m)}_x) -  \nabla_{xy}^2 g^{(m)}(x, y^{(m)}_x)[\nabla_{y^2}^2 g^{(m)}(x, y^{(m)}_x)]^{-1} \nabla_y f^{(m)}(x, y^{(m)}_x) \nonumber \\
& \qquad-  \left(\nabla_x f^{(j)}(x, y^{(j)}_x) -  \nabla_{xy}^2 g^{(j)}(x, y^{(j)}_x) [\nabla_{y^2}^2 g^{(j)}(x, y^{(j)}_x)]^{-1} \nabla_y f^{(j)}(x, y^{(j)}_x)\right)\| \nonumber \\
& \leq  \bigg\|\nabla_x f^{(m)} (x, y_x^{(m)}) - \nabla_x f^{(j)} (x, y^{(j)}_{x}) \bigg\|  + \bigg\| \nabla_{xy} g^{(m)}(x, y_x^{(m)}) \\
& \quad - \nabla_{xy} g^{(j)}(x, y^{(j)}_{x}) \bigg\| \bigg\|\Big(\nabla_{yy} g^{(m)}(x, y^{(m)}_{x})\Big)^{-1}\nabla_{y} f^{(m)}(x, y_x^{(m)})\bigg\| \nonumber \\
& \quad + \bigg\| \nabla_{xy} g^{(j)}(x, y^{(j)}_{x}) \bigg\|\bigg\| \Big(\nabla_{yy} g^{(m)}(x, y^{(m)}_{x})\Big)^{-1}\nabla_{y} f^{(m)}(x, y_x^{(m)}) \\ & \quad - \Big(\nabla_{yy} g^{(j)}(x, y^{(j)}_{x})\Big)^{-1}\nabla_{y} f^{(j)}(x, y^{(j)}_{x}) \bigg\| \nonumber \\
\end{align*}
where the inequality is due to the triangle inequality. Next we bound the three terms separately. For the first term:
\begin{align}
    \bigg\|\nabla_x f^{(m)} (x, y_x^{(m)}) - \nabla_x f^{(j)} (x, y^{(j)}_{x}) \bigg\| 
    & \leq \bigg\|\nabla_x f^{(m)} (x, y_x^{(m)}) - \nabla_x f^{(j)} (x, y^{(m)}_{x}) \bigg\| \nonumber \\
    & \qquad + \bigg\|\nabla_x f^{(j)} (x, y_x^{(m)}) - \nabla_x f^{(j)} (x, y^{(j)}_{x}) \bigg\| \nonumber \\
    & \leq \zeta_{f} + L\bigg\| y_x^{(m)} - y^{(j)}_{x} \bigg\| \leq \zeta_{f} + L\zeta_{g^{\ast}} \nonumber\\
\label{eq:diff1}
\end{align}
where the second inequality is due to the Assumption~\ref{assumption:hetero} and smoothness assumption the Assumption~\ref{assumption:f_smoothness}. The last inequality also follows the Assumption~\ref{assumption:hetero}. Next, for the second term, we have:
\begin{align*}
  &\bigg\| \nabla_{xy} g^{(m)}(x, y_x^{(m)}) - \nabla_{xy} g^{(j)}(x, y^{(j)}_{x}) \bigg\| \bigg\|\Big(\nabla_{yy} g^{(m)}(x, y^{(m)}_{x})\Big)^{-1}\nabla_{y} f^{(m)}(x, y_x^{(m)})\bigg\| \nonumber \\
  & \leq \frac{\mu}{M}\bigg\| \nabla_{xy} g^{(m)}(x, y_x^{(m)}) - \nabla_{xy} g^{(j)}(x, y^{(j)}_{x}) \bigg\| \nonumber \\
  & \leq \frac{\mu}{M}\bigg\| \nabla_{xy} g^{(m)}(x, y_x^{(m)}) - \nabla_{xy} g^{(j)}(x, y^{(m)}_{x}) \bigg\| + \frac{\mu}{M}\bigg\| \nabla_{xy} g^{(j)}(x, y_x^{(m)}) - \nabla_{xy} g^{(j)}(x, y^{(j)}_{x}) \bigg\| \nonumber \\ \nonumber \\
  & \leq \frac{\mu\zeta_{g,xy}}{M}+ \frac{L_{g, xy}\mu}{M}\bigg\| y_x^{(m)} -  y^{(j)}_{x}) \bigg\| \leq \frac{\mu\zeta_{g,xy}}{M}+ \frac{L_{g, xy}\mu\zeta_{g^\ast}}{M}  \nonumber \\
\end{align*}
where the first inequality follows from the Assumption~\ref{assumption:function},~\ref{assumption:f_smoothness}; the second inequality follows from triangle inequality; the third inequality follows from Assumption~\ref{assumption:hetero},~\ref{assumption:g_smoothness}, the last inequality follows from Assumption~\ref{assumption:hetero}.  Next, for the third term, we have:
\begin{align*}
   &\bigg\| \nabla_{xy} g^{(j)}(x, y^{(j)}_{x}) \bigg\|\bigg\| \Big(\nabla_{yy} g^{(m)}(x, y^{(m)}_{x})\Big)^{-1}\nabla_{y} f^{(m)}(x, y_x^{(m)}) - \Big(\nabla_{yy} g^{(j)}(x, y^{(j)}_{x})\Big)^{-1}\nabla_{y} f^{(j)}(x, y^{(j)}_{x}) \bigg\| \nonumber \\ 
   & \leq C_{g,xy}\bigg\| \Big(\nabla_{yy} g^{(m)}(x, y^{(m)}_{x})\Big)^{-1}\nabla_{y} f^{(m)}(x, y_x^{(m)}) - \Big(\nabla_{yy} g^{(j)}(x, y^{(j)}_{x})\Big)^{-1}\nabla_{y} f^{(j)}(x, y^{(j)}_{x}) \bigg\| \nonumber \\
   & \leq C_{g,xy}\bigg\|\Big(\nabla_{yy} g^{(m)}(x, y^{(m)}_{x})\Big)^{-1}\bigg\|\bigg\|\nabla_{y} f^{(m)}(x, y_x^{(m)}) - \nabla_{y} f^{(j)}(x, y^{(j)}_{x}) \bigg\| \nonumber \\
   & \qquad + C_{g,xy}\bigg\| \Big(\nabla_{yy} g^{(m)}(x, y^{(m)}_{x})\Big)^{-1} - \Big(\nabla_{yy} g^{(j)}(x, y^{(j)}_{x})\Big)^{-1} \bigg\|\bigg\|\nabla_{y} f^{(j)}(x, y^{(j)}_{x})\bigg\| \nonumber \\
   & \leq \frac{C_{g,xy}}{\mu}\bigg\|\nabla_{y} f^{(m)}(x, y_x^{(m)}) - \nabla_{y} f^{(j)}(x, y^{(j)}_{x}) \bigg\| \nonumber \\
   & \qquad + MC_{g,xy}\bigg\| \Big(\nabla_{yy} g^{(m)}(x, y^{(m)}_{x})\Big)^{-1} - \Big(\nabla_{yy} g^{(j)}(x, y^{(j)}_{x})\Big)^{-1} \bigg\|\nonumber \\
   & \leq \frac{C_{g,xy}(\zeta_{f} + L\zeta_{g^{\ast}})}{\mu} + MC_{g,xy}\bigg\|\Big(\nabla_{yy} g^{(m)}(x, y^{(m)}_{x})\Big)^{-1}\bigg\|\times \nonumber \\
   & \qquad \qquad \qquad \bigg\| \nabla_{yy} g^{(m)}(x, y^{(m)}_{x}) - \nabla_{yy} g^{(j)}(x, y^{(j)}_{x}) \bigg\|\bigg\|\Big(\nabla_{yy} g^{(j)}(x, y^{(j)}_{x})\Big)^{-1}\bigg\| \\
   & \leq \frac{C_{g,xy}(\zeta_{f} + L\zeta_{g^{\ast}})}{\mu} + \frac{MC_{g,xy}(\zeta_{g, yy} + L_{g,y^2}\zeta_{g^\ast})}{\mu^2}
\end{align*}
where the first inequality is by Assumption~\ref{assumption:g_smoothness}; the second inequality is by triangle inequality; the third inequality is by Assumption~\ref{assumption:g_smoothness},~\ref{assumption:f_smoothness}; the fourth inequality is by Cauchy Schwartz inequality; the last inequality is by Assumption~\ref{assumption:function},~\ref{assumption:g_smoothness} and the result in Eq.~\eqref{eq:diff1}. Combine everything together, we have:
\begin{align*}
    \bigg\|\nabla_x f^{(m)} (x, y_x^{(m)}) - \nabla_x f^{(j)} (x, y^{(j)}_{x}) \bigg\| & \leq \zeta_{f,x} + L\zeta_{g^{\ast}} + \frac{\mu\zeta_{g,xy}}{M}+ \frac{L_{g, xy}\mu\zeta_{g^\ast}}{M} + \frac{C_{g,xy}(\zeta_{f,x} + L\zeta_{g^{\ast}})}{\mu} \nonumber \\
    & \qquad \qquad+ \frac{MC_{g,xy}(\zeta_{g, yy} + L_{g,y^2}\zeta_{g^\ast})}{\mu^2}
\end{align*}
which completes the proof.
\end{proof}

\section{Proof for the FedBiO Algorithm}
In this section, we present the proofs for the FedBiO algorithm, we will focus on the deterministic case. 

\subsection{Hyper-Gradient Bias}

\begin{lemma}
\label{lemma:hg_bound}
For all $t \in [\bar{t}_{s-1}, \bar{t}_s - 1]$, the iterates generated satisfy:
\begin{align*}
 \Big\|  \nabla h(\bar{x}_{t})  - \bar{\nu}_{t}   \Big\|^2 \leq \frac{L_h^2}{M}\sum_{m=1}^M \bigg( \bigg(1 + 2\rho^2\bigg)\bigg\|x_t^{(m)} - \bar{x}_t \bigg\|^2 + \bigg\| y^{(m)}_t - y^{(m)}_{x^{(m)}_{t}} \bigg\|^2 \bigg)
\end{align*}
\end{lemma}

\begin{proof}

\begin{align*}
      \Big\|  \nabla h(\bar{x}_{t})  - \bar{\nu}_{t}   \Big\|^2  &=   \bigg\| \frac{1}{M}\sum_{m=1}^M \big(\nu^{(m)}_t - \nabla h^{(m)}(\bar{x}_t) \big) \bigg\|^2 \overset{(a)}{\leq}  \frac{1}{M}\sum_{m=1}^M   \bigg\|\nu^{(m)}_t - \nabla h^{(m)}(\bar{x}_t) \bigg\|^2  \nonumber \\
     & \overset{(b)}{\leq} \frac{L_h^2}{M}\sum_{m=1}^M  \bigg( \bigg\|x_t^{(m)} - \bar{x}_t \bigg\|^2 +  \bigg\| y_t^{(m)} - y^{(m)}_{\bar{x}_t} \bigg\|^2\bigg)   \\
     & \leq \frac{L_h^2}{M}\sum_{m=1}^M  \bigg( \bigg\|x_t^{(m)} - \bar{x}_t \bigg\|^2 + \bigg\| y^{(m)}_t - y^{(m)}_{x^{(m)}_{t}} + y^{(m)}_{x^{(m)}_{t}} - y^{(m)}_{\bar{x}_{t}} \bigg\|^2 \bigg)  \\
     & \leq \frac{L_h^2}{M}\sum_{m=1}^M  \bigg( \bigg(1 + 2\rho^2\bigg)\bigg\|x_t^{(m)} - \bar{x}_t \bigg\|^2 + \bigg\| y^{(m)}_t - y^{(m)}_{x^{(m)}_{t}} \bigg\|^2 \bigg) 
\end{align*}

\end{proof}

where inequality (a) follows the generalized triangle inequality; inequality (b) follows the Proposition~\ref{some smoothness}.

\subsection{Inner Variable Drift Lemma}
\begin{lemma}
\label{lemma: inner_drift}
When $\gamma < \frac{1}{L}$, we have:
\begin{align*}
    \sum_{t  = 1}^{T} \frac{1}{M}\sum_{m=1}^M \bigg\|y^{(m)}_{\bar{t}_s} - y^{(m)}_{x^{(m)}_{\bar{t}_s}} \bigg\|^2
    & \leq \frac{B_{\bar{t}_{0}}}{1 - q} + \frac{B_{\bar{t}_{0}}}{(1 - q)(1 -q_1q^I)} + \frac{\rho^2\eta^2q_1\bar{q}M_h^2(S-1)}{(1-q)^2(1-q_1q^I)} \nonumber \\
    & \qquad \qquad + \frac{\bar{q}_1\rho^2q_1q^I}{(1 - q)(1 - q_1q^I)}\sum_{t=1}^{T}\hat{A}_{\bar{t}_j}  +  \frac{\bar{q}\rho^2\eta^2TM_h^2 }{1 - q}
\end{align*}
where $B_t = \frac{1}{M}\sum_{m=1}^M \bigg\|y^{(m)}_{\bar{t}_s} - y^{(m)}_{x^{(m)}_{\bar{t}_s}} \bigg\|^2$ and $\hat{A}_{\bar{t}_s} = \frac{1}{M}\sum_{m=1}^M  \bigg\|\hat{x}^{(m)}_{\bar{t}_{s}} - \bar{x}_{\bar{t}_{s}}\bigg\|^2$, $q = (1 - \frac{\mu\gamma}{2})$, $\bar{q} = (1 + \frac{2}{\mu\gamma})$, $q_1 = 1 + \frac{\mu\gamma}{4}$ and $\bar{q}_1 = 1 + \frac{4}{\mu\gamma}$, $M_h = \frac{M(\mu + C_{g,xy})}{\mu}$.
\end{lemma}

\begin{proof}

Note from Algorithm and the definition of $\bar{t}_s$ that at $t = \bar{t}_{s - 1}$ with $s \in [S]$, $x_{t}^{(m)} = \bar{x}_{t}$, for all $k$. For $t \in [\bar{t}_{s-1} + 1,  \bar{t}_s - 1]$, with $s \in [S]$, we have:

\begin{align*}
\quad\ \bigg\|y^{(m)}_t - y^{(m)}_{x^{(m)}_{t}} \bigg\|^2  &\leq (1 + \frac{\mu\gamma}{2})\bigg\|y^{(m)}_{t} - y^{(m)}_{x^{(m)}_{t-1}}\bigg\|^2 + (1 + \frac{2}{\mu\gamma})\bigg\|y^{(m)}_{x^{(m)}_{t}} - y^{(m)}_{x^{(m)}_{t-1}}\bigg\|^2  \nonumber \\
& \leq (1 + \frac{\mu\gamma}{2})(1 - \mu\gamma)\bigg\|y^{(m)}_{t-1} - y^{(m)}_{x^{(m)}_{t-1}}\bigg\|^2 +  (1 + \frac{2}{\mu\gamma}) \bigg\|y^{(m)}_{x^{(m)}_{t}} - y^{(m)}_{x^{(m)}_{t-1}}\bigg\|^2  \nonumber \\
& \leq (1 - \frac{\mu\gamma}{2})\bigg\|y^{(m)}_{t-1} - y^{(m)}_{x^{(m)}_{t-1}}\bigg\|^2 + \rho^2 (1 + \frac{2}{\mu\gamma})\bigg\|x^{(m)}_{t} - x^{(m)}_{t-1}\bigg\|^2 \nonumber \\ 
&\leq (1 - \frac{\mu\gamma}{2})\bigg\|y^{(m)}_{t-1} - y^{(m)}_{x^{(m)}_{t-1}}\bigg\|^2 + \rho^2\eta^2 (1 + \frac{2}{\mu\gamma}) \bigg\|\nu^{(m)}_{t-1}\bigg\|^2  \nonumber \\
& \leq (1 - \frac{\mu\gamma}{2})\bigg\|y^{(m)}_{t-1} - y^{(m)}_{x^{(m)}_{t-1}}\bigg\|^2 + \rho^2\eta^2 (1 + \frac{2}{\mu\gamma})M_h^2 
\end{align*}
where the second inequality is due to the property of gradient descent for strongly convex function when $\gamma < 1/L$. For the last inequality, we use the fact that:
\begin{align*}
    \|\nu^{(m)}_{t-1}\| = \|\Phi^{(m)} (x^{(m)}_{t}, y^{(m)}_{t})\| \leq  M + \frac{MC_{g,xy}}{\mu}
\end{align*}
and we denote $M_h = M + \frac{MC_{g,xy}}{\mu}$. We also denote $q = (1 - \frac{\mu\gamma}{2})$ and $\bar{q} = (1 + \frac{2}{\mu\gamma})$ for ease of notation. By telescoping two sides, we have:
\begin{align*}
    \bigg\|y^{(m)}_{t} - y^{(m)}_{x^{(m)}_{t}} \bigg\|^2
    & \leq q^{t - \bar{t}_{s-1}}\bigg\|y^{(m)}_{\bar{t}_{s-1}} - y^{(m)}_{x^{(m)}_{\bar{t}_{s-1}}}\bigg\|^2 + \rho^2\eta^2\bar{q}\sum_{l=\bar{t}_{s-1}}^{t - 1} q^{t-l-1}M_h^2 \nonumber \\
    & \overset{(b)}{\leq} q^{t - \bar{t}_{s-1}}\bigg\|y^{(m)}_{\bar{t}_{s-1}} - y^{(m)}_{x^{(m)}_{\bar{t}_{s-1}}}\bigg\|^2 + \frac{\rho^2\eta^2\bar{q}M_h^2}{1-q} \nonumber
\end{align*}
where in inequality $(b)$, we use the fact $q^{t- \bar{t}_{s-1}} < 1$ for any $t$. Then we average over all $M$ clients and have:
\begin{align}
\label{eq:inner_ave_err}
    \frac{1}{M}\sum_{j=1}^M \bigg\|y^{(m)}_{t} - y^{(m)}_{x^{(m)}_{t}} \bigg\|^2 \nonumber &\leq \frac{1}{M}\sum_{m=1}^M q^{t - \bar{t}_{s-1}}\bigg\|y^{(m)}_{\bar{t}_{s-1}} - y^{(m)}_{x^{(m)}_{\bar{t}_{s-1}}}\bigg\|^2 + \frac{\rho^2\eta^2\bar{q}M_h^2}{1-q} \\
\end{align}
As for $t = \bar{t}_{s}$, we average variable $x$ over the $m$ clients and $x^{(m)}_{\bar{t}_s} = \bar{x}_{\bar{t}_s}$, while the inner variable error is related to the $x$ variable before averaging, \emph{i.e.} $\hat{x}^{(m)}_{\bar{t}_s}$. By the generalized triangle inequality, we have:
\begin{align*}
     \bigg\|y^{(m)}_{\bar{t}_{s}} - y^{(m)}_{\bar{x}_{\bar{t}_{s}}} \bigg\|^2 &\leq (1 + \frac{\mu\gamma}{4})\bigg\|y^{(m)}_{\bar{t}_{s}} - y^{(m)}_{\hat{x}^{(m)}_{\bar{t}_{s}}} \bigg\|^2  + (1 + \frac{4}{\mu\gamma}) \bigg\|y^{(m)}_{\hat{x}^{(m)}_{\bar{t}_{s}}} - y^{(m)}_{\bar{x}_{\bar{t}_{s}}}\bigg\|^2 \nonumber \\
     & \leq (1 + \frac{\mu\gamma}{4})\bigg\|y^{(m)}_{\bar{t}_{s}} - y^{(m)}_{\hat{x}^{(m)}_{\bar{t}_{s}}} \bigg\|^2 + \rho^2(1 + \frac{4}{\mu\gamma}) \bigg\|\hat{x}^{(m)}_{\bar{t}_{s}} - \bar{x}_{\bar{t}_{s}}\bigg\|^2 \\
\end{align*}
We denote $q_1 = 1 + \frac{\mu\gamma}{4}$ and $\bar{q}_1 = 1 + \frac{4}{\mu\gamma}$. By averaging over $M$ clients, we have:
\begin{align*}
\frac{1}{M}\sum_{m=1}^M  \bigg\|y^{(m)}_{\bar{t}_{s}} - y^{(m)}_{x^{(m)}_{\bar{t}_{s}}} \bigg\|^2
& \leq \frac{q_1}{M}\sum_{m=1}^M \bigg\|y^{(m)}_{\bar{t}_{s}} - y^{(m)}_{\hat{x}^{(m)}_{\bar{t}_{s}}} \bigg\|^2 + \frac{\rho^2\bar{q}_1}{M}\sum_{m=1}^M  \bigg\|\hat{x}^{(m)}_{\bar{t}_{s}} - \bar{x}_{\bar{t}_{s}}\bigg\|^2 
\end{align*}
The first term satisfies Eq.~(\ref{eq:inner_ave_err}) by setting $t = \bar{t}_s$. So we have:
\begin{align*}
    \frac{1}{M}\sum_{m=1}^M \bigg\|y^{(m)}_{\bar{t}_s} - y^{(m)}_{x^{(m)}_{\bar{t}_s}} \bigg\|^2  & \leq \frac{q_1q^{I}}{M}\sum_{m=1}^M \bigg\|y^{(m)}_{\bar{t}_{s-1}} - y^{(m)}_{x^{(m)}_{\bar{t}_{s-1}}}\bigg\|^2 + \frac{\bar{q}_1\rho^2}{M}\sum_{m=1}^M  \bigg\|\hat{x}^{(m)}_{\bar{t}_{s}} - \bar{x}_{\bar{t}_{s}}\bigg\|^2  + \frac{\rho^2\eta^2q_1\bar{q}M_h^2}{1-q} \nonumber \\
\end{align*}
For ease of notation, we denote $B_t = \frac{1}{M}\sum_{m=1}^M \bigg\|y^{(m)}_{\bar{t}_s} - y^{(m)}_{x^{(m)}_{\bar{t}_s}} \bigg\|^2$ and $\hat{A}_{\bar{t}_s} = \frac{1}{M}\sum_{m=1}^M  \bigg\|\hat{x}^{(m)}_{\bar{t}_{s}} - \bar{x}_{\bar{t}_{s}}\bigg\|^2  $ . Then the above equations can be written as:
\begin{align}
    B_t \nonumber &\leq q^{t - \bar{t}_{s-1}}B_{\bar{t}_{s-1}} + \frac{\rho^2\eta^2\bar{q}M_h^2}{1-q},  t \in [\bar{t}_{s-1} + 1,  \bar{t}_s - 1]\\
\label{eq:B_recur}
\end{align}
and:
\begin{align*}
    B_{\bar{t}_{s}}& \leq q_1q^{I}B_{\bar{t}_{s-1}} + \bar{q}_1\rho^2\hat{A}_{\bar{t}_s} + \frac{\rho^2\eta^2q_1\bar{q}M_h^2}{1-q}, t = \bar{t}_{s} \nonumber \\
\end{align*}
By telescoping, for $s\ge 1$ we have:
\begin{align}
     B_{\bar{t}_{s}} & \leq q_1^sq^{sI}B_{\bar{t}_{0}} + \frac{\rho^2\eta^2q_1\bar{q}M_h^2}{1 - q}\sum_{j=0}^{s-1}q_1^jq^{jI} + \bar{q}_1\rho^2\sum_{j=1}^{s}q_1^{s-j}q^{(s-j)I}\hat{A}_{\bar{t}_j} \nonumber \\
     & \leq q_1^sq^{sI}B_{\bar{t}_{0}} + \frac{\rho^2q_1\bar{q}\eta^2M_h^2}{(1 - q)(1 - q_1q^I)} + \bar{q}_1\rho^2\sum_{j=1}^{s}q_1^{(s-j)}q^{(s-j)I}\hat{A}_{\bar{t}_j} \nonumber \\
\label{eq:Bs}
\end{align}
Then by summing Eq.~(\ref{eq:B_recur}) from $\bar{t}_{s-1}$ to $t$, we have:
\begin{align*}
    \sum_{l  = \bar{t}_{s-1}}^{t} B_{l} & \leq \sum_{l  = \bar{t}_{s-1}}^{t} q^{l - \bar{t}_{s-1}}B_{\bar{t}_{s-1}} + \frac{\rho^2\eta^2\bar{q}(t - \bar{t}_{s-1} - 1)}{1 - q}M_h^2 \leq \frac{B_{\bar{t}_{s-1}} + \rho^2\eta^2\bar{q}(t - \bar{t}_{s-1} - 1)M_h^2 }{1 - q} \nonumber \\
\end{align*}
Combine the above inequality with Eq.~(\ref{eq:Bs}) and for $S\ge 2$, we have:
\begin{align}
    \sum_{t^{'}  = \bar{t}_{s-1}}^{t} B_{t^{'}} & \leq \frac{q_1^{s-1}q^{(s-1)I}B_{\bar{t}_{0}}}{1 - q} + \frac{\rho^2\eta^2q_1\bar{q}M_h^2}{(1 - q)^2(1 - q_1q^I)}  + \bar{q}_1\rho^2\sum_{j=1}^{s-1}\frac{q_1^{s-1-j}q^{(s-1-j)I}\hat{A}_{\bar{t}_j}}{1 - q}  +  \frac{\bar{q}\rho^2\eta^2(I- 1)M_h^2 }{1 - q} \nonumber \\
\label{eq:b_sum}
\end{align}
and for $s=1$, we have:
\begin{align}
     \sum_{t^{'}  = \bar{t}_{0}}^{\bar{t}_{1} - 1} B_{t^{'}} & \leq \frac{B_{\bar{t}_{0}}}{1 - q} + \frac{\bar{q}\rho^2\eta^2(I- 1)M_h^2 }{1 - q} \nonumber \\
\label{eq:case1}
\end{align}
Finally, we sum $t$ from $1 \to T$ and have:
\begin{align*}
    \sum_{t  = 1}^{T} B_{t} & \leq \frac{B_{\bar{t}_{0}}}{1 - q} + \sum_{s  = 2}^{S}\frac{q_1^{s-1}q^{(s-1)I}B_{\bar{t}_{0}}}{1 - q} + \frac{\rho^2\eta^2q_1\bar{q}M_h^2(S-1)}{(1-q)^2(1-q_1q^I)}  + \bar{q}_1\rho^2\sum_{s  = 2}^{S}\sum_{j=1}^{s-1}\frac{q_1^{s-1-j}q^{(s-1-j)I}\hat{A}_{\bar{t}_j}}{1 - q}  +  \frac{\bar{q}\rho^2\eta^2S(I- 1)M_h^2 }{1 - q} \nonumber \\
    & \overset{(a)}{\leq} \frac{B_{\bar{t}_{0}}}{1 - q} + \sum_{s  = 2}^{S}\frac{q_1^{s-1}q^{(s-1)I}B_{\bar{t}_{0}}}{1 - q} + \frac{\rho^2\eta^2q_1\bar{q}M_h^2(S-1)}{(1-q)^2(1-q_1q^I)}  + \bar{q}_1\rho^2\sum_{j=1}^{S-1}\sum_{s=1}^{S-j}\frac{q_1^{s}q^{sI}\hat{A}_{\bar{t}_j}}{1 - q}  +  \frac{\bar{q}\rho^2\eta^2S(I- 1)M_h^2 }{1 - q}  \nonumber \\
    & \leq \frac{B_{\bar{t}_{0}}}{1 - q} + \frac{B_{\bar{t}_{0}}}{(1 - q)(1 -q_1q^I)} + \frac{\rho^2\eta^2q_1\bar{q}M_h^2(S-1)}{(1-q)^2(1-q_1q^I)}  + \bar{q}_1\rho^2\sum_{j=1}^{S-1}\frac{q_1q^I\hat{A}_{\bar{t}_j}}{(1 - q)(1 - q_1q^I)}  +  \frac{\bar{q}\rho^2\eta^2TM_h^2 }{1 - q} \nonumber \\
    & \leq \frac{B_{\bar{t}_{0}}}{1 - q} + \frac{B_{\bar{t}_{0}}}{(1 - q)(1 -q_1q^I)} + \frac{\rho^2\eta^2q_1\bar{q}M_h^2(S-1)}{(1-q)^2(1-q_1q^I)}  + \frac{\bar{q}_1\rho^2q_1q^I}{(1 - q)(1 - q_1q^I)}\sum_{j=1}^{S-1}\hat{A}_{\bar{t}_j}  +  \frac{\bar{q}\rho^2\eta^2TM_h^2 }{1 - q} \nonumber \\
    & \leq \frac{B_{\bar{t}_{0}}}{1 - q} + \frac{B_{\bar{t}_{0}}}{(1 - q)(1 -q_1q^I)} + \frac{\rho^2\eta^2q_1\bar{q}M_h^2(S-1)}{(1-q)^2(1-q_1q^I)}  + \frac{\bar{q}_1\rho^2q_1q^I}{(1 - q)(1 - q_1q^I)}\sum_{t=1}^{T}\hat{A}_{\bar{t}_j}  +  \frac{\bar{q}\rho^2\eta^2TM_h^2 }{1 - q} \nonumber \\
\end{align*}
where inequality (a) rearranges the terms in the fourth sum term. This completes the proof.


\end{proof}

\subsection{Bound for Client Drift}
\begin{lemma}
\label{lem: ErrorAccumulation_Iterates_FedAvg}
For $\eta < min\bigg(\frac{\sqrt{(1-q)(1-q_1q^I)}}{2\Gamma\rho I\sqrt{\bar{q}_1q_1q^I }}, \frac{1}{12L_hI}, \frac{\mu\gamma}{2}, 1\bigg)$ and $\gamma < \frac{1}{L}$, then we have:
\begin{align*}
    \sum_{t=1}^{T}  \hat{A}_t \leq \frac{6SL_h^2 I^2\eta^2B_{\bar{t}_{0}}}{1 - q} + \frac{18(S-1)L_h^2 I^2\rho^2\eta^2M_h^2}{(1 - q)^2(1 - q_1q^I)}  +  \frac{12L_h^2 TI\rho^2\eta^2(I- 1)M_h^2 }{1 - q}   +   18TI^2  \zeta^2\eta^2 
\end{align*}
where $\hat{A}_t$, $q$, $q_1$, $\bar{q}_1$, $M_h$ are defined as in Lemma~\ref{lemma: inner_drift}.
\end{lemma}
\begin{proof}
Note from Algorithm and the definition of $\bar{t}_s$ that at $t = \bar{t}_{s - 1}$ with $s \in [S]$, $x_{t}^{(m)} = \bar{x}_{t}$, for all $k$. 
For $t \in [\bar{t}_{s-1} + 1,  \bar{t}_s]$, with $s \in [S]$, we have: $\hat{x}_{t}^{(m)} = \hat{x}_{t-1}^{(m)} - \eta  \nu_{t-1}^{(m)}$, this implies that:
\begin{align*}
    \hat{x}_t^{(m)} = x_{\bar{t}_{s-1}}^{(m)} - \sum_{\ell = \bar{t}_{s-1}}^{t-1} \eta  \nu_\ell^{(m)} \quad \text{and} \quad \bar{x}_{t}  = \bar{x}_{\bar{t}_{s-1}}  - \sum_{\ell = \bar{t}_{s-1}}^{t-1} \eta  \bar{\nu}_\ell.
\end{align*}
So for $t \in [\bar{t}_{s-1} + 1,  \bar{t}_s - 1]$, with $s \in [S]$ we have:
\begin{align}
\label{Eq: ConsensusError_FedAvg}
\frac{1}{M} \sum_{m = 1}^M  \| \hat{x}_t^{(m)}-  \bar{x}_t \|^2 & = \frac{1}{M} \sum_{m = 1}^M \Big\| x_{\bar{t}_{s-1}}^{(m)} - \bar{x}_{\bar{t}_{s-1}}  - \Big( \sum_{\ell = \bar{t}_{s-1}}^{t-1} \eta  \nu_\ell^{(m)} -   \sum_{\ell =  \bar{t}_{s-1}}^{t-1} \eta  \bar{\nu}_\ell  \Big) \Big\|^2 \nonumber \overset{(a)}{=} \frac{1}{M} \sum_{m = 1}^M \Big\|  \sum_{\ell = \bar{t}_{s-1}}^{t-1} \eta\big(  \nu_\ell^{(m)} -      \bar{\nu}_\ell  \big) \Big\|^2  \nonumber\\
& \overset{(b)}{\leq} \frac{2}{M}  \sum_{m = 1}^M \bigg\|   \sum_{\ell = \bar{t}_{s-1}}^{t-1}  \eta \bigg(\nu_\ell^{(m)} - \nabla h^{(m)} (x_\ell^{(m)})  \bigg) \nonumber - \frac{1}{M} \sum_{j = 1}^M \bigg(  \nu_\ell^{(j)} - \nabla h^{(j)} (x_\ell^{(j)})\bigg)\bigg\|^2 \nonumber\\
&\qquad \qquad \qquad \qquad  + \frac{2}{M} \sum_{m = 1}^M \bigg\| \sum_{\ell = \bar{t}_{s-1}}^{t-1}  \eta\bigg( \nabla h^{(m)} (x_\ell^{(m)})  - \frac{1}{M} \sum_{j = 1}^M \nabla h^{(j)} (x_\ell^{(j)}) \bigg) \bigg\|^2  \nonumber\\
& \overset{(c)}{\leq} \frac{2}{M}  \sum_{m = 1}^M \bigg\|   \sum_{\ell = \bar{t}_{s-1}}^{t-1}  \eta \bigg(\nu_\ell^{(m)} - \nabla h^{(m)} (x_\ell^{(m)})  \bigg) \nonumber \bigg\|^2 \nonumber\\
&\qquad \qquad \qquad \qquad  + \frac{2}{M} \sum_{m = 1}^M \bigg\| \sum_{\ell = \bar{t}_{s-1}}^{t-1}  \eta\bigg( \nabla h^{(m)} (x_\ell^{(m)})  - \frac{1}{M} \sum_{j = 1}^M \nabla h^{(j)} (x_\ell^{(j)}) \bigg) \bigg\|^2  \nonumber\\
\end{align}
where the equality $(a)$ follows from the fact that $x_{\bar{t}_{s-1}}^{(m)} = \bar{x}_{\bar{t}_{s-1}}$ for $t = \bar{t}_{s - 1}$; $(b)$ uses triangle inequality and $(c)$ follows from the application of Proposition~\ref{prop: Sum_Mean_Kron}. Then for the first term of \eqref{Eq: ConsensusError_FedAvg}, we have:
\begin{align}
     \bigg\|\sum_{\ell = \bar{t}_{s-1}}^{t-1}  \eta \bigg(\nu_\ell^{(m)} - \nabla h^{(m)} (x_\ell^{(m)})  \bigg) \nonumber \bigg\|^2 & \leq  I\eta^2\sum_{\ell  = \bar{t}_{s-1}}^{t-1}  \bigg\|  \bigg(\nu_\ell^{(m)} - \nabla h^{(m)} (x_\ell^{(m)})  \bigg) \nonumber     \bigg\|^2 \nonumber \\
  & \overset{(a)}{\leq} L_h^2 I \eta^2\sum_{\ell  = \bar{t}_{s-1}}^{t-1}  \bigg\| y^{(m)}_\ell - y^{(m)}_{x^{(m)}_{\ell}} \bigg\|^2 \nonumber \\
\label{Eq: GradVar_FedAvg}
\end{align}
where $(a)$ Follows Proposition~\ref{some smoothness}. Next, for the second term of \eqref{Eq: ConsensusError_FedAvg} we have:
\begin{align}
&  \sum_{m = 1}^M  \bigg\| \sum_{\ell = \bar{t}_{s-1}}^{t-1}  \eta\bigg( \nabla h^{(m)} (x_\ell^{(m)})  - \frac{1}{M} \sum_{j = 1}^M \nabla h^{(j)} (x_\ell^{(j)}) \bigg) \bigg\|^2   \nonumber\\
& \overset{(a)}{\leq} I\sum_{\ell = \bar{t}_{s-1}}^{t-1} \sum_{m = 1}^M  \bigg\| \eta\bigg( \nabla h^{(m)} (x_\ell^{(m)})  - \frac{1}{M} \sum_{j = 1}^M \nabla h^{(j)} (x_\ell^{(j)}) \bigg) \bigg\|^2  \nonumber\\
& \overset{(b)}{\leq} I\sum_{\ell = \bar{t}_{s-1}}^{t-1} \eta^2     \bigg[3\sum_{m = 1}^M       \big\|    \nabla h^{(m)} (x_\ell^{(m)})  -  \nabla h^{(m)}(\bar{x}_\ell)   \big\|^2 + 3\sum_{m = 1}^M  \bigg\|    \nabla h (\bar{x}_\ell)  - \frac{1}{M} \sum_{j = 1}^M \nabla h^{(j)} (x_\ell^{(j)})   \bigg\|^2 \nonumber\\
&   \quad \qquad \qquad \qquad \quad +  3\sum_{m = 1}^M  \big\|    \nabla h^{(m)} (\bar{x}_\ell)  -   \nabla h (\bar{x}_\ell)   \big\|^2 \bigg]  \nonumber\\
&  \overset{(c)}{\leq} I \sum_{\ell = \bar{t}_{s-1}}^{t-1} \eta^2 \bigg[ 6L_h^2 \sum_{m = 1}^M       \big\|x_\ell^{(m)}  -  \bar{x}_\ell \big\|^2  +  3\sum_{m = 1}^M  \bigg\|    \nabla h^{(m)} (\bar{x}_\ell)  -   \frac{1}{M} \sum_{j = 1}^M \nabla h^{(j)} (\bar{x}_\ell)   \bigg\|^2 \bigg]  \nonumber\\
& \overset{(d)}{\leq} 6L_h^2 I \eta^2\sum_{\ell = \bar{t}_{s-1}}^{t-1}     \sum_{m = 1}^M      \big\|x_\ell^{(m)}  -  \bar{x}_\ell \big\|^2  +  3MI^2\eta^2  \zeta^2,
\label{Eq: InterNodeVar_FedAvg}
\end{align}
where $(a)$ utilizes the fact that $t - \bar{t}_{s - 1} \leq I$ for $t \in [\bar{t}_{s- 1} + 1, \bar{t}_s]$ and the generalized triangle inequality; $(b)$ follows the generalized triangle inequality; $(c)$ follows from the $L_h$ lipschitzness of $h$; and $(d)$ utilizes the heterogeneity Assumption~\ref{assumption:hetero} and also the fact that $t - \bar{t}_{s - 1}  \leq I$ for $t \in [\bar{t}_{s- 1} + 1 , \bar{t}_s]$. 

Substituting \eqref{Eq: GradVar_FedAvg} and  \eqref{Eq: InterNodeVar_FedAvg} in \eqref{Eq: ConsensusError_FedAvg} we get:
\begin{align*}
     \frac{1}{M} \sum_{m = 1}^M  \| \hat{x}_t^{(m)}-  \bar{x}_t \|^2 & \leq  \frac{2L_h^2 I\eta^2}{M}  \sum_{m = 1}^M \sum_{\ell  = \bar{t}_{s-1}}^{t-1}  \bigg\| y^{(m)}_\ell - y^{(m)}_{x^{(m)}_{\ell}} \bigg\|^2  + 12L_h^2 I\eta^2  \sum_{\ell = \bar{t}_{s-1}}^{t - 1} \frac{1}{M} \sum_{m = 1}^M  \| x_\ell^{(m)} - \bar{x}_\ell \|^2 +   6I^2  \zeta^2  \eta^2.
\end{align*}
Next we use $\hat{A}_t = \frac{1}{M} \sum_{m = 1}^M  \| \hat{x}_t^{(m)}-  \bar{x}_t \|^2$ and $B_t = \frac{1}{M}\sum_{m=1}^M \bigg\|y^{(m)}_{\bar{t}_s} - y^{(m)}_{x^{(m)}_{\bar{t}_s}} \bigg\|^2$ as in Lemma~\ref{lemma: inner_drift}, then the above inequality can be simplified to:
\begin{align*}
     \hat{A}_t & \leq  2L_h^2 I\eta^2\sum_{\ell  = \bar{t}_{s-1}}^{t-1} B_l    + 12 L_h^2 I\eta^2  \sum_{\ell = \bar{t}_{s-1}}^{t - 1} \hat{A}_l + 6I^2\zeta^2\eta^2.
\end{align*}
For $s\ge 2$, we substitute Eq.~(\ref{eq:b_sum}) to get:
\begin{align*}
     \hat{A}_t & \leq  2L_h^2 I\eta^2\sum_{t^{'}  = \bar{t}_{s-1}}^{t-1} B_{t^{'}}  + 12 L^2I \eta^2 \sum_{\ell = \bar{t}_{s-1}}^{t - 1} \hat{A}_l +   6I^2  \zeta^2\eta^2 \\
     & \leq \frac{2L_h^2 Iq_1^{s-1}q^{(s-1)I}B_{\bar{t}_{0}}\eta^2}{1 - q} + \frac{2L_h^2 I\rho^2\eta^4q_1\bar{q}M_h^2}{(1 - q)^2(1 - q_1q^I)}  + 2L_h^2 I\bar{q}_1\rho^2\eta^2\sum_{j=1}^{s-1}\frac{q_1^{s-1-j}q^{(s-1-j)I}\hat{A}_{\bar{t}_j}}{1 - q}  +  \frac{2L_h^2 I\bar{q}\rho^2\eta^4(I- 1)M_h^2 }{1 - q} \nonumber \\
     & \quad + 6I^2  \zeta^2\eta^2 + 12 L_h^2I \eta^2 \sum_{\ell = \bar{t}_{s-1}}^{t - 1} \hat{A}_l \\
\end{align*}

Summing both sides from $t = \bar{t}_{s - 1} + 1$ to $\bar{t}_{s}$, we get:
\begin{align*}
    \sum_{t = \bar{t}_{s-1} + 1}^{\bar{t}_s} \hat{A}_t & \leq  \frac{2L_h^2 I^2q_1^{s-1}q^{(s-1)I}B_{\bar{t}_{0}}\eta^2}{1 - q} + \frac{2L_h^2 I^2\rho^2\eta^4q_1\bar{q}M_h^2}{(1 - q)^2(1 - q_1q^I)}  +  \frac{2L_h^2 I^2\bar{q}\rho^2\eta^4(I- 1)M_h^2 }{1 - q}   +   6I^3  \zeta^2\eta^2\\
    & \qquad + 2L_h^2 I\bar{q}_1\rho^2\eta^2\sum_{j=1}^{s-1}\frac{q_1^{s-1-j}q^{(s-1-j)I}\hat{A}_{\bar{t}_j}}{1 - q} +  12 L_h^2 I \eta^2 \sum_{t = \bar{t}_{s-1} + 1}^{\bar{t}_s}\sum_{\ell = \bar{t}_{s-1}}^{t - 1} \hat{A}_l \\
    & \leq   \frac{2L_h^2 I^2q_1^{s-1}q^{(s-1)I}B_{\bar{t}_{0}}\eta^2}{1 - q} + \frac{2L_h^2 I^2\rho^2\eta^4q_1\bar{q}M_h^2}{(1 - q)^2(1 - q_1q^I)}  +  \frac{2L_h^2 I^2\bar{q}\rho^2\eta^4(I- 1)M_h^2 }{1 - q}   +   6I^3  \zeta^2\eta^2\\
    & \qquad + 2L_h^2 I\bar{q}_1\rho^2\eta^2\sum_{j=1}^{s-1}\frac{q_1^{s-1-j}q^{(s-1-j)I}\hat{A}_{\bar{t}_j}}{1 - q} +  12 L_h^2 I^2 \eta^2 \sum_{\ell = \bar{t}_{s-1}}^{\bar{t}_s} \hat{A}_l \\
\end{align*}
For ease of notation, we denote $C_s = \sum_{t = \bar{t}_{s-1} + 1}^{\bar{t}_s} \hat{A}_t$, then we have:
\begin{align}
    C_s  & \leq  \frac{2L_h^2 I^2q_1^{s-1}q^{(s-1)I}B_{\bar{t}_{0}}\eta^2}{1 - q} + \frac{2L_h^2 I^2\rho^2\eta^4q_1\bar{q}M_h^2}{(1 - q)^2(1 - q_1q^I)}  +  \frac{2L_h^2 I^2\bar{q}\rho^2\eta^4(I- 1)M_h^2 }{1 - q}   +   6I^3  \zeta^2\eta^2\nonumber\\
    & \qquad + 2L_h^2 I\bar{q}_1\rho^2\eta^2\sum_{j=1}^{s-1}\frac{q_1^{s-1-j}q^{(s-1-j)I}C_j}{1 - q} +  12 L_h^2 I^2 \eta^2 C_s \nonumber\\
    & \leq\frac{2L_h^2 I^2B_{\bar{t}_{0}}\eta^2}{1 - q} + \frac{2L_h^2 I^2\rho^2\eta^4q_1\bar{q}M_h^2}{(1 - q)^2(1 - q_1q^I)}  +  \frac{2L_h^2 I^2\bar{q}\rho^2\eta^4(I- 1)M_h^2 }{1 - q}   +   6I^3  \zeta^2\eta^2\nonumber\\
    & \qquad + 2L_h^2 I\bar{q}_1\rho^2\eta^2\sum_{j=1}^{s-1}\frac{q_1^{s-1-j}q^{(s-1-j)I}C_j}{1 - q} +  12 L_h^2 I^2 \eta^2 C_s \nonumber\\
\label{eq:c_sum}
\end{align}
The second inequality follows that $q_1q = (1 + \mu\gamma/4)(1 - \mu\gamma/2) \leq 1 - \mu\gamma/4 < 1$.
Next for $s=1$, substitute Eq.~\ref{eq:case1}, we have:
\begin{align*}
     \hat{A}_t & \leq  2L_h^2 I\eta^2\sum_{\ell  = \bar{t}_{0}}^{t-1} B_l    + 12 L_h^2 I\eta^2  \sum_{\ell = \bar{t}_{s-1}}^{t - 1} \hat{A}_l + 6I^2\zeta^2\eta^2 \nonumber \\
     & \leq \frac{2L_h^2 I\eta^2B_{\bar{t}_{0}}}{1 - q} + \frac{2\bar{q}L_h^2 I\eta^2\rho^2\eta^2(I- 1)M_h^2 }{1 - q} + 12 L_h^2\Gamma I\eta^2  \sum_{\ell = \bar{t}_{s-1}}^{t - 1} \hat{A}_l + 6I^2\zeta^2\eta^2 \nonumber \\
\end{align*}
Summing both sides from $t = \bar{t}_{0} + 1$ to $\bar{t}_{s}$, we get:
\begin{align}
    C_1 = \sum_{t = \bar{t}_{0} + 1}^{\bar{t}_s} \hat{A}_t & \leq \frac{2L_h^2 I^2\eta^2B_{\bar{t}_{0}}}{1 - q} + \frac{2\bar{q}L_h^2 I^2\rho^2\eta^4(I- 1)M_h^2 }{1 - q} +  6I^3  \zeta^2\eta^2  +   12L_h^2I^2 \eta^2 C_s \nonumber\\
    \label{eq:c_sum1}
\end{align}
Then we combine \eqref{eq:c_sum} and \eqref{eq:c_sum1} to have:
\begin{align*}
    \sum_{s=1}^{S}  C_s & \leq  \frac{2SL_h^2 I^2\eta^2B_{\bar{t}_{0}}}{1 - q} + \frac{2(S-1)L_h^2 I^2\rho^2\eta^4q_1\bar{q}M_h^2}{(1 - q)^2(1 - q_1q^I)}  +  \frac{2SL_h^2 I^2\bar{q}\rho^2\eta^4(I- 1)M_h^2 }{1 - q}   +   6SI^3  \zeta^2\eta^2\\
    & \qquad + 2\bar{q}_1L_h^2 I^2 \eta^2\rho^2\sum_{s=2}^{S}\sum_{j=1}^{s-1}\frac{(q_1q^{I})^{s-1 -j}C_j}{1 - q} +  12 L_h^2 I^2 \eta^2 \sum_{s=1}^{S}C_s \\  
    & \leq  \frac{2SL_h^2 I^2\eta^2B_{\bar{t}_{0}}}{1 - q} + \frac{2(S-1)L_h^2 I^2\rho^2\eta^4q_1\bar{q}M_h^2}{(1 - q)^2(1 - q_1q^I)}  +  \frac{2L_h^2 TI\bar{q}\rho^2\eta^4(I- 1)M_h^2 }{1 - q}   +   6SI^3  \zeta^2\eta^2\\
    & \qquad +  \frac{2\bar{q}_1L_h^2 I^2\eta^2\rho^2q_1q^I}{(1 - q)(1 - q_1q^I)}\sum_{s=1}^{S}C_s +  12 L_h^2 I^2 \eta^2 \sum_{s=1}^{S}C_s \\  
\end{align*}
and by rearranging the terms, we have:
\begin{align*}
     \bigg(1 -  \frac{2\bar{q}_1L_h^2 I^2\eta^2\rho^2q_1q^I}{(1 - q)(1 - q_1q^I)}-  12 L_h^2I^2 \eta^2 \bigg)\sum_{t=1}^{T}  \hat{A}_t & \leq \frac{2SL_h^2 I^2\eta^2B_{\bar{t}_{0}}}{1 - q} + \frac{2(S-1)L_h^2 I^2\rho^2\eta^4q_1\bar{q}M_h^2}{(1 - q)^2(1 - q_1q^I)} \nonumber\\
     & \qquad \qquad +  \frac{2L_h^2 TI\bar{q}\rho^2\eta^4(I- 1)M_h^2 }{1 - q}   +   6SI^3  \zeta^2\eta^2 \\
\end{align*}
Suppose $\eta < min\bigg(\frac{\sqrt{(1-q)(1-q_1q^I)}}{2\Gamma\rho I\sqrt{\bar{q}_1q_1q^I }}, \frac{1}{12L_hI}\bigg)$, then we have 
\begin{align*}
    1 -  \frac{2\bar{q}_1L_h^2 I^2\eta^2\rho^2q_1q^I}{(1 - q)(1 - q_1q^I)}-  12 L_h^2I^2 \eta^2 \ge 1 - \frac{1}{2} - \frac{1}{12} > \frac{1}{3}
\end{align*}
So we have:
\begin{align*}
    \sum_{t=1}^{T}  \hat{A}_t \leq \frac{6SL_h^2 I^2\eta^2B_{\bar{t}_{0}}}{1 - q} + \frac{6(S-1)L_h^2 I^2\rho^2\eta^4q_1\bar{q}M_h^2}{(1 - q)^2(1 - q_1q^I)}  +  \frac{6L_h^2 TI\bar{q}\rho^2\eta^4(I- 1)M_h^2 }{1 - q}   +   18SI^3  \zeta^2\eta^2 
\end{align*}
Note that we have:
\begin{align*}
    q_1\bar{q} = (1 + \frac{\mu\gamma}{4})(1 + \frac{2}{\mu\gamma}) = \frac{3}{2} + \frac{\mu\gamma}{4} + \frac{2}{\mu\gamma} < 2 + \frac{2}{\mu\gamma}
\end{align*}
and by the assumption that $\eta < min(1, \frac{\mu\gamma}{2})$, we simplify the above inequality as:
\begin{align*}
    \sum_{t=1}^{T}  \hat{A}_t \leq \frac{6SL_h^2 I^2\eta^2B_{\bar{t}_{0}}}{1 - q} + \frac{18(S-1)L_h^2 I^2\rho^2\eta^2M_h^2}{(1 - q)^2(1 - q_1q^I)}  +  \frac{12L_h^2 TI\rho^2\eta^2(I- 1)M_h^2 }{1 - q}   +   18TI^2  \zeta^2\eta^2 
\end{align*}
Therefore, the lemma is proved. 
\end{proof}

\subsection{Descent Lemma}
\begin{lemma}[Descent Lemma]
\label{lemma:desent}
For all $t \in [\bar{t}_{s-1}, \bar{t}_s - 1]$ and $s \in [S]$, the iterates generated satisfy:
\begin{align*}
     h(\bar{x}_{t + 1}) & \leq     h(\bar{x}_{t }) - \frac{\eta}{2}  \|\nabla h(\bar{x}_t) \|^2  + \frac{L_h^2(1 + 2\rho^2)\eta}{2M}\sum_{m=1}^M\|x_t^{(m)} - \bar{x}_t \|^2 + \frac{L_h^2\eta}{2M}\sum_{m=1}^M\| y^{(m)}_t - y^{(m)}_{x^{(m)}_{t}} \|^2
\end{align*}
where the expectation is w.r.t the stochasticity of the algorithm.
\end{lemma}
\begin{proof}
Using the smoothness of $f$ we have:
\begin{align*}
    h(\bar{x}_{t + 1}) 
    & \leq   h(\bar{x}_{t }) + \langle \nabla h(\bar{x}_{t}),  \bar{x}_{t + 1} - \bar{x}_{t}\rangle + \frac{L_h}{2} \| \bar{x}_{t + 1} - \bar{x}_{t } \|^2  \nonumber\\
    &  \overset{(a)}{=}  h(\bar{x}_{t}) - \eta \langle \nabla h(\bar{x}_{t}),  \bar{\nu}_t \rangle + \frac{\eta^2 L_h}{2} \| \bar{\nu}_{t}  \|^2  \nonumber\\
    & \overset{(b)}{=}      h(\bar{x}_{t}) - \frac{\eta}{2}  \Big\| \bar{\nu}_{t}  \Big\|^2  - \frac{\eta}{2} \| \nabla h(\bar{x}_{t}) \|^2   + \frac{\eta}{2} \Big\|  \nabla h(\bar{x}_{t})  - \bar{\nu}_{t}   \Big\|^2 + \frac{\eta_{t}^2 L_h}{2} \Big\| \bar{\nu}_{t} \Big\|^2   \nonumber \\  
    & =      h(\bar{x}_{t }) -  \left( \frac{\eta}{2} -  \frac{\eta^2 L_h}{2}  \right)  \Big\| \bar{\nu}_t  \Big\|^2 - \frac{\eta}{2} \|\nabla h(\bar{x}_t) \|^2  + \frac{\eta}{2} \Big\|  \nabla h(\bar{x}_{t})  - \bar{\nu}_{t}   \Big\|^2  \nonumber \\
    & \overset{(c)}{\leq}      h(\bar{x}_{t }) - \frac{\eta}{2} \|\nabla h(\bar{x}_t) \|^2  + \frac{\eta}{2} \Big\|  \nabla h(\bar{x}_{t})  - \bar{\nu}_{t}   \Big\|^2  \nonumber \\
    & \overset{(d)}{\leq}      h(\bar{x}_{t }) - \frac{\eta}{2} \|\nabla h(\bar{x}_t) \|^2  + \frac{L_h^2\eta}{2M}\sum_{m=1}^M \bigg( \bigg(1 + 2\rho^2\bigg)\bigg\|x_t^{(m)} - \bar{x}_t \bigg\|^2 + \bigg\| y^{(m)}_t - y^{(m)}_{x^{(m)}_{t}} \bigg\|^2 \bigg)  \nonumber \\
\end{align*}
 where equality $(a)$ follows from the iterate update given in Step 6 of Algorithm~\ref{alg:FedBiO}; $(b)$ uses $\langle a , b \rangle = \frac{1}{2} [\|a\|^2 + \|b\|^2 - \|a - b \|^2]$;  (c) follows the assumption that  $\eta < 1/L_h$; (d) follows lemma~\ref{lemma:hg_bound}
Hence, the lemma is proved.
\end{proof}

\subsection{Proof of Convergence Theorem}
\label{appendix:fedbio}
\begin{theorem}
\label{theorem:FedBiO}
For $\delta < min\bigg(\frac{\sqrt{(1-q)(1-q_1q^I)}}{2\Gamma\rho I\sqrt{\bar{q}_1q_1q^I }}, \frac{1}{12L_hI}, \frac{\mu\gamma}{2}, 1\bigg)$, $\gamma < \frac{1}{L}$ and $\eta = \frac{\delta}{\sqrt{T}}$, we have:
\begin{align*}
  \frac{1}{T}  \sum_{t = 1}^T \|\nabla h(\bar{x}_t)\|^2  & \leq \frac{2(h(\bar{x}_t) - h^\ast)}{\delta\sqrt{T}} + \frac{L_h^2 B_{\bar{t}_{0}}}{(1 - q)T} + \frac{2L_h^2 B_{\bar{t}_{0}}}{(1 - q)(1 -q_1q^I)T} + \frac{\delta^2L_h^2 \rho^2q_1\bar{q}M_h^2}{(1-q)^2(1-q^I)IT} + \frac{\delta^2L_h^2\bar{q}\rho^2M_h^2 }{(1 - q)T}\\  & \qquad + \bigg(\frac{\bar{q}_1L_h^2\rho^2q_1q^I}{(1 - q)(1 - q_1q^I)} + L_h^2(1 + 2\rho^2)\bigg)\bigg(\frac{6SL_h^2 I^2B_{\bar{t}_{0}}}{1 - q} + \frac{18(S-1)L_h^2 I^2\rho^2M_h^2}{(1 - q)^2(1 - q_1q^I)}\nonumber \\
   & \qquad +  \frac{12L_h^2 TI\rho^2(I- 1)M_h^2 }{1 - q}   +  18TI^2\zeta^2 \bigg)\frac{\delta^2}{T}
\end{align*}
where $B_{\bar{t}_0} = \frac{1}{M}\sum_{m=1}^M \|y^{(m)}_{1} - y^{(m)}_{x^{(m)}_{1}} \|^2$, $q = (1 - \frac{\mu\gamma}{2})$, $\bar{q} = (1 + \frac{2}{\mu\gamma})$, $q_1 = 1 + \frac{\mu\gamma}{4}$ and $\bar{q}_1 = 1 + \frac{4}{\mu\gamma}$
\end{theorem}
\begin{proof}
From the result of Lemma~\ref{lemma:desent}, we su
for $t = [T]$ and multiply both sides by $2/\eta T$ we get
\begin{align*}
   \frac{1}{T}  \sum_{t = 1}^T \|\nabla h(\bar{x}_t)\|^2  & \leq \sum_{t = 1}^T \frac{2 (h(\bar{x}_t) - h(\bar{x}_{t+1}))}{\eta T} + \frac{L_h^2(1 + 2\rho^2)}{MT}\sum_{t = 1}^T\sum_{m=1}^M\|x_t^{(m)} - \bar{x}_t \|^2 + \frac{L_h^2}{MT}\sum_{t = 1}^T\sum_{m=1}^M\| y^{(m)}_t - y^{(m)}_{x^{(m)}_{t}} \|^2 \\
   & \leq  \frac{2 (h(\bar{x}_t) - h^\ast)}{\eta T} + \frac{L_h^2(1 + 2\rho^2)}{MT}\sum_{t = 1}^T\sum_{m=1}^M\|\hat{x}_t^{(m)} - \bar{x}_t \|^2 + \frac{L_h^2}{MT}\sum_{t = 1}^T\sum_{m=1}^M\| y^{(m)}_t - y^{(m)}_{x^{(m)}_{t}} \|^2 \\
   & \leq \frac{2 (h(\bar{x}_t) - h^\ast)}{\eta T} + \frac{L_h^2B_{\bar{t}_{0}}}{(1 - q)T} + \frac{L_h^2B_{\bar{t}_{0}}}{(1 - q)(1 -q_1q^I)T} + \frac{L_h^2\rho^2\eta^2q_1\bar{q}M_h^2(S-1)}{(1-q)^2(1-q_1q^I)T} \nonumber \\
    & \qquad + \frac{L_h^2\bar{q}_1\rho^2q_1q^I}{MT(1 - q)(1 - q_1q^I)}\sum_{t=1}^{T}\sum_{m=1}^M\|\hat{x}_t^{(m)} - \bar{x}_t \|^2  +  \frac{L_h^2\bar{q}\rho^2\eta^2M_h^2 }{1 - q} + \frac{L_h^2(1 + 2\rho^2)}{MT}\sum_{t = 1}^T\sum_{m=1}^M\|\hat{x}_t^{(m)} - \bar{x}_t \|^2 \\
   & \leq \frac{2 (h(\bar{x}_t) - h^\ast)}{\eta T} + \frac{L_h^2 B_{\bar{t}_{0}}}{(1 - q)T} + \frac{L_h^2 B_{\bar{t}_{0}}}{(1 - q)(1 -q_1q^I)T} + \frac{L_h^2 \rho^2\eta^2q_1\bar{q}M_h^2(S-1)}{(1-q)^2(1-q^I)T} + \frac{L_h^2\bar{q}\rho^2\eta^2M_h^2 }{1 - q} \\  & \qquad + \bigg(\frac{\bar{q}_1L_h^2\rho^2q_1q^I}{(1 - q)(1 - q_1q^I)} + L_h^2(1 + 2\rho^2)\bigg)\bigg(\frac{6SL_h^2 I^2B_{\bar{t}_{0}}}{1 - q} + \frac{18(S-1)L_h^2 I^2\rho^2M_h^2}{(1 - q)^2(1 - q_1q^I)}\nonumber \\
   & \qquad +  \frac{12L_h^2 TI\rho^2(I- 1)M_h^2 }{1 - q}   +  18TI^2\zeta^2 \bigg)\eta^2 \nonumber \\
\end{align*}
where the second inequality uses $f(\bar{x}_{t - 1}) \geq f^\ast$ and the fact $\sum_{m=1}^M\|x_t^{(m)} - \bar{x}_t \|^2 \leq \sum_{m=1}^M\|\hat{x}_t^{(m)} - \bar{x}_t \|^2$ for all $t$. The third inequality uses Lemma~\ref{lemma: inner_drift} and the fourth inequality uses \ref{lem: ErrorAccumulation_Iterates_FedAvg}. Finally, choice of $\eta = \frac{\delta}{\sqrt{T}}$, $\delta$ is a constant such that $\delta < min\bigg(\frac{\sqrt{(1-q)(1-q_1q^I)}}{2\Gamma\rho I\sqrt{\bar{q}_1q_1q^I }}, \frac{1}{12L_hI}, \frac{\mu\gamma}{2}, 1\bigg)$ we get
\begin{align*}
  \frac{1}{T}  \sum_{t = 1}^T \|\nabla h(\bar{x}_t)\|^2  & \leq \frac{2 (h(\bar{x}_t) - h^\ast)}{\delta\sqrt{T}} + \frac{L_h^2 B_{\bar{t}_{0}}}{(1 - q)T} + \frac{2L_h^2 B_{\bar{t}_{0}}}{(1 - q)(1 -q_1q^I)T} + \frac{\delta^2L_h^2 \rho^2q_1\bar{q}M_h^2}{(1-q)^2(1-q^I)IT} + \frac{\delta^2L_h^2\bar{q}\rho^2M_h^2 }{(1 - q)T}\\  & \qquad + \bigg(\frac{\bar{q}_1L_h^2\rho^2q_1q^I}{(1 - q)(1 - q_1q^I)} + L_h^2(1 + 2\rho^2)\bigg)\bigg(\frac{6SL_h^2 I^2B_{\bar{t}_{0}}}{1 - q} + \frac{18(S-1)L_h^2 I^2\rho^2M_h^2}{(1 - q)^2(1 - q_1q^I)}\nonumber \\
   & \qquad +  \frac{12L_h^2 TI\rho^2(I- 1)M_h^2 }{1 - q}   +  18TI^2\zeta^2 \bigg)\frac{\delta^2}{T}
\end{align*}
Therefore, we have the theorem. 
\end{proof}

\section{Proof for the FedBiOAcc Algorithm}
In this section, we prove the convergence of the FedBiOAcc Algorithm
\subsection{Bound for Hyper-Gradient Bias}
\begin{lemma}
\label{lemma:hg_bound_storm}
With all assumptions hold and $c_\nu\alpha_t^2 < 1$, then for all $t \in [\bar{t}_{s-1}, \bar{t}_s - 1]$, we have:
\begin{align*}
  &\mathbb{E} \bigg[ \Big\| \bar{\nu}_{t} - \frac{1}{M} \sum_{m=1}^M  \nabla h(x^{(m)}_t)  \Big\|^2\bigg] \leq  ( 1 - c_{\nu}\alpha_{t-1}^2)^2\mathbb{E} \bigg[ \bigg\| \bar{\nu}_{t-1} - \frac{1}{M}\sum_{m=1}^M \nabla h^{(m)}(\bar{x}_{t-1}) \bigg\|^2 \bigg] \nonumber \\
  & \qquad + \frac{4(c_{\nu}\alpha_{t-1}^2)^2\sigma^2}{M} + 8(c_{\nu}\alpha_{t-1}^2)^2G^2  +  \frac{8L_h^2(c_{\nu}\alpha_{t-1}^2)^2}{M}\sum_{m=1}^M\mathbb{E} \bigg[\bigg\| y_{t-1}^{(m)} - y^{(m)}_{x^{(m)}_{t-1}} \bigg\|^2  \bigg] \nonumber \\
  & \qquad +  \frac{40L_h^2\eta^2\alpha_{t-1}^2}{M}\sum_{m=1}^M\mathbb{E}\bigg[\bigg\|\nu^{(m)}_{t-1} - \bar{\nu}_{t-1}\bigg\|^2 \bigg] + \frac{40L_h^2\eta^2\alpha_{t-1}^2}{M}\sum_{m=1}^M\mathbb{E}\bigg[\bigg\|\bar{\nu}_{t-1}\bigg\|^2 \bigg] + \frac{12L_h^2\gamma^2\alpha_{t-1}^2}{M}\sum_{m=1}^M\mathbb{E}\bigg[ \bigg\|\omega^{(m)}_{t-1}\bigg\|^2\bigg] 
\end{align*}
where the expectation is w.r.t the stochasticity of the algorithm.
\end{lemma}

\begin{proof}

\begin{align}
    \mathbb{E} \bigg[ \Big\| \bar{\nu}_{t} - \frac{1}{M} \sum_{m=1}^M  \nabla h(x^{(m)}_t)  \Big\|^2 \bigg]
    &= \mathbb{E}  \bigg[\bigg\| \frac{1}{M}\sum_{m=1}^M \big(\hat{\nu}^{(m)}_t - \nabla h^{(m)}(x^{(m)}_t) \big) \bigg\|^2\bigg] \nonumber \\
    & \leq  \mathbb{E} \bigg[ \bigg\|\frac{1}{M}\sum_{m=1}^M \bigg(\mu_{t}^{(m)} + ( 1 -  c_{\nu}\alpha_{t-1}^2) (\nu_{t-1}^{(m)} - \mu_{t-1}^{(m)}) - \nabla h^{(m)}(x^{(m)}_t)\bigg) \bigg\|^2 \bigg] \nonumber \\
    & = \mathbb{E} \bigg[ \bigg\| ( 1 - c_{\nu}\alpha_{t-1}^2) \bigg(\bar{\nu}_{t-1} - \frac{1}{M}\sum_{m=1}^M \nabla h^{(m)}(x^{(m)}_{t-1})\bigg) \nonumber \\
    & \qquad + \frac{1}{M}\sum_{m=1}^M \bigg(\mu_{t}^{(m)} - \nabla h^{(m)}(x^{(m)}_t) \nonumber +  (1 - c_{\nu}\alpha_{t-1}^2)(\nabla h^{(m)}(x^{(m)}_{t-1}) - \mu_{t-1}^{(m)}) \bigg) \bigg\|^2  \bigg] \nonumber \\
    & \overset{(a)}{\leq} ( 1 - c_{\nu}\alpha_{t-1}^2)^2\mathbb{E} \bigg[ \bigg\| \bar{\nu}_{t-1} - \frac{1}{M}\sum_{m=1}^M \nabla h^{(m)}(\bar{x}_{t-1}) \bigg\|^2 \bigg] \nonumber \\
    & \qquad + \mathbb{E} \bigg[ \bigg\| \frac{1}{M}\sum_{m=1}^M \bigg(\mu_{t}^{(m)} - \nabla h^{(m)}(x^{(m)}_t) \nonumber +  (1 - c_{\nu}\alpha_{t-1}^2)(\nabla h^{(m)}(x^{(m)}_{t-1}) - \mu_{t-1}^{(m)}) \bigg)  \bigg\|^2 \bigg] \nonumber \\
\label{eq:niu_diff_storm}
\end{align}
where inequality $(a)$ uses the fact that the cross product is zero in expectation; Next for the second term of the above equation. Now suppose we denote $\tilde{\mu}_{t}^{(m)} = \mathbb{E}[\mu_{t}^{(m)}]$, then by the triangle inequality, we have:
\begin{align*}
    &\mathbb{E} \bigg[ \bigg\| \frac{1}{M}\sum_{m=1}^M \bigg(\mu_{t}^{(m)} - \nabla h^{(m)}(x^{(m)}_t) \nonumber +  (1 - c_{\nu}\alpha_{t-1}^2)(\nabla h^{(m)}(x^{(m)}_{t-1}) - \mu_{t-1}^{(m)}) \bigg)  \bigg\|^2 \bigg] \nonumber \\
    & \leq 2\mathbb{E} \bigg[ \bigg\| \frac{1}{M}\sum_{m=1}^M \bigg(\mu_{t}^{(m)} - \tilde{\mu}_{t}^{(m)}  +  (1 - c_{\nu}\alpha_{t-1}^2)(\tilde{\mu}_{t-1}^{(m)} - \mu_{t-1}^{(m)}) \bigg)\bigg\|^2 \bigg] \nonumber \\
    & \qquad \qquad + 2\mathbb{E} \bigg[ \bigg\| \frac{1}{M}\sum_{m=1}^M\bigg( \tilde{\mu}_{t}^{(m)} - \nabla h^{(m)}(x^{(m)}_t) \nonumber  +  (1 - c_{\nu}\alpha_{t-1}^2)(\nabla h^{(m)}(x^{(m)}_{t-1}) - \tilde{\mu}_{t-1}^{(m)}) \bigg) \bigg\|^2 \bigg] \nonumber \\
    & \leq \frac{2}{M^2}\sum_{m=1}^M\mathbb{E} \bigg[ \bigg\| \bigg(\mu_{t}^{(m)} - \tilde{\mu}_{t}^{(m)}  +  (1 - c_{\nu}\alpha_{t-1}^2)(\tilde{\mu}_{t-1}^{(m)} - \mu_{t-1}^{(m)}) \bigg)\bigg\|^2 \bigg] \nonumber \\
    & \qquad \qquad + \frac{2}{M}\sum_{m=1}^M\mathbb{E} \bigg[ \bigg\| \bigg( \tilde{\mu}_{t}^{(m)} - \nabla h^{(m)}(x^{(m)}_t) \nonumber  +  (1 - c_{\nu}\alpha_{t-1}^2)(\nabla h^{(m)}(x^{(m)}_{t-1}) - \tilde{\mu}_{t-1}^{(m)}) \bigg) \bigg\|^2 \bigg] \nonumber \\
\end{align*}
The last inequality is by the generalized triangle inequality for the second term, the first term uses the fact that the cross product is zero in expectation. We bound the two terms in the above inequality separately. For the first term, we have:
\begin{align*}
   & 2\mathbb{E} \bigg[ \bigg\| \mu_{t}^{(m)} - \tilde{\mu}_{t}^{(m)}  +  (1 - c_{\nu}\alpha_{t-1}^2)(\tilde{\mu}_{t-1}^{(m)} - \mu_{t-1}^{(m)}) \bigg\|^2 \bigg]\\
   & \overset{(a)}{\leq} 4(c_{\nu}\alpha_{t-1}^2)^2\mathbb{E} \bigg[ \bigg\| \mu_{t}^{(m)} - \tilde{\mu}_{t}^{(m)} \bigg\|^2 \bigg]  +  4(1 - c_{\nu}\alpha_{t-1}^2)^2 \mathbb{E} \bigg[ \bigg\|\mu_{t}^{(m)} - \mu_{t-1}^{(m)} - \tilde{\mu}_{t}^{(m)} + \tilde{\mu}_{t-1}^{(m)}) \bigg\|^2 \bigg] \nonumber \\
   & \overset{(b)}{\leq} 4(c_{\nu}\alpha_{t-1}^2)^2\mathbb{E} \bigg[ \bigg\| \mu_{t}^{(m)} - \tilde{\mu}_{t}^{(m)} \bigg\|^2 \bigg]  +  4(1 - c_{\nu}\alpha_{t-1}^2)^2 \mathbb{E} \bigg[ \bigg\|\mu_{t}^{(m)} - \mu_{t-1}^{(m)} \bigg\|^2 \bigg] \nonumber \\
   & \overset{(c)}{\leq} 4(c_{\nu}\alpha_{t-1}^2)^2\sigma^2  +  4L_h^2(1 - c_{\nu}\alpha_{t-1}^2)^2\mathbb{E}\bigg[\bigg\|x^{(m)}_{t} - x^{(m)}_{t-1}\bigg\|^2 + \bigg\|y^{(m)}_{t} - y^{(m)}_{t-1}\bigg\|^2\bigg] \nonumber \\
   & \leq 4(c_{\nu}\alpha_{t-1}^2)^2\sigma^2  +  4L_h^2(1 - c_{\nu}\alpha_{t-1}^2)^2\mathbb{E}\bigg[\bigg\|\eta\alpha_{t-1}\nu^{(m)}_{t-1}\bigg\|^2 + \bigg\|\gamma\alpha_{t-1}\omega^{(m)}_{t-1}\bigg\|^2\bigg] \nonumber \\
    & \overset{(d)}{\leq} 4(c_{\nu}\alpha_{t-1}^2)^2\sigma^2  +  4L_h^2\mathbb{E}\bigg[\bigg\|\eta\alpha_{t-1}\nu^{(m)}_{t-1}\bigg\|^2 + \bigg\|\gamma\alpha_{t-1}\omega^{(m)}_{t-1}\bigg\|^2\bigg] \nonumber \\
\end{align*}
where inequality (a) follows the triangle inequality Proposition~\ref{prop:generali_tri}; (b) follows Propostion~\ref{prop: Sum_Mean_Kron} due to the definition of $\tilde{\mu}^{(m)}_t$; (c) follows the smoothness property of $L_h$ and the bounded variance assumption~\ref{assumption:outer_noise}; (d) follows the fact that $c_\nu\alpha_t^2 < 1$. Next for the second term, we have:
\begin{align*}
    &2\mathbb{E} \bigg[ \bigg\| \tilde{\mu}_{t}^{(m)} - \nabla h^{(m)}(x^{(m)}_{t}) \nonumber  +  (1 - c_{\nu}\alpha_{t-1}^2)(\nabla h^{(m)}(x^{(m)}_{t-1}) - \tilde{\mu}_{t-1}^{(m)}) \bigg\|^2 \bigg]\\
    & \overset{(a)}{\leq} 4(c_{\nu}\alpha_{t-1}^2)^2\mathbb{E} \bigg[ \bigg\| \tilde{\mu}_{t-1}^{(m)} - \nabla h^{(m)}(x^{(m)}_{t-1}) \bigg \| \bigg] + 8\mathbb{E} \bigg[ \bigg\|\tilde{\mu}_{t}^{(m)} - \tilde{\mu}_{t-1}^{(m)}\bigg\|^2 \bigg] + 8\mathbb{E} \bigg[ \bigg\|\nabla h^{(m)}(x^{(m)}_{t}) -  \nabla h^{(m)}(x^{(m)}_{t-1})\bigg\|^2 \bigg] \nonumber \\
    & \overset{(b)}{\leq} 8(c_{\nu}\alpha_{t-1}^2)^2\mathbb{E} \bigg[ \bigg\| \tilde{\mu}_{t-1}^{(m)} - \Phi^{(m)}(x^{(m)}_{t-1}, y^{(m)}_{t-1}) \bigg \| \bigg] + 8(c_{\nu}\alpha_{t-1}^2)^2\mathbb{E} \bigg[ \bigg\| \Phi^{(m)}(x^{(m)}_{t-1}, y^{(m)}_{t-1}) - \nabla h^{(m)}(x^{(m)}_{t-1}) \bigg \| \bigg]\nonumber \\ 
    & \qquad  +  8\mathbb{E} \bigg[ \bigg\|\tilde{\mu}_{t}^{(m)} - \tilde{\mu}_{t-1}^{(m)}\bigg\|^2 \bigg] + 8\mathbb{E} \bigg[ \bigg\|\nabla h^{(m)}(x^{(m)}_{t}) -  \nabla h^{(m)}(x^{(m)}_{t-1})\bigg\|^2 \bigg] \nonumber \\
    & \overset{(c)}{\leq} 8(c_{\nu}\alpha_{t-1}^2)^2G^2 + 8L_h^2(c_{\nu}\alpha_{t-1}^2)^2 \mathbb{E} \bigg[\bigg(\bigg\| y_{t-1}^{(m)} - y^{(m)}_{x^{(m)}_{t-1}} \bigg\|^2\bigg)  \bigg] \nonumber \\
    & \qquad + 8L_h^2\mathbb{E}\bigg[\bigg\|\eta\alpha_{t-1}\nu^{(m)}_{t-1}\bigg\|^2 + \bigg\|\gamma\alpha_{t-1}\omega^{(m)}_{t-1}\bigg\|^2\bigg] + 8L_h^2 \mathbb{E}\bigg[\bigg\|\eta\alpha_{t-1}\nu^{(m)}_{t-1}\bigg\|^2 \bigg]
\end{align*}
where inequality (a) and (b) follows the generalized triangle inequality; (c) follows the smoothness of $h(x)$ and the bounded bias assumption~\ref{assumption:outer_noise}. Combine everything together, we have:
\begin{align*}
  &\mathbb{E} \bigg[ \Big\| \bar{\nu}_{t} - \frac{1}{M} \sum_{m=1}^M  \nabla h(x^{(m)}_t)  \Big\|^2\bigg] \\ & \leq ( 1 - c_{\nu}\alpha_{t-1}^2)^2\mathbb{E} \bigg[ \bigg\| \bar{\nu}_{t-1} - \frac{1}{M}\sum_{m=1}^M \nabla h^{(m)}(\bar{x}_{t-1}) \bigg\|^2 \bigg] + \frac{4(c_{\nu}\alpha_{t-1}^2)^2\sigma^2}{M} + 8(c_{\nu}\alpha_{t-1}^2)^2G^2 \nonumber \\ & \qquad + \frac{8L_h^2(c_{\nu}\alpha_{t-1}^2)^2}{M}\sum_{m=1}^M\mathbb{E} \bigg[\bigg\| y_{t-1}^{(m)} - y^{(m)}_{x^{(m)}_{t-1}} \bigg\|^2  \bigg] +  \frac{20L_h^2\eta^2\alpha_{t-1}^2}{M}\sum_{m=1}^M\mathbb{E}\bigg[\bigg\|\nu^{(m)}_{t-1}\bigg\|^2 \bigg] \nonumber  + \frac{12L_h^2\gamma^2\alpha_{t-1}^2}{M}\sum_{m=1}^M\mathbb{E}\bigg[ \bigg\|\omega^{(m)}_{t-1}\bigg\|^2\bigg] \\
  & \overset{(a)}{\leq} ( 1 - c_{\nu}\alpha_{t-1}^2)^2\mathbb{E} \bigg[ \bigg\| \bar{\nu}_{t-1} - \frac{1}{M}\sum_{m=1}^M \nabla h^{(m)}(\bar{x}_{t-1}) \bigg\|^2 \bigg] + \frac{4(c_{\nu}\alpha_{t-1}^2)^2\sigma^2}{M} + 8(c_{\nu}\alpha_{t-1}^2)^2G^2 \\
  & \quad +  \frac{8L_h^2(c_{\nu}\alpha_{t-1}^2)^2}{M}\sum_{m=1}^M\mathbb{E} \bigg[\bigg\| y_{t-1}^{(m)} - y^{(m)}_{x^{(m)}_{t-1}} \bigg\|^2  \bigg] \nonumber \\
  & \qquad +  \frac{40L_h^2\eta^2\alpha_{t-1}^2}{M}\sum_{m=1}^M\mathbb{E}\bigg[\bigg\|\nu^{(m)}_{t-1} - \bar{\nu}_{t-1}\bigg\|^2 \bigg] + \frac{40L_h^2\eta^2\alpha_{t-1}^2}{M}\sum_{m=1}^M\mathbb{E}\bigg[\bigg\|\bar{\nu}_{t-1}\bigg\|^2 \bigg] + \frac{12L_h^2\gamma^2\alpha_{t-1}^2}{M}\sum_{m=1}^M\mathbb{E}\bigg[ \bigg\|\omega^{(m)}_{t-1}\bigg\|^2\bigg] 
\end{align*}
In  inequality (a) we use the generalized triangle inequality ~\ref{prop:generali_tri}. This completes the proof.
\end{proof}

\subsection{Bound for Inner Variable Drift}
\begin{lemma}
\label{lemma: inner_est_error_storm}
Suppose $c_{\omega}\alpha_{t-1}^2 < 1$, then for $t \in [\bar{t}_{s-1} + 1,  \bar{t}_s]$, with $s \in [S]$, we have:
\begin{align*}
    &\frac{1}{M}\sum_{m=1}^M\mathbb{E} \bigg[ \bigg\|\omega^{(m)}_t - \nabla_y g^{(m)}(x^{(m)}_{t}, y^{(m)}_{t} ) \bigg\|^2 \bigg] \leq ( 1 - c_{\omega}\alpha_{t-1}^2)^2\frac{1}{M}\sum_{m=1}^M\mathbb{E} \bigg[ \bigg\| \omega_{t-1}^{(m)} - \nabla_y g^{(m)}(x^{(m)}_{t-1}, y^{(m)}_{t-1})  \bigg\|^2 \bigg] + 2(c_{\omega}\alpha_{t-1}^2)^2\sigma^2 \nonumber\\
    & \qquad + 2(1 - c_{\omega}\alpha_{t-1}^2)^2L^2\frac{1}{M}\sum_{m=1}^M\mathbb{E}\bigg[2\eta^2\alpha_{t-1}^2\bigg(\bigg\|\nu^{(m)}_{t-1} -\bar{\nu}_{t-1}\bigg\|^2 +  \bigg\|\bar{\nu}_{t-1}\bigg\|^2\bigg) +  \gamma^2\alpha_{t-1}^2\bigg\|\omega^{(m)}_{t-1}\bigg\|^2\bigg] \nonumber \\
\end{align*}
where the expectation is w.r.t the stochasticity of the algorithm.
\end{lemma}
\begin{proof}
For $t \in [\bar{t}_{s-1} + 1,  \bar{t}_s - 1]$, with $s \in [S]$, we follow similar derivation as in Eq.~(\ref{eq:niu_diff_storm}) and get:
\begin{align}
    &\mathbb{E} \bigg[ \bigg\|\omega^{(m)}_t - \nabla_y g^{(m)}(x^{(m)}_{t}, y^{(m)}_{t} ) \bigg\|^2 \bigg] \nonumber \\
    & =  \mathbb{E} \bigg[ \bigg\|\nabla_y g^{(m)} (x^{(m)}_{t}, y^{(m)}_{t} ,\mathcal{B}_{y}) + ( 1 -  c_{\omega}\alpha_{t-1}^2) (\omega_{t-1}^{(m)} - \nabla_y g^{(m)} (x^{(m)}_{t-1}, y^{(m)}_{t-1} ,\mathcal{B}_{y})) - \nabla_y g^{(m)}(x^{(m)}_{t}, y^{(m)}_{t} ) \bigg\|^2 \bigg] \nonumber \\
    & = \mathbb{E} \bigg[ \bigg\| ( 1 - c_{\omega}\alpha_{t-1}^2) (\omega_{t-1}^{(m)} - \nabla_y g^{(m)}(x^{(m)}_{t-1}, y^{(m)}_{t-1}))  + \nabla g^{(m)} (x^{(m)}_{t}, y^{(m)}_{t} ,\mathcal{B}_{y}) - \nabla g^{(m)}(x^{(m)}_{t}, y^{(m)}_{t} ) \nonumber \\
    & \qquad \qquad +  (1 - c_{\omega}\alpha_{t-1}^2)(\nabla g^{(m)}(x^{(m)}_{t-1}, y^{(m)}_{t-1}) - \nabla_y g^{(m)} (x^{(m)}_{t-1}, y^{(m)}_{t-1} ,\mathcal{B}_{y})) \bigg\|^2 \bigg] \nonumber \\
    & \overset{(a)}{\leq} ( 1 - c_{\omega}\alpha_{t-1}^2)^2\mathbb{E} \bigg[ \bigg\| \omega_{t-1}^{(m)} - \nabla_y g^{(m)}(x^{(m)}_{t-1}, y^{(m)}_{t-1})  \bigg\|^2 \bigg] \nonumber \\
    & \qquad + \mathbb{E} \bigg[ \bigg\| \nabla g^{(m)} (x^{(m)}_{t}, y^{(m)}_{t} ,\mathcal{B}_{y}) - \nabla g^{(m)}(x^{(m)}_{t}, y^{(m)}_{t} ) \nonumber  +  (1 - c_{\omega}\alpha_{t-1}^2)(\nabla g^{(m)}(x^{(m)}_{t-1}, y^{(m)}_{t-1}) - \nabla_y g^{(m)} (x^{(m)}_{t-1}, y^{(m)}_{t-1} ,\mathcal{B}_{y})) \bigg\|^2 \bigg] \nonumber \\
    & \overset{(b)}{\leq} ( 1 - c_{\omega}\alpha_{t-1}^2)^2\mathbb{E} \bigg[ \bigg\| \omega_{t-1}^{(m)} - \nabla_y g^{(m)}(x^{(m)}_{t-1}, y^{(m)}_{t-1})  \bigg\|^2 \bigg] \nonumber \\
    & \qquad + 2(c_{\omega}\alpha_{t-1}^2)^2\mathbb{E} \bigg[ \bigg\| \nabla g^{(m)} (x^{(m)}_{t}, y^{(m)}_{t} ,\mathcal{B}_{y}) - \nabla g^{(m)}(x^{(m)}_{t}, y^{(m)}_{t} ) \bigg\|^2\bigg] \nonumber \\
    & \qquad +  2(1 - c_{\omega}\alpha_{t-1}^2)^2\mathbb{E} \bigg[\bigg\| -\nabla g^{(m)}(x^{(m)}_{t}, y^{(m)}_{t}) + \nabla_y g^{(m)} (x^{(m)}_{t}, y^{(m)}_{t} ,\mathcal{B}_{y}) +  \nabla g^{(m)}(x^{(m)}_{t-1}, y^{(m)}_{t-1}) - \nabla_y g^{(m)} (x^{(m)}_{t-1}, y^{(m)}_{t-1} ,\mathcal{B}_{y}) \bigg\|^2 \bigg] \nonumber \\
    & \overset{(c)}{\leq} ( 1 - c_{\omega}\alpha_{t-1}^2)^2\mathbb{E} \bigg[ \bigg\| \omega_{t-1}^{(m)} - \nabla_y g^{(m)}(x^{(m)}_{t-1}, y^{(m)}_{t-1})  \bigg\|^2 \bigg] \nonumber \\
    & \qquad + 2(c_{\omega}\alpha_{t-1}^2)^2\sigma^2 + 2(1 - c_{\omega}\alpha_{t-1}^2)^2\mathbb{E} \bigg[ \bigg\| \nabla g^{(m)} (x^{(m)}_{t}, y^{(m)}_{t} ,\mathcal{B}_{y}) - \nabla_y g^{(m)} (x^{(m)}_{t-1}, y^{(m)}_{t-1} ,\mathcal{B}_{y}) \bigg\|^2 \bigg] \nonumber \\
    & \overset{(d)}{\leq} ( 1 - c_{\omega}\alpha_{t-1}^2)^2\mathbb{E} \bigg[ \bigg\| \omega_{t-1}^{(m)} - \nabla_y g^{(m)}(x^{(m)}_{t-1}, y^{(m)}_{t-1})  \bigg\|^2 \bigg] \nonumber \\
    & \qquad + 2(c_{\omega}\alpha_{t-1}^2)^2\sigma^2 + 2(1 - c_{\omega}\alpha_{t-1}^2)^2L^2\mathbb{E}\bigg[\bigg\|\eta\alpha_{t-1}\nu^{(m)}_{t-1}\bigg\|^2 + \bigg\|\gamma\alpha_{t-1}\omega^{(m)}_{t-1}\bigg\|^2\bigg] \nonumber \\
    & \overset{(e)}{\leq} ( 1 - c_{\omega}\alpha_{t-1}^2)^2\mathbb{E} \bigg[ \bigg\| \omega_{t-1}^{(m)} - \nabla_y g^{(m)}(x^{(m)}_{t-1}, y^{(m)}_{t-1})  \bigg\|^2 \bigg] + 2(c_{\omega}\alpha_{t-1}^2)^2\sigma^2 \nonumber\\
    & \qquad + 2(1 - c_{\omega}\alpha_{t-1}^2)^2L^2\mathbb{E}\bigg[2\eta^2\alpha_{t-1}^2\bigg(\bigg\|\nu^{(m)}_{t-1} -\bar{\nu}_{t-1}\bigg\|^2 +  \bigg\|\bar{\nu}_{t-1}\bigg\|^2\bigg) +  \gamma^2\alpha_{t-1}^2\bigg\|\omega^{(m)}_{t-1}\bigg\|^2\bigg] \nonumber \\
\label{eq:omega_diff_storm}
\end{align}
where inequality (a) uses the fact that the cross product term is zero in expectation; inequality (b) uses the generalized triangle inequality; inequality (c) follows the bounded variance assumption~\ref{assumption:noise_assumption} and Proposition~\ref{prop: Sum_Mean_Kron}; inequality (d) uses the smoothness assumption~\ref{assumption:g_smoothness}; inequality (e) uses the generalized triangle inequality.

When $ t = \bar{t}_s$, the only difference is that we use $\bar{x}_{\bar{t}_s - 1}$ in Line 9 of the algorithm~\ref{alg:FedBiOAcc} to evaluate $\omega^{(m)}_{\bar{t}_s}$ instead of $x^{(m)}_{\bar{t}_s - 1}$ when $t < \bar{t}_s$. We follow similar derivation as in Eq~\eqref{eq:omega_diff_storm} and get:
\begin{align}
    &\mathbb{E} \bigg[ \bigg\|\omega^{(m)}_{\bar{t}_s} - \nabla_y g^{(m)}(x^{(m)}_{\bar{t}_s}, y^{(m)}_{\bar{t}_s} ) \bigg\|^2\\
    & \leq ( 1 - c_{\omega}\alpha_{\bar{t}_s - 1}^2)^2\mathbb{E} \bigg[ \bigg\| \omega_{\bar{t}_s-1}^{(m)} - \nabla_y g^{(m)}(x^{(m)}_{\bar{t}_s-1}, y^{(m)}_{\bar{t}_s-1})  \bigg\|^2 \bigg] \nonumber \\
    &\qquad  + 2(c_{\omega}\alpha_{\bar{t}_s - 1}^2)^2\sigma^2 + 2(1 - c_{\omega}\alpha_{\bar{t}_s - 1}^2)^2L^2\mathbb{E}\bigg[\bigg\|x^{(m)}_{\bar{t}_s} - \bar{x}_{\bar{t}_s-1}\bigg\|^2 + \bigg\|\gamma\alpha_{\bar{t}_s - 1}\omega^{(m)}_{\bar{t}_s-1}\bigg\|^2\bigg] \nonumber \\
    & \leq ( 1 - c_{\omega}\alpha_{\bar{t}_s - 1}^2)^2\mathbb{E} \bigg[ \bigg\| \omega_{\bar{t}_s-1}^{(m)} - \nabla_y g^{(m)}(x^{(m)}_{\bar{t}_s-1}, y^{(m)}_{\bar{t}_s-1})  \bigg\|^2 \bigg] \nonumber \\
    &\qquad  + 2(c_{\omega}\alpha_{\bar{t}_s - 1}^2)^2\sigma^2 + 2(1 - c_{\omega}\alpha_{\bar{t}_s - 1}^2)^2L^2\mathbb{E}\bigg[\bigg\|\bar{x}_{\bar{t}_s} - \bar{x}_{\bar{t}_s-1}\bigg\|^2 + \bigg\|\gamma\alpha_{\bar{t}_s - 1}\omega^{(m)}_{\bar{t}_s-1}\bigg\|^2\bigg] \nonumber \\
    & \leq ( 1 - c_{\omega}\alpha_{\bar{t}_s - 1}^2)^2\mathbb{E} \bigg[ \bigg\| \omega_{\bar{t}_s-1}^{(m)} - \nabla_y g^{(m)}(x^{(m)}_{\bar{t}_s-1}, y^{(m)}_{\bar{t}_s-1})  \bigg\|^2 \bigg] \nonumber \\
    &\qquad  + 2(c_{\omega}\alpha_{\bar{t}_s - 1}^2)^2\sigma^2 + 2(1 - c_{\omega}\alpha_{\bar{t}_s - 1}^2)^2L^2\mathbb{E}\bigg[\frac{1}{M}\sum_{j=1}^M\bigg\|\eta\alpha_{\bar{t}_s - 1}\nu^{(j)}_{\bar{t}_{s}-1}\bigg\|^2 + \bigg\|\gamma\alpha_{\bar{t}_s - 1}\omega^{(m)}_{\bar{t}_s-1}\bigg\|^2\bigg] \nonumber \\
\label{eq:omega_diff_storm1}
\end{align}
The second inequality follows the fact that $x^{(m)}_{\bar{t}_s} = \bar{x}_{\bar{t}_s}$; the last inequality follows the generalized triangle inequality. Finally, combine Eq.~\eqref{eq:omega_diff_storm} and~\eqref{eq:omega_diff_storm1} and average over all M clients finish the proof.
\end{proof}

\begin{lemma}
\label{lemma: inner_drift_storm}
For $\gamma < \frac{1}{15L}$ and $0 < \alpha_t < 1$, we have for $t \in [\bar{t}_{s-1} + 1,  \bar{t}_s-1]$:
\begin{align*}
\mathbb{E}\bigg[\bigg\|y^{(m)}_t - y^{(m)}_{x^{(m)}_{t}} \bigg\|^2\bigg] & \leq (1 - \frac{\mu\gamma\alpha_{t-1}}{4})\mathbb{E}\bigg[\bigg\|y^{(m)}_{t-1} - y^{(m)}_{x^{(m)}_{t-1}}\bigg\|^2\bigg] - \frac{3\gamma^2\alpha_{t-1}}{4}\mathbb{E}\bigg[\bigg\|\omega^{(m)}_{t-1}\bigg\|^2\bigg]\nonumber \\
& \qquad + \frac{25\gamma\alpha_{t-1}}{6\mu}\mathbb{E}\bigg[\bigg\|\omega^{(m)}_{t-1} - \nabla_y g^{(m)}(x^{(m)}_{t-1}, y^{(m)}_{t-1} )\bigg\|^2\bigg] + \frac{25L^2\eta^2\alpha_{t-1}}{6\mu^3\gamma}\mathbb{E}\bigg[\bigg\|\nu^{(m)}_{t-1}\bigg\|^2\bigg] \nonumber \\
\end{align*}
and when $t = \bar{t}_s$, we have:
\begin{align*}
  \mathbb{E} \bigg[\bigg\|y^{(m)}_{t} - y^{(m)}_{x^{(m)}_{t}}\bigg\|^2\bigg] & \leq (1 - \frac{\mu\gamma\alpha_{t - 1}}{8} )\mathbb{E}\bigg[\bigg\|y^{(m)}_{{t}-1} - y^{(m)}_{x^{(m)}_{{t}-1}}\bigg\|^2\bigg] - \frac{3\gamma^2\alpha_{t-1}}{4}\mathbb{E}\bigg[\bigg\|\omega^{(m)}_{{t}-1}\bigg\|^2\bigg]\nonumber \\
& \qquad + \frac{5\gamma\alpha_{t - 1}}{\mu} \mathbb{E}\bigg[\bigg\|\omega^{(m)}_{{t}-1} - \nabla_y g^{(m)}(x^{(m)}_{{t}-1}, y^{(m)}_{{t}-1} )\bigg\|^2\bigg] \nonumber \\
& \qquad + \frac{5L^2\eta^2\alpha_{{t}-1}}{\mu^3\gamma}\mathbb{E}\bigg[\bigg\|\nu^{(m)}_{{t}-1}\bigg\|^2\bigg] + (1+\frac{8}{\mu\gamma\alpha_{t - 1}})\rho^2\mathbb{E} \bigg[\bigg\|\hat{x}^{(m)}_{t} - \bar{x}_{t}\bigg\|^2\bigg]\nonumber \\
\end{align*}
\end{lemma}

\begin{proof}
For $t \in [\bar{t}_{s-1} + 1,  \bar{t}_s - 1]$, with $s \in [S]$, following Lemma 9 in~\cite{yang2021provably}, for $\gamma \leq \frac{1}{6L}$, we have:
\begin{align}
\mathbb{E}\bigg[\bigg\|y^{(m)}_t - y^{(m)}_{x^{(m)}_{t}} \bigg\|^2\bigg] & \leq (1 - \frac{\mu\gamma\alpha_{t-1}}{4})\mathbb{E}\bigg[\bigg\|y^{(m)}_{t-1} - y^{(m)}_{x^{(m)}_{t-1}}\bigg\|^2\bigg] - \frac{3\gamma^2\alpha_{t-1}}{4}\mathbb{E}\bigg[\bigg\|\omega^{(m)}_{t-1}\bigg\|^2\bigg]\nonumber \\
& \qquad + \frac{25\gamma\alpha_{t-1}}{6\mu}\mathbb{E}\bigg[\bigg\|\omega^{(m)}_{t-1} - \nabla_y g^{(m)}(x^{(m)}_{t-1}, y^{(m)}_{t-1} )\bigg\|^2\bigg] + \frac{25L^2\eta^2\alpha_{t-1}}{6\mu^3\gamma}\mathbb{E}\bigg[\bigg\|\nu^{(m)}_{t-1}\bigg\|^2\bigg] \nonumber \\
\label{eq:3}
\end{align}
When $t = \bar{t}_{s}$, we average variable $x$ over the $m$ clients, \emph{i.e.} $x^{(m)}_{\bar{t}_s} = \bar{x}_{\bar{t}_s}$. For $\bigg\|y^{(m)}_{\bar{t}_s} - y^{(m)}_{\hat{x}^{(m)}_{\bar{t}_s}} \bigg\|^2$, we can get similar recursive relation as:

\begin{align}
\mathbb{E}\bigg[\bigg\|y^{(m)}_{\bar{t}_s} - y^{(m)}_{\hat{x}^{(m)}_{\bar{t}_s}} \bigg\|^2\bigg] & \leq (1 - \frac{\mu\gamma\alpha_{{\bar{t}_s}-1}}{4})\mathbb{E}\bigg[\bigg\|y^{(m)}_{{\bar{t}_s}-1} - y^{(m)}_{x^{(m)}_{{\bar{t}_s}-1}}\bigg\|^2\bigg] - \frac{3\gamma^2\alpha_{t-1}}{4}\mathbb{E}\bigg[\bigg\|\omega^{(m)}_{{\bar{t}_s}-1}\bigg\|^2\bigg]\nonumber \\
& \qquad + \frac{25\gamma\alpha_{{\bar{t}_s}-1}}{6\mu}\mathbb{E}\bigg[\bigg\|\omega^{(m)}_{{\bar{t}_s}-1} - \nabla_y g^{(m)}(x^{(m)}_{{\bar{t}_s}-1}, y^{(m)}_{{\bar{t}_s}-1} )\bigg\|^2\bigg] + \frac{25L^2\eta^2\alpha_{{\bar{t}_s}-1}}{6\mu^3\gamma}\mathbb{E}\bigg[\bigg\|\nu^{(m)}_{{\bar{t}_s}-1}\bigg\|^2\bigg]
\label{eq:1}
\end{align}
while for $\bigg\|y^{(m)}_{\bar{t}_s} - y^{(m)}_{x^{(m)}_{\bar{t}_s}} \bigg\|^2 = \bigg\|y^{(m)}_{\bar{t}_{s}} - y^{(m)}_{\bar{x}_{\bar{t}_{s}}}\bigg\|^2$, by generalized triangle inequality, we have:
\begin{align}
  \mathbb{E} \bigg[\bigg\|y^{(m)}_{\bar{t}_{s}} - y^{(m)}_{\bar{x}_{\bar{t}_{s}}}\bigg\|^2\bigg] & \leq (1+\frac{\mu\gamma\alpha_{\bar{t}_s - 1}}{8})\mathbb{E} \bigg[\bigg\|y^{(m)}_{\bar{t}_{s}} - y^{(m)}_{\hat{x}^{(m)}_{\bar{t}_{s}}} \bigg\|^2\bigg]  + (1+\frac{8}{\mu\gamma\alpha_{\bar{t}_s - 1}})\mathbb{E} \bigg[\bigg\|y^{(m)}_{\hat{x}^{(m)}_{\bar{t}_{s}}}  - y^{(m)}_{\bar{x}_{\bar{t}_{s}}}\bigg\|^2\bigg] \nonumber \\
  & \leq (1+\frac{\mu\gamma\alpha_{\bar{t}_s - 1}}{8})\mathbb{E} \bigg[\bigg\|y^{(m)}_{\bar{t}_{s}} - y^{(m)}_{\hat{x}^{(m)}_{\bar{t}_{s}}} \bigg\|^2\bigg]  + (1+\frac{8}{\mu\gamma\alpha_{\bar{t}_s - 1}})\rho^2\mathbb{E} \bigg[\bigg\|\hat{x}^{(m)}_{\bar{t}_{s}} - \bar{x}_{\bar{t}_{s}}\bigg\|^2\bigg]\nonumber \\
\label{eq:2}
\end{align}
Combine Eq.~\eqref{eq:1} and Eq.~\eqref{eq:2} together, we have:
\begin{align*}
  \mathbb{E} \bigg[\bigg\|y^{(m)}_{\bar{t}_{s}} - y^{(m)}_{\bar{x}_{\bar{t}_{s}}}\bigg\|^2\bigg] & \leq (1+\frac{\mu\gamma\alpha_{\bar{t}_s - 1}}{8})(1 - \frac{\mu\gamma\alpha_{{\bar{t}_s}-1}}{4})\mathbb{E}\bigg[\bigg\|y^{(m)}_{{\bar{t}_s}-1} - y^{(m)}_{x^{(m)}_{{\bar{t}_s}-1}}\bigg\|^2\bigg] - (1+\frac{\mu\gamma\alpha_{\bar{t}_s - 1}}{8})\frac{3\gamma^2\alpha_{\bar{t}_s-1}}{4}\mathbb{E}\bigg[\bigg\|\omega^{(m)}_{{\bar{t}_s}-1}\bigg\|^2\bigg]\nonumber \\
& \qquad + (1+\frac{\mu\gamma\alpha_{\bar{t}_s - 1}}{8})\frac{25\gamma\alpha_{{\bar{t}_s}-1}}{6\mu}\mathbb{E}\bigg[\bigg\|\omega^{(m)}_{{\bar{t}_s}-1} - \nabla_y g^{(m)}(x^{(m)}_{{\bar{t}_s}-1}, y^{(m)}_{{\bar{t}_s}-1} )\bigg\|^2\bigg] \nonumber \\
& \qquad + (1+\frac{\mu\gamma\alpha_{\bar{t}_s - 1}}{8})\frac{25L^2\eta^2\alpha_{{\bar{t}_s}-1}}{6\mu^3\gamma}\mathbb{E}\bigg[\bigg\|\nu^{(m)}_{{\bar{t}_s}-1}\bigg\|^2\bigg] + (1+\frac{8}{\mu\gamma\alpha_{\bar{t}_s - 1}})\rho^2\mathbb{E} \bigg[\bigg\|\hat{x}^{(m)}_{\bar{t}_{s}} - \bar{x}_{\bar{t}_{s}}\bigg\|^2\bigg]\nonumber \\
\end{align*}
For the coefficients, since we set $\gamma < \frac{1}{15L} < \frac{1}{15\mu}$ and $ 0 < \alpha_t < 1$, it is straightforward to verify the following inequalities hold: 
\begin{align*}
(1+\frac{\mu\gamma\alpha_{\bar{t}_s - 1}}{8})(1 - \frac{\mu\gamma\alpha_{{\bar{t}_s}-1}}{4}) &= 1 - \frac{\mu\gamma\alpha_{\bar{t}_s - 1}}{8} - \frac{\mu^2\gamma^2\alpha_{\bar{t}_s - 1}^2}{32}\leq 1 - \frac{\mu\gamma\alpha_{\bar{t}_s - 1}}{8} \\
-(1+\frac{\mu\gamma\alpha_{\bar{t}_s - 1}}{8})\frac{3\gamma^2\alpha_{\bar{t}_s-1}}{4} & \leq -\frac{3\gamma^2\alpha_{\bar{t}_s-1}}{4}\\
(1+\frac{\mu\gamma\alpha_{\bar{t}_s - 1}}{8})\frac{25\gamma\alpha_{{\bar{t}_s}-1}}{6\mu} &\leq \frac{605\gamma\alpha_{\bar{t}_s - 1}}{144\mu} \leq \frac{5\gamma\alpha_{\bar{t}_s - 1}}{\mu} \\
(1+\frac{\mu\gamma\alpha_{\bar{t}_s - 1}}{8})\frac{25L^2\eta^2\alpha_{{\bar{t}_s}-1}}{6\mu^3\gamma} & \leq \frac{605L^2\eta^2\alpha_{{\bar{t}_s}-1}}{144\mu^3\gamma} \leq \frac{5L^2\eta^2\alpha_{{\bar{t}_s}-1}}{\mu^3\gamma}\nonumber \\
\end{align*}
So we have for $ t = \bar{t}_s$:
\begin{align*}
  \mathbb{E} \bigg[\bigg\|y^{(m)}_{\bar{t}_{s}} - y^{(m)}_{\bar{x}_{\bar{t}_{s}}}\bigg\|^2\bigg] & \leq (1 - \frac{\mu\gamma\alpha_{\bar{t}_s - 1}}{8} )\mathbb{E}\bigg[\bigg\|y^{(m)}_{{\bar{t}_s}-1} - y^{(m)}_{x^{(m)}_{{\bar{t}_s}-1}}\bigg\|^2\bigg] - \frac{3\gamma^2\alpha_{t-1}}{4}\mathbb{E}\bigg[\bigg\|\omega^{(m)}_{{\bar{t}_s}-1}\bigg\|^2\bigg]\nonumber \\
& \qquad + \frac{5\gamma\alpha_{\bar{t}_s - 1}}{\mu} \mathbb{E}\bigg[\bigg\|\omega^{(m)}_{{\bar{t}_s}-1} - \nabla_y g^{(m)}(x^{(m)}_{{\bar{t}_s}-1}, y^{(m)}_{{\bar{t}_s}-1} )\bigg\|^2\bigg] \nonumber \\
& \qquad + \frac{5L^2\eta^2\alpha_{{\bar{t}_s}-1}}{\mu^3\gamma}\mathbb{E}\bigg[\bigg\|\nu^{(m)}_{{\bar{t}_s}-1}\bigg\|^2\bigg] + (1+\frac{8}{\mu\gamma\alpha_{\bar{t}_s - 1}})\rho^2\mathbb{E} \bigg[\bigg\|\hat{x}^{(m)}_{\bar{t}_{s}} - \bar{x}_{\bar{t}_{s}}\bigg\|^2\bigg]\nonumber \\
\end{align*}
Combine with cases when $t \in [\bar{t}_{s-1} + 1,  \bar{t}_s - 1]$ in Eq.~\eqref{eq:3} completes the proof.
\end{proof}

\subsection{Bound for Outer Variable Drift}
\begin{lemma}
For $\alpha < \frac{1}{16IL_h}$ and $0 < \eta < 1$, we have for $t \in [\bar{t}_{s-1} + 1,  \bar{t}_s-1]$:
\label{lem: ErrorAccumulation_Iterates_FedAvg_storm}
\begin{align*}
    \sum_{m=1}^M \mathbb{E} \| \nu_{t}^{(m)} - \bar{\nu}_{t} \|^2 
    & \leq \bigg(1 + \frac{33}{32I}\bigg) \sum_{m=1}^M \mathbb{E} \|  \nu_{t-1}^{(m)}  - \bar{\nu}_{t-1} \|^2  + 4 I L_h^2\alpha_{t-1}^2 \sum_{m=1}^M \mathbb{E}\bigg[2\| \eta\bar{\nu}_{t-1} \|^2 + \| \gamma \omega^{(m)}_{t-1} \|^2 \bigg] \nonumber \\
    & \qquad +  \frac{8I M (c_{\nu}\alpha_{t-1}^2)^2G^2}{2L_h} +  \frac{M c_{\nu}^2\alpha_{t-1}^3\zeta^2}{L_h}  +  \frac{M c_{\nu}^2\alpha_{t-1}^3\zeta^2}{L_h} \nonumber \\ & \qquad \qquad \qquad + \frac{\eta^2 c_{\nu}^2\alpha_{t-1}^2(1 + \rho^2)}{2} \sum_{\ell = \bar{t}_{s-1}}^{t-1} \alpha_l^2 \sum_{m = 1}^M \Big\| \big(  \nu_\ell^{(m)} -  \bar{\nu}_\ell  \big) \Big\|^2  \nonumber\\
\end{align*}
where the expectation is w.r.t the stochasticity of the algorithm.
\end{lemma}
\begin{proof}
For $t \in [\bar{t}_{s-1} + 1,  \bar{t}_s]$, with $s \in [S]$, we have: $\hat{x}_{t}^{(m)} = \hat{x}_{t-1}^{(m)} - \eta\alpha_{t-1}  \nu_{t-1}^{(m)}$, this implies that:
\begin{align*}
    \hat{x}_t^{(m)} = x_{\bar{t}_{s-1}}^{(m)} - \sum_{\ell = \bar{t}_{s-1}}^{t-1} \eta  \nu_\ell^{(m)} \quad \text{and} \quad \bar{x}_{t}  = \bar{x}_{\bar{t}_{s-1}}  - \sum_{\ell = \bar{t}_{s-1}}^{t-1} \eta  \bar{\nu}_\ell.
\end{align*}
So for $t \in [\bar{t}_{s-1} + 1,  \bar{t}_s]$, with $s \in [S]$ we have:
\begin{align}
\label{Eq: ConsensusError_FedAvg_storm}
\frac{1}{M} \sum_{m = 1}^M  \| \hat{x}_t^{(m)}-  \bar{x}_t \|^2 & = \frac{1}{M} \sum_{m = 1}^M \Big\| x_{\bar{t}_{s-1}}^{(m)} - \bar{x}_{\bar{t}_{s-1}}  - \Big( \sum_{\ell = \bar{t}_{s-1}}^{t-1} \eta\alpha_\ell  \nu_\ell^{(m)} -   \sum_{\ell =  \bar{t}_{s-1}}^{t-1} \eta \alpha_\ell  \bar{\nu}_\ell  \Big) \Big\|^2 \nonumber \\ & \overset{(a)}{=} \frac{1}{M} \sum_{m = 1}^M \Big\|  \sum_{\ell = \bar{t}_{s-1}}^{t-1} \eta\alpha_\ell\big(  \nu_\ell^{(m)} -      \bar{\nu}_\ell  \big) \Big\|^2
\overset{(b)}{\leq} \sum_{\ell = \bar{t}_{s-1}}^{t-1} \frac{I\eta^2\alpha_l^2}{M} \sum_{m = 1}^M \Big\| \big(  \nu_\ell^{(m)} -  \bar{\nu}_\ell  \big) \Big\|^2  \nonumber\\
\end{align}
where the equality $(a)$ follows from the fact that $x_{\bar{t}_{s-1}}^{(m)} = \bar{x}_{\bar{t}_{s-1}}$ for $t = \bar{t}_{s - 1}$; inequality (b) is due to $t - \bar{t}_{s-1} \leq I$ and the generalized triangle inequality.  

Next, we bound the term $\| \nu_{t}^{(m)} - \bar{\nu}_{t} \|^2$, for $t \in [\bar{t}_{s-1} + 1,  \bar{t}_s - 1]$, with $s \in [S]$:
\begin{align*}
\sum_{m=1}^M \mathbb{E} \| \nu_{t}^{(m)} - \bar{\nu}_{t} \|^2
& = \sum_{m=1}^M \mathbb{E} \bigg\| \mu^{(m)}_t+ (1 - c_{\nu}\alpha_{t-1}^2) \big( \nu_{t-1}^{(m)} -   \mu^{(m)}_{t-1}\big) - \bigg( \frac{1}{M} \sum_{j=1}^M  \mu^{(j)}_{t} + (1 - c_{\nu}\alpha_{t-1}^2) \big( \bar{\nu}_{t-1} - \frac{1}{M} \sum_{j=1}^M  \mu^{(j)}_{t-1}\big) \bigg)   \bigg\|^2 \nonumber \\
& =  \sum_{m=1}^M \mathbb{E} \bigg\| (1 - c_{\nu}\alpha_{t-1}^2) \big( \nu_{t-1}^{(m)}  - \bar{\nu}_{t-1} \big) +  \mu^{(m)}_t- \frac{1}{M} \sum_{j=1}^M  \mu^{(j)}_{t} - (1- c_{\nu}\alpha_{t-1}^2) \bigg(  \mu^{(m)}_{t-1} - \frac{1}{M} \sum_{j=1}^M  \mu^{(j)}_{t-1} \bigg)  \bigg\|^2 \nonumber \\
& \overset{(a)}{\leq} (1 + \beta) (1 - c_{\nu}\alpha_{t-1}^2)^2 \sum_{m=1}^M \mathbb{E} \|  \nu_{t-1}^{(m)}  - \bar{\nu}_{t-1} \|^2 \nonumber \\
& \qquad \qquad + \Big( 1 + \frac{1}{\beta} \bigg) \sum_{m=1}^M \mathbb{E} \bigg\|  \mu^{(m)}_t- \frac{1}{M} \sum_{j=1}^M  \mu^{(j)}_{t} - (1- c_{\nu}\alpha_{t-1}^2) \bigg(  \mu^{(m)}_{t-1} - \frac{1}{M} \sum_{j=1}^M  \mu^{(j)}_{t-1} \bigg)  \bigg\|^2
\end{align*}
where $(a)$ follows from the the generalized triangle inequality for some $\beta>0$. Next we bound the second term:
\begin{align}
   & \sum_{m=1}^M \mathbb{E}\bigg\|  \mu^{(m)}_t - \frac{1}{M} \sum_{j=1}^M  \mu^{(j)}_{t}  - (1- c_{\nu}\alpha_{t-1}^2) \bigg(  \mu^{(m)}_{t-1} - \frac{1}{M} \sum_{j=1}^M  \mu^{(j)}_{t-1} \bigg)  \bigg\|^2 \nonumber\\
    & = \sum_{m=1}^M \mathbb{E}\bigg\| \mu^{(m)}_t - \frac{1}{M} \sum_{j=1}^M  \mu^{(j)}_{t} - \bigg( \mu^{(m)}_{t-1} - \frac{1}{M} \sum_{j=1}^M  \mu^{(j)}_{t-1} \bigg) + c_{\nu}\alpha_{t-1}^2 \bigg( \mu^{(m)}_{t-1} - \frac{1}{M} \sum_{j=1}^M  \mu^{(j)}_{t-1} \bigg) \bigg\|^2 \nonumber\\
    & \overset{(a)}{\leq}  2 \sum_{m=1}^M \mathbb{E}\bigg\| \mu^{(m)}_t - \frac{1}{M} \sum_{j=1}^M  \mu^{(j)}_{t} -  \bigg( \mu^{(m)}_{t-1} - \frac{1}{M} \sum_{j=1}^M  \mu^{(j)}_{t-1} \bigg)\bigg\|^2 + 2 (c_{\nu}\alpha_{t-1}^2)^2 \sum_{m=1}^M \mathbb{E} \bigg\|  \mu^{(m)}_{t-1} - \frac{1}{M} \sum_{j=1}^M  \mu^{(j)}_{t-1}  \bigg\|^2 \nonumber \\
   & \overset{(b)}{\leq}  2 \sum_{m=1}^M \mathbb{E} \bigg\| \mu^{(m)}_t - \mu^{(m)}_{t-1} \bigg\|^2 + 2 (c_{\nu}\alpha_{t-1}^2)^2 \sum_{m=1}^M \mathbb{E} \bigg\|  \mu^{(m)}_{t-1} - \frac{1}{M} \sum_{j=1}^M  \mu^{(j)}_{t-1}  \bigg\|^2 \nonumber \\
   & \overset{(c)}{\leq}  2 L_h^2 \sum_{m=1}^M \mathbb{E}\bigg[ \| x_t^{(m)} - x_{t-1}^{(m)} \|^2 + \| y_t^{(m)} - y_{t-1}^{(m)} \|^2 \bigg]  + 2 (c_{\nu}\alpha_{t-1}^2)^2 \sum_{m=1}^M \mathbb{E} \bigg\|  \mu^{(m)}_{t-1} - \frac{1}{M} \sum_{j=1}^M  \mu^{(j)}_{t-1}  \bigg\|^2 \nonumber \\
   & \leq  2 L_h^2\alpha_{t-1}^2 \sum_{m=1}^M \mathbb{E}\bigg[ \| \eta\nu^{(m)}_{t-1} \|^2 + \| \gamma \omega^{(m)}_{t-1} \|^2 \bigg]  + 2 (c_{\nu}\alpha_{t-1}^2)^2 \sum_{m=1}^M \mathbb{E} \bigg\|  \mu^{(m)}_{t-1} - \frac{1}{M} \sum_{j=1}^M  \mu^{(j)}_{t-1}  \bigg\|^2 
\end{align}
where inequality $(a)$ is from the triangle inequality, $(b)$ follows Proposition~\ref{prop: Sum_Mean_Kron}; $(c)$ follows from the Lipschitz-smoothness of the $h$. Next for the second term of the above equation:
\begin{align}
\sum_{m=1}^M \mathbb{E} \bigg\|  \mu^{(m)}_{t-1} - \frac{1}{M} \sum_{j=1}^M  \mu^{(j)}_{t-1}  \bigg\|^2  & = \sum_{m=1}^M \mathbb{E}\bigg\| \mu^{(m)}_{t-1} - \tilde{\mu}^{(m)}_{t-1}  - \frac{1}{M} \sum_{j=1}^M \big(  \mu^{(j)}_{t-1} - \tilde{\mu}^{(j)}_{t-1}  \big)   + \tilde{\mu}^{(m)}_{t-1} -  \frac{1}{M} \sum_{j=1}^M \tilde{\mu}^{(j)}_{t-1} \bigg\|^2 
\nonumber \\
& \overset{(a)}{\leq}  2 \sum_{m=1}^M \mathbb{E} \bigg\|  \mu^{(m)}_{t-1} - \tilde{\mu}^{(m)}_{t-1}  - \frac{1}{M} \sum_{j=1}^M \big(  \mu^{(j)}_{t-1} - \tilde{\mu}^{(j)}_{t-1}  \big) \bigg\|^2 + 2 \sum_{k=1}^K \mathbb{E} \bigg\| \tilde{\mu}^{(m)}_{t-1} -  \frac{1}{M} \sum_{j=1}^M \tilde{\mu}^{(j)}_{t-1}   \bigg\|^2 \nonumber \\
& \overset{(b)}{\leq}   2 \sum_{m=1}^M \mathbb{E} \bigg\|  \mu^{(m)}_{t-1} - \tilde{\mu}^{(m)}_{t-1} \bigg\|^2 + 2 \sum_{k=1}^K \mathbb{E} \bigg\| \tilde{\mu}^{(m)}_{t-1} -  \frac{1}{M} \sum_{j=1}^M \tilde{\mu}^{(j)}_{t-1}   \bigg\|^2 \nonumber \\
& \overset{(c)}{\leq}  2M\sigma^2 + 4 \sum_{m=1}^M \mathbb{E}\big\|\nabla h^{(m)}(\bar{x}_{t-1}) - \nabla h(\bar{x}_{t-1}) \big\|^2 + 8 \sum_{m=1}^M \mathbb{E} \big\| \tilde{\mu}^{(m)}_{t-1} - \nabla h^{(m)}(\bar{x}_{t-1})   \big\|^2 \nonumber \\
& \qquad \qquad + 8  \sum_{m=1}^M \mathbb{E} \bigg\| \nabla h(\bar{x}_{t-1}) -  \frac{1}{M} \sum_{j=1}^M \tilde{\mu}^{(j)}_{t-1}   \bigg\|^2  \nonumber\\
& \overset{(d)}{\leq}  2M\sigma^2 + 4 \sum_{m=1}^M \mathbb{E}\big\|\nabla h^{(m)}(\bar{x}_{t-1}) - \nabla h(\bar{x}_{t-1}) \big\|^2 + 16 \sum_{m=1}^M \mathbb{E} \big\| \tilde{\mu}^{(m)}_{t-1} - \nabla h^{(m)}(\bar{x}_{t-1})   \big\|^2 \nonumber \\
& \overset{(e)}{\leq}  2M\sigma^2 + 4 \sum_{m=1}^M \mathbb{E}\big\|\nabla h^{(m)}(\bar{x}_{t-1}) - \nabla h(\bar{x}_{t-1}) \big\|^2 + 32 \sum_{m=1}^M \mathbb{E} \big\| \tilde{\mu}^{(m)}_{t-1} - \Phi^{(m)}(x^{(m)}_{t-1}, y^{(m)}_{t-1})   \big\|^2 \nonumber \\
& \qquad \qquad + 32 \sum_{m=1}^M \mathbb{E} \big\| \Phi^{(m)}(x^{(m)}_{t-1}, y^{(m)}_{t-1})  - \nabla h^{(m)}(\bar{x}_{t-1})   \big\|^2  \nonumber\\
& \overset{(f)}{\leq} 2M\sigma^2 + 32MG^2  + 4 \sum_{m=1}^M \frac{1}{M} \sum_{j=1}^M  \mathbb{E}  \|  \nabla h^{(m)}(\bar{x}_{t-1}) - \nabla h^{(j)}(\bar{x}_{t-1})  \|^2 \nonumber \\
& \qquad \qquad +  32 L_h^2 \sum_{m = 1}^M \mathbb{E}\bigg[\| x_{t - 1}^{(m)} - \bar{x}_{t-1}\|^2 + \| y^{(m)}_{x_{t - 1}^{(m)}} - y^{(m)}_{\bar{x}_{t-1}}\|^2 \bigg] \nonumber\\
& \overset{(g)}{\leq} 2M\sigma^2 + 32MG^2 + 4 M \zeta^2   +  32 L_h^2 (1 + \rho^2) \sum_{m = 1}^M \mathbb{E}\bigg[\| x_{t - 1}^{(m)} - \bar{x}_{t-1}\|^2\bigg] \nonumber\\
\end{align}
inequality (a) uses triangle inequality; inequality (b) follows Proposition~\ref{prop: Sum_Mean_Kron}; inequality (c) follow Assumption~\ref{assumption:outer_noise} and generalized triangle inequality; inequality (d) and (e) follows the generalized inequality; inequality (f) follows the Assumption~\ref{assumption:outer_noise}; inequality (g) utilizes intra-node heterogeneity assumption and Proposition~\ref{prop:hg_var}. 

Finally, combine everything together, we have:
\begin{align*}
    & \sum_{m=1}^M \mathbb{E} \| \nu_{t}^{(m)} - \bar{\nu}_{t} \|^2 \nonumber \\
    & \leq (1 + \beta) (1 - c_{\nu}\alpha_{t-1}^2)^2 \sum_{m=1}^M \mathbb{E} \|  \nu_{t-1}^{(m)}  - \bar{\nu}_{t-1} \|^2  + 2 L_h^2\bigg( 1 + \frac{1}{\beta}\bigg)\alpha_{t-1}^2 \sum_{m=1}^M \mathbb{E}\bigg[ \| \eta\nu^{(m)}_{t-1} \|^2 + \| \gamma \omega^{(m)}_{t-1} \|^2 \bigg] \nonumber \\
    & \qquad + 4M\bigg( 1 + \frac{1}{\beta}\bigg) (c_{\nu}\alpha_{t-1}^2)^2\sigma^2 + 64M\bigg( 1 + \frac{1}{\beta}\bigg) (c_{\nu}\alpha_{t-1}^2)^2G^2 + 8M\bigg( 1 + \frac{1}{\beta}\bigg) (c_{\nu}\alpha_{t-1}^2)^2\zeta^2 \nonumber \\& \qquad  +  64\bigg( 1 + \frac{1}{\beta}\bigg) (c_{\nu}\alpha_{t-1}^2)^2L_h^2(1 + \rho^2) \sum_{m = 1}^M \mathbb{E}\bigg[\| x_{t - 1}^{(m)} - \bar{x}_{t-1}\|^2\bigg] \nonumber\\
    & \overset{(a)}{\leq} (1 + \beta) (1 - c_{\nu}\alpha_{t-1}^2)^2 \sum_{m=1}^M \mathbb{E} \|  \nu_{t-1}^{(m)}  - \bar{\nu}_{t-1} \|^2  + 2 L_h^2\bigg( 1 + \frac{1}{\beta}\bigg)\alpha_{t-1}^2 \sum_{m=1}^M \mathbb{E}\bigg[ \| \eta\nu^{(m)}_{t-1} \|^2 + \| \gamma \omega^{(m)}_{t-1} \|^2 \bigg] \nonumber \\
    & \qquad + 4M\bigg( 1 + \frac{1}{\beta}\bigg) (c_{\nu}\alpha_{t-1}^2)^2\sigma^2 + 64M\bigg( 1 + \frac{1}{\beta}\bigg) (c_{\nu}\alpha_{t-1}^2)^2G^2 + 8M\bigg( 1 + \frac{1}{\beta}\bigg) (c_{\nu}\alpha_{t-1}^2)^2\zeta^2 \\
    & \qquad \qquad \qquad +  64\bigg( 1 + \frac{1}{\beta}\bigg) (c_{\nu}\alpha_{t-1}^2)^2L_h^2(1 + \rho^2) \sum_{\ell = \bar{t}_{s-1}}^{t-1} I\eta^2\alpha_l^2 \sum_{m = 1}^M \Big\| \big(  \nu_\ell^{(m)} -  \bar{\nu}_\ell  \big) \Big\|^2  \nonumber\\
    & \overset{(b)}{\leq} \bigg(1 + \frac{1}{I}\bigg) \sum_{m=1}^M \mathbb{E} \|  \nu_{t-1}^{(m)}  - \bar{\nu}_{t-1} \|^2  + 4 I L_h^2\alpha_{t-1}^2 \sum_{m=1}^M \mathbb{E}\bigg[ \| \eta\nu^{(m)}_{t-1} \|^2 + \| \gamma \omega^{(m)}_{t-1} \|^2 \bigg] \nonumber \\
    & \qquad + 8I M (c_{\nu}\alpha_{t-1}^2)^2\sigma^2 + 128I M (c_{\nu}\alpha_{t-1}^2)^2G^2 + 16I M (c_{\nu}\alpha_{t-1}^2)^2\zeta^2 \nonumber\\ & \qquad \qquad \qquad + 128I^2\eta^2 (c_{\nu}\alpha_{t-1}^2)^2L_h^2(1 + \rho^2) \sum_{\ell = \bar{t}_{s-1}}^{t-1} \alpha_l^2 \sum_{m = 1}^M \Big\| \big(  \nu_\ell^{(m)} -  \bar{\nu}_\ell  \big) \Big\|^2  \nonumber\\
    & \overset{(c)}{\leq} \bigg(1 + \frac{1}{I}\bigg) \sum_{m=1}^M \mathbb{E} \|  \nu_{t-1}^{(m)}  - \bar{\nu}_{t-1} \|^2  + 4 I L_h^2\alpha_{t-1}^2 \sum_{m=1}^M \mathbb{E}\bigg[ 2\| \eta\nu^{(m)}_{t-1} - \eta\bar{\nu}_{t-1}\|^2 + 2\| \eta\bar{\nu}_{t-1} \|^2 + \| \gamma \omega^{(m)}_{t-1} \|^2 \bigg] \nonumber \\
    & \qquad + 8I M (c_{\nu}\alpha_{t-1}^2)^2\sigma^2 + 128I M (c_{\nu}\alpha_{t-1}^2)^2G^2 + 16I M (c_{\nu}\alpha_{t-1}^2)^2\zeta^2\nonumber \\ &  \qquad \qquad \qquad + 128I^2\eta^2 (c_{\nu}\alpha_{t-1}^2)^2L_h^2(1 + \rho^2) \sum_{\ell = \bar{t}_{s-1}}^{t-1} \alpha_l^2 \sum_{m = 1}^M \Big\| \big(  \nu_\ell^{(m)} -  \bar{\nu}_\ell  \big) \Big\|^2 \nonumber\\
    & \leq \bigg(1 + \frac{1}{I} + 8 I L_h^2\eta^2\alpha_{t-1}^2\bigg) \sum_{m=1}^M \mathbb{E} \|  \nu_{t-1}^{(m)}  - \bar{\nu}_{t-1} \|^2  + 4 I L_h^2\alpha_{t-1}^2 \sum_{m=1}^M \mathbb{E}\bigg[2\| \eta\bar{\nu}_{t-1} \|^2 + \| \gamma \omega^{(m)}_{t-1} \|^2 \bigg] \nonumber \\
    & \qquad + 8I M (c_{\nu}\alpha_{t-1}^2)^2\sigma^2 + 128I M (c_{\nu}\alpha_{t-1}^2)^2G^2 + 16I M (c_{\nu}\alpha_{t-1}^2)^2\zeta^2 \nonumber \\ & \qquad \qquad \qquad + 128I^2\eta^2 (c_{\nu}\alpha_{t-1}^2)^2L_h^2(1 + \rho^2) \sum_{\ell = \bar{t}_{s-1}}^{t-1} \alpha_l^2 \sum_{m = 1}^M \Big\| \big(  \nu_\ell^{(m)} -  \bar{\nu}_\ell  \big) \Big\|^2  \nonumber\\
    & \overset{(d)}{\leq} \bigg(1 + \frac{33}{32I}\bigg) \sum_{m=1}^M \mathbb{E} \|  \nu_{t-1}^{(m)}  - \bar{\nu}_{t-1} \|^2  + 4 I L_h^2\alpha_{t-1}^2 \sum_{m=1}^M \mathbb{E}\bigg[2\| \eta\bar{\nu}_{t-1} \|^2 + \| \gamma \omega^{(m)}_{t-1} \|^2 \bigg] \nonumber \\
    & \qquad +  \frac{I M (c_{\nu}\alpha_{t-1}^2)^2\sigma^2}{L_h} + \frac{8I M (c_{\nu}\alpha_{t-1}^2)^2G^2}{2L_h} +  \frac{M c_{\nu}^2\alpha_{t-1}^3\zeta^2}{L_h} + \frac{\eta^2 c_{\nu}^2\alpha_{t-1}^2(1 + \rho^2)}{2} \sum_{\ell = \bar{t}_{s-1}}^{t-1} \alpha_l^2 \sum_{m = 1}^M \Big\| \big(  \nu_\ell^{(m)} -  \bar{\nu}_\ell  \big) \Big\|^2  \nonumber\\
\end{align*}
where inequality (a) follows Eq.~\eqref{Eq: ConsensusError_FedAvg_storm}; in inequality $(b)$, we set $\beta = 1/I$ and use $I\ge 1$; Inequality $(c)$ uses the generalized triangle inequality; Inequality $(d)$, we use $\alpha_t < \frac{1}{16L_h I}$ and $\eta < 1$
Therefore, the lemma is proved. 
\end{proof}

\begin{lemma}
\label{lemma:d_bound}
For $\alpha_t < \frac{1}{16L_hI}$, we have:
\begin{align*}
    \bigg(1 - \frac{3\eta^2 c_{\nu}^2(1 + \rho^2)}{16^3*32IL_h^4}\bigg)\sum_{t = \bar{t}_{s-1}}^{\bar{t}_s-1} \alpha_{t} D_{t} \leq \frac{3\eta^2}{32}\sum_{t=\bar{t}_{s-1}}^{\bar{t}_s-1}\alpha_{t}E_{t} + \frac{3\gamma^2}{64}\sum_{t=\bar{t}_{s-1}}^{\bar{t}_s-1}\alpha_{t}F_{t} + \bigg(\frac{27c_{\nu}^2\sigma^2}{16L_h^2} +  \frac{ 3c_{\nu}^2\zeta^2}{16L_h^2}\bigg)\sum_{t=\bar{t}_{s-1}}^{\bar{t}_s-1}\alpha_{t}^3 \\
\end{align*}
where the terms $D_t$, $E_t$ and $F_t$ are denoted below.
\end{lemma}
\begin{proof}
To simplify the notation, we denote $A_t = \mathbb{E} \bigg[\Big\| \bar{\nu}_{t} - \frac{1}{M} \sum_{m=1}^M  \nabla h(x^{(m)}_t)  \Big\|^2\bigg]$, $B_t = \mathbb{E} \bigg[\frac{1}{M} \sum_{m=1}^M \bigg\|y^{(m)}_t - y^{(m)}_{x^{(m)}_{t}} \bigg\|^2 \bigg]$, $C_t = \mathbb{E} \bigg [\frac{1}{M} \sum_{m=1}^M \bigg\|\omega^{(m)}_t - \nabla_y g^{(m)}(x^{(m)}_{t}, y^{(m)}_{t} ) \bigg\|^2 \bigg]$, $D_t = \frac{1}{M}\sum_{m=1}^M\mathbb{E}\bigg[\bigg\|\nu^{(m)}_{t} - \bar{\nu}_{t}\bigg\|^2 \bigg]$, $E_t = \mathbb{E}\bigg[\bigg\|\bar{\nu}_{t}\bigg\|^2 \bigg]$, $F_t = \frac{1}{M}\sum_{m=1}^M\mathbb{E}\bigg[ \bigg\|\omega^{(m)}_{t}\bigg\|^2\bigg] $. Then we rewrite Lemma~\ref{lem: ErrorAccumulation_Iterates_FedAvg_storm} with our new notation as follows:
\begin{align*}
    D_t \nonumber
    & \leq \bigg(1 + \frac{33}{32I}\bigg) D_{t-1} + 8 I L_h^2\alpha_{t-1}^2\eta^2E_{t-1} + 4 I L_h^2\alpha_{t-1}^2\gamma^2F_{t-1} \nonumber \\
    & \qquad + \frac{c_{\nu}^2\alpha_{t-1}^3\sigma^2}{2L_h} + \frac{8c_{\nu}^2\alpha_{t-1}^3G^2}{L_h} + \frac{ c_{\nu}^2\alpha_{t-1}^3\zeta^2}{L_h} + \frac{\eta^2 c_{\nu}^2\alpha_{t-1}^2(1 + \rho^2)}{2}\sum_{\ell = \bar{t}_{s-1}}^{t-1} \alpha_l^2 D_l
\end{align*}
Next apply the above equation recursively from $\bar{t}_{s-1} + 1$ to $t$. Note that $D_{\bar{t}_{s-1}} = 1/M\sum_{m=1}^M \mathbb{E} \| \nu_{\bar{t}_{s-1}}^{(m)} - \bar{\nu}_{\bar{t}_{s-1}} \|^2 = 0$, so we have:
\begin{align*}
    D_t & \leq 8 I L_h^2\eta^2\sum_{\ell=\bar{t}_{s-1}}^{t-1}\bigg(1 + \frac{33}{32I}\bigg)^{t - \ell}\alpha_{\ell}^2E_{\ell} + 4 I L_h^2\gamma^2\sum_{\ell=\bar{t}_{s-1}}^{t-1}\bigg(1 + \frac{33}{32I}\bigg)^{t - \ell}\alpha_{\ell}^2F_{\ell} \nonumber \\
    & \qquad + \bigg(\frac{c_{\nu}^2\sigma^2}{2L_h} + \frac{8c_{\nu}^2G^2}{L_h} +  \frac{ c_{\nu}^2\zeta^2}{L_h}\bigg)\sum_{\ell=\bar{t}_{s-1}}^{t-1}\bigg(1 + \frac{33}{32I}\bigg)^{t - \ell}\alpha_{\ell}^3 + \frac{\eta^2 c_{\nu}^2(1 + \rho^2)}{2}\sum_{\ell=\bar{t}_{s-1}}^{t-1}\bigg(1 + \frac{33}{32I}\bigg)^{t - \ell}\alpha_{\ell}^2\sum_{\bar{\ell} = \bar{t}_{s-1}}^{\ell} \alpha_{\bar{\ell}}^2 D_{\bar{\ell}} \nonumber \\
    & \leq  24I L_h^2\eta^2\sum_{\ell=\bar{t}_{s-1}}^{t-1}\alpha_{\ell}^2E_{\ell} + 12 I L_h^2\gamma^2\sum_{\ell=\bar{t}_{s-1}}^{t-1}\alpha_{\ell}^2F_{\ell}  + \bigg(\frac{3c_{\nu}^2\sigma^2}{2L_h} + \frac{24c_{\nu}^2G^2}{L_h} +  \frac{ 3c_{\nu}^2\zeta^2}{L_h}\bigg)\sum_{\ell=\bar{t}_{s-1}}^{t-1}\alpha_{\ell}^3 \nonumber \\ & \qquad \qquad + \frac{3\eta^2 c_{\nu}^2(1 + \rho^2)}{2}\sum_{\ell=\bar{t}_{s-1}}^{t-1}\alpha_{\ell}^2\sum_{\bar{\ell} = \bar{t}_{s-1}}^{\ell} \alpha_{\bar{\ell}}^2 D_{\bar{\ell}} \nonumber \\
\end{align*}
The second inequality uses the fact that $t -l \le I$ and the inequality $log(1+ a/x) \leq a/x$ for $x > -a$, so we have $(1+a/x)^x \leq e^{a/x}$, Then we choose $a = 33/32$ and $x = I$. Finally, we use the fact that $e^{33/(32I)} \leq e^{33/32} \leq 3$. 

Next we multiply $\alpha_t$ over both sides and take sum from $\bar{t}_{s-1} + 1$ to $\bar{t}_{s}$, we have:
\begin{align*}
    \sum_{t=\bar{t}_{s-1}+1}^{\bar{t}_s} \alpha_tD_t  & \leq
    24I L_h^2\eta^2\sum_{t=\bar{t}_{s-1}}^{\bar{t}_s - 1} \alpha_t\sum_{\ell=\bar{t}_{s-1}}^{t-1}\alpha_{\ell}^2E_{\ell} + 12 I L_h^2\gamma^2\sum_{t=\bar{t}_{s-1}}^{\bar{t}_s - 1} \alpha_t\sum_{\ell=\bar{t}_{s-1}}^{t-1}\alpha_{\ell}^2F_{\ell} \noindent \\
    & \qquad + \bigg(\frac{3c_{\nu}^2\sigma^2}{2L_h} + \frac{24c_{\nu}^2G^2}{L_h} +  \frac{ 3c_{\nu}^2\zeta^2}{L_h}\bigg)\sum_{t=\bar{t}_{s-1}}^{\bar{t}_s - 1} \alpha_t\sum_{\ell=\bar{t}_{s-1}}^{t-1}\alpha_{\ell}^3 + \frac{3\eta^2 c_{\nu}^2(1 + \rho^2)}{2}\sum_{t=\bar{t}_{s-1}}^{\bar{t}_s - 1} \alpha_t\sum_{\ell=\bar{t}_{s-1}}^{t-1}\alpha_{\ell}^2\sum_{\bar{\ell} = \bar{t}_{s-1}}^{\ell} \alpha_{\bar{\ell}}^2 D_{\bar{\ell}} \nonumber \\
    & \leq 24I L_h^2\eta^2\bigg(\sum_{t=\bar{t}_{s-1}}^{\bar{t}_s - 1} \alpha_t\bigg)\sum_{t=\bar{t}_{s-1}}^{\bar{t}_s-1}\alpha_{t}^2E_{t} + 12 I L_h^2\gamma^2\bigg(\sum_{t=\bar{t}_{s-1}}^{\bar{t}_s - 1} \alpha_t\bigg)\sum_{t=\bar{t}_{s-1}}^{\bar{t}_s-1}\alpha_{t}^2F_{t} \noindent \\
    & \qquad + \bigg(\frac{3c_{\nu}^2\sigma^2}{2L_h} + \frac{24c_{\nu}^2G^2}{L_h} +  \frac{ 3c_{\nu}^2\zeta^2}{L_h}\bigg)\bigg(\sum_{t=\bar{t}_{s-1}}^{\bar{t}_s - 1} \alpha_t\bigg)\sum_{t=\bar{t}_{s-1}}^{\bar{t}_s-1}\alpha_{t}^3 + \frac{3\eta^2 c_{\nu}^2(1 + \rho^2)}{2}\bigg(\sum_{t=\bar{t}_{s-1}}^{\bar{t}_s - 1} \alpha_t\bigg)\sum_{t=\bar{t}_{s-1}}^{\bar{t}_s-1}\alpha_{t}^2\sum_{\bar{\ell} = \bar{t}_{s-1}}^{t} \alpha_{\bar{\ell}}^2 D_{\bar{\ell}} \nonumber \\
    & \overset{(a)}{\leq} \frac{3I L_h\eta^2}{2}\sum_{t=\bar{t}_{s-1}}^{\bar{t}_s-1}\alpha_{t}^2E_{t} + \frac{3 I L_h\gamma^2}{4}\sum_{t=\bar{t}_{s-1}}^{\bar{t}_s-1}\alpha_{t}^2F_{t} \noindent \\
    & \qquad + \bigg(\frac{3c_{\nu}^2\sigma^2}{32L_h} + \frac{3c_{\nu}^2G^2}{2L_h} +  \frac{ 3c_{\nu}^2\zeta^2}{16L_h}\bigg)\sum_{t=\bar{t}_{s-1}}^{\bar{t}_s-1}\alpha_{t}^3 + \frac{3\eta^2 c_{\nu}^2(1 + \rho^2)}{32L_h}\sum_{t=\bar{t}_{s-1}}^{\bar{t}_s-1}\alpha_{t}^2\sum_{\bar{\ell} = \bar{t}_{s-1}}^{t} \alpha_{\bar{\ell}}^2 D_{\bar{\ell}} \nonumber \\
    & \overset{(b)}{\leq} \frac{3\eta^2}{32}\sum_{t=\bar{t}_{s-1}}^{\bar{t}_s-1}\alpha_{t}E_{t} + \frac{3\gamma^2}{64}\sum_{t=\bar{t}_{s-1}}^{\bar{t}_s-1}\alpha_{t}F_{t} + \bigg(\frac{3c_{\nu}^2\sigma^2}{32L_h} + \frac{3c_{\nu}^2G^2}{2L_h} +  \frac{ 3c_{\nu}^2\zeta^2}{16L_h}\bigg)\sum_{t=\bar{t}_{s-1}}^{\bar{t}_s-1}\alpha_{t}^3 + \frac{3\eta^2 c_{\nu}^2(1 + \rho^2)}{16^3*32I^2L_h^4}\sum_{\bar{\ell} = \bar{t}_{s-1}}^{\bar{t}_s-1} \alpha_{\bar{\ell}} D_{\bar{\ell}} \nonumber \\
\end{align*}
In inequalities $(a)$ and $(b)$, we use $\alpha_t < \frac{1}{16L_hI}$ multiple times. next notice that $\sum_{t=\bar{t}_{s-1}+1}^{\bar{t}_s} \alpha_tD_t = \sum_{t = \bar{t}_{s-1}}^{\bar{t}_s-1} \alpha_{t} D_{t}$ as $D_{\bar{t}_s} = D_{\bar{t}_{s-1}} =0$, so we have:
\begin{align*}
    \bigg(1 - \frac{3\eta^2 c_{\nu}^2(1 + \rho^2)}{16^3*32IL_h^4}\bigg)\sum_{t = \bar{t}_{s-1}}^{\bar{t}_s-1} \alpha_{t} D_{t} \leq \frac{3\eta^2}{32}\sum_{t=\bar{t}_{s-1}}^{\bar{t}_s-1}\alpha_{t}E_{t} + \frac{3\gamma^2}{64}\sum_{t=\bar{t}_{s-1}}^{\bar{t}_s-1}\alpha_{t}F_{t} + \bigg(\frac{3c_{\nu}^2\sigma^2}{32L_h} + \frac{3c_{\nu}^2G^2}{2L_h} +  \frac{ 3c_{\nu}^2\zeta^2}{16L_h}\bigg)\sum_{t=\bar{t}_{s-1}}^{\bar{t}_s-1}\alpha_{t}^3 \\
\end{align*}
\end{proof}
\subsection{Descent Lemma}
\begin{lemma}[Descent Lemma]
\label{lemma:desent_storm}
For all $t \in [\bar{t}_{s-1}, \bar{t}_s - 1]$ and $s \in [S]$, the iterates generated satisfy:
\begin{align*}
\mathbb{E}\bigg[  h(\bar{x}_{t + 1}) \bigg] & \leq \mathbb{E} \bigg[    h(\bar{x}_{t }) \bigg]-  \left( \frac{\eta\alpha_t}{2} -  \frac{\eta^2\alpha_t^2 L}{2}  \right)  \mathbb{E} \bigg[\Big\| \bar{\nu}_t  \Big\|^2 \bigg] - \frac{\eta\alpha_t}{2} \mathbb{E} \bigg[\|\nabla h(\bar{x}_t) \|^2 \bigg] \\
& \qquad + \frac{L_h^2I\eta^3\alpha_t}{M} \sum_{\ell = \bar{t}_{s-1}}^{t-1}  \alpha_l^2\sum_{m = 1}^M \Big\| \big(  \nu_\ell^{(m)} -  \bar{\nu}_\ell  \big) \Big\|^2 + \eta\alpha_t\mathbb{E} \bigg[ \Big\| \frac{1}{M} \sum_{m=1}^M  \nabla h(x^{(m)}_t)  - \bar{\nu}_{t}   \Big\|^2 \bigg]
\end{align*}
where the expectation is w.r.t the stochasticity of the algorithm.
\end{lemma}
\begin{proof}
Using the smoothness of $h(x)$ we have:
\begin{align*}
    \mathbb{E}[  h(\bar{x}_{t + 1}) ] 
    & \leq \mathbb{E} \Big[ h(\bar{x}_{t }) + \langle \nabla h(\bar{x}_{t}),  \bar{x}_{t + 1} - \bar{x}_{t}\rangle + \frac{L_h}{2} \| \bar{x}_{t + 1} - \bar{x}_{t } \|^2 \Big] \nonumber\\
    &  \overset{(a)}{=}\mathbb{E} \Big[ h(\bar{x}_{t}) - \eta\alpha_t \langle \nabla h(\bar{x}_{t}),  \bar{\nu}_t \rangle + \frac{\eta^2\alpha_t^2 L_h}{2} \| \bar{\nu}_{t}  \|^2 \Big] \nonumber\\
    & \overset{(b)}{=}    \mathbb{E} \bigg[ h(\bar{x}_{t}) - \frac{\eta\alpha_t}{2}  \Big\| \bar{\nu}_{t}  \Big\|^2  - \frac{\eta\alpha_t}{2} \| \nabla h(\bar{x}_{t}) \|^2   + \frac{\eta\alpha_t}{2} \Big\|  \nabla h(\bar{x}_{t})  - \bar{\nu}_{t}   \Big\|^2 + \frac{\eta\alpha_t^2 L_h}{2} \Big\| \bar{\nu}_{t} \Big\|^2  \bigg] \nonumber \\  
    & = \mathbb{E} \bigg[    h(\bar{x}_{t }) -  \left( \frac{\eta\alpha_t}{2} -  \frac{\eta^2\alpha_t^2 L_h}{2}  \right)  \Big\| \bar{\nu}_t  \Big\|^2 - \frac{\eta\alpha_t}{2} \|\nabla h(\bar{x}_t) \|^2  + \frac{\eta\alpha_t}{2} \Big\|  \nabla h(\bar{x}_{t})  - \bar{\nu}_{t}   \Big\|^2 \bigg] \nonumber \\
\end{align*}
where equality $(a)$ follows from the iterate update given in Line 15 of Algorithm~\ref{alg:FedBiOAcc}; $(b)$ uses $\langle a , b \rangle = \frac{1}{2} [\|a\|^2 + \|b\|^2 - \|a - b \|^2]$; For the last term, we have:
\begin{align*}
    \mathbb{E} \bigg[ \Big\|  \nabla h(\bar{x}_{t})  - \bar{\nu}_{t}   \Big\|^2  \bigg] & \overset{(a)}{\leq} 2\mathbb{E} \bigg[ \Big\|  \nabla h(\bar{x}_{t})  - \frac{1}{M} \sum_{m=1}^M \nabla h(x^{(m)}_t)   \Big\|^2 \bigg] + 2\mathbb{E} \bigg[ \Big\| \frac{1}{M} \sum_{m=1}^M  \nabla h(x^{(m)}_t)  - \bar{\nu}_{t}   \Big\|^2 \bigg] \nonumber \\
    & \overset{(b)}{\leq} \frac{2}{M} \sum_{m=1}^M\mathbb{E} \bigg[ \Big\|  \nabla h(\bar{x}_{t})  -  \nabla h(x^{(m)}_t)   \Big\|^2 \bigg] + 2\mathbb{E} \bigg[ \Big\| \frac{1}{M} \sum_{m=1}^M  \nabla h(x^{(m)}_t)  - \bar{\nu}_{t}   \Big\|^2 \bigg] \nonumber \\
    & \overset{(c)}{\leq} \frac{2L_h^2}{M} \sum_{m=1}^M\mathbb{E} \bigg[ \Big\| \bar{x}_{t}  -  x^{(m)}_t  \Big\|^2 \bigg] + 2\mathbb{E} \bigg[ \Big\| \frac{1}{M} \sum_{m=1}^M  \nabla h(x^{(m)}_t)  - \bar{\nu}_{t}   \Big\|^2 \bigg] \nonumber \\
    & \overset{(d)}{\leq} \frac{2L_h^2I\eta^2}{M} \sum_{\ell = \bar{t}_{s-1}}^{t-1}  \alpha_l^2\sum_{m = 1}^M \Big\| \big(  \nu_\ell^{(m)} -  \bar{\nu}_\ell  \big) \Big\|^2 + 2\mathbb{E} \bigg[ \Big\| \frac{1}{M} \sum_{m=1}^M  \nabla h(x^{(m)}_t)  - \bar{\nu}_{t}   \Big\|^2 \bigg]
\end{align*} 
where inequality (a) uses triangle inequality, (b) uses the generalized triangle inequality, (c) uses the smoothness of $h(x)$, (d) uses Eq.~\ref{Eq: ConsensusError_FedAvg_storm}. Combine the above two equations together, we get:
\begin{align*}
\mathbb{E}\bigg[  h(\bar{x}_{t + 1}) \bigg] & \leq \mathbb{E} \bigg[    h(\bar{x}_{t }) \bigg]-  \left( \frac{\eta\alpha_t}{2} -  \frac{\eta^2\alpha_t^2 L_h}{2}  \right)  \mathbb{E} \bigg[\Big\| \bar{\nu}_t  \Big\|^2 \bigg] - \frac{\eta\alpha_t}{2} \mathbb{E} \bigg[\|\nabla h(\bar{x}_t) \|^2 \bigg] \\
& \qquad + \frac{L_h^2I\eta^3\alpha_t}{M} \sum_{\ell = \bar{t}_{s-1}}^{t-1}  \alpha_l^2\sum_{m = 1}^M \Big\| \big(  \nu_\ell^{(m)} -  \bar{\nu}_\ell  \big) \Big\|^2 + \eta\alpha_t\mathbb{E} \bigg[ \Big\| \frac{1}{M} \sum_{m=1}^M  \nabla h(x^{(m)}_t)  - \bar{\nu}_{t}   \Big\|^2 \bigg]
\end{align*}
Hence, the lemma is proved.
\end{proof}
\subsection{Descent in Potential Function}
We first denote the following potential function $\mathcal{G}(t)$:
\begin{align*}
    \mathcal{G}_t &= h(\bar{x}_{t}) + \frac{\eta}{320L_h^2\alpha_{t}}\Big\| \bar{\nu}_{t} - \frac{1}{M} \sum_{m=1}^M  \nabla h(x^{(m)}_t)  \Big\|^2 + \frac{1}{M} \sum_{m=1}^M \bigg\|y^{(m)}_t - y^{(m)}_{x^{(m)}_{t}} \bigg\|^2 \nonumber \\
    & \qquad \qquad + \frac{\gamma}{32L^2\alpha_{t}}\times\frac{1}{M} \sum_{m=1}^M \bigg\|\omega^{(m)}_t - \nabla_y g^{(m)}(x^{(m)}_{t}, y^{(m)}_{t} ) \bigg\|^2
\end{align*}
\begin{lemma}
Suppose $\frac{1}{\gamma} > max(15L, 1)$, $\frac{1}{\eta} > max(\frac{600L^3}{\mu^3}  + \frac{1}{15L} + 1, \frac{12(1 + \rho^2)}{I^2} + \frac{97}{256} + \frac{1}{120L} + \frac{16I^2}{\mu}, \frac{960L_h^2}{\mu\gamma}, \frac{1}{\gamma}, \frac{40L_h^2}{\mu^3\gamma})$, $c_{\nu} = 320L_h^2 + \frac{\sigma^2 }{24 \delta^3 L_hI}$, $c_{\omega} = \frac{160L^2}{\mu} + \frac{\sigma^2}{24 \delta^3 L_hI}$, $u = max(2\sigma^2, \delta^3, c_{\nu}^{3/2}\delta^3, c_{\omega}^{3/2}\delta^3, 16^3I^3)$, $\delta = \frac{\sigma^{2/3}}{16IL_h}$ then we have:
\begin{align*}
    \mathbb{E}[\mathcal{G}_{\bar{t}_s}] - \mathbb{E}[\mathcal{G}_{\bar{t}_{s - 1}}] 
    & \leq - \sum_{t = \bar{t}_{s-1}}^{\bar{t}_s-1}\frac{\eta\alpha_t}{2} \mathbb{E} \bigg[ \|\nabla h(\bar{x}_t) \|^2 \bigg] + \bigg(\frac{c_{\omega}^2\sigma^2}{16L^2} +  \frac{9c_{\nu}^2\sigma^2}{L_h^2} + \frac{ 3c_{\nu}^2\zeta^2}{16L_h^2}\bigg) \sum_{t=\bar{t}_{s-1}}^{\bar{t}_s-1}\alpha_{t}^3\nonumber\\
\end{align*}
where the expectation is w.r.t the stochasticity of the algorithm.
\label{lemma:potential_descent}
\end{lemma}
Take expectation for both sides of the potential function and we use the notation used in Lemma~\ref{lemma:d_bound}, the potential function has the following form:
\begin{align*}
    \mathbb{E}[\mathcal{G}_t] = \mathbb{E} [h(\bar{x}_{t})] + \frac{\eta A_t}{\hat{c}_{\nu}\alpha_{t}} + B_t + \frac{\gamma C_t}{\hat{c}_{\omega}\alpha_{t}}
\end{align*}
We first bound the term $A_t/\alpha_{t-1}  - A_{t-1}/\alpha_{t-2}$. For $t \in [\bar{t}_{s-1} + 1, \bar{t}_s]$. By the condition that $u \ge c_{\nu}^{3/2}\delta^3$, it is straightforward to verify that $c_{\nu}\alpha^2 < 1$. Then we rewrite Lemma~\ref{lemma:hg_bound_storm} as follows using our new notation:
\begin{align*}
   A_t & \leq ( 1 - c_{\nu}\alpha_{t-1}^2)^2 A_{t-1} + 4(c_{\nu}\alpha_{t-1}^2)^2\sigma^2/M + 8(c_{\nu}\alpha_{t-1}^2)^2G^2 + 8L_h^2(c_{\nu}\alpha_{t-1}^2)^2 B_{t-1} + 40L_h^2\eta^2\alpha_{t-1}^2D_{t-1} \\
  & \qquad + 40L_h^2\eta^2\alpha_{t-1}^2E_{t-1}\nonumber + 12L_h^2\gamma^2\alpha_{t-1}^2F_{t-1}
\end{align*}
Naturally, we get:
\begin{align*}
   \frac{A_t}{\alpha_{t-1}} - \frac{A_{t-1}}{\alpha_{t-2}} & \leq \bigg(\frac{( 1 - c_{\nu}\alpha_{t-1}^2)^2}{\alpha_{t-1}} - \frac{1}{\alpha_{t-2}}\bigg) A_{t-1} + 4c_{\nu}^2\alpha_{t-1}^3\sigma^2/M + 8c_{\nu}^2\alpha_{t-1}^3G^2 + 8L_h^2c_{\nu}^2\alpha_{t-1}^3 B_{t-1} \\
  & \qquad+ 40L_h^2\eta^2\alpha_{t-1}D_{t-1}  + 40L_h^2\alpha_{t-1}E_{t-1} + 12L_h^2\gamma^2\alpha_{t-1}F_{t-1} \\
  & \leq \bigg(\alpha_{t-1}^{-1} - \alpha_{t-2}^{-1} - c_{\nu}\alpha_{t-1}\bigg) A_{t-1} + 4c_{\nu}^2\alpha_{t-1}^3\sigma^2/M + 8c_{\nu}^2\alpha_{t-1}^3G^2 + 8L_h^2c_{\nu}^2\alpha_{t-1}^3 B_{t-1} \\
  & \qquad + 40L_h^2\eta^2\alpha_{t-1}D_{t-1}  + 40L_h^2\eta^2\alpha_{t-1}E_{t-1} + 12L_h^2(1 - c_{\nu}\alpha_{t-1}^2)^2\gamma^2\alpha_{t-1}F_{t-1} \\
\end{align*}
where the inequality is due to the fact that $(1 - c_{\nu}a_{t-1}^2)^2 \leq 1 - c_{\nu}a_{t-1}^2 \leq 1$ for all $t \in [T]$. Next for the term $\alpha_{t-1}^{-1} - \alpha_{t-2}^{-1}$ we have:
\begin{align*}
    \alpha_{t}^{-1} - \alpha_{t-1}^{-1} & =  \frac{(u + \sigma^2 t)^{1/3}}{\delta} -  \frac{(u + \sigma^2 (t-1))^{1/3}}{\delta} 
    & \overset{(a)}{\leq}  \frac{\sigma^2}{3 \delta (u + \sigma^2 (t-1))^{2/3}} \overset{(b)}{\leq} \frac{2^{2/3} \sigma^2 \delta^2}{3 \delta^3 (u + \sigma^2 t)^{2/3}} \overset{(c)}{=} \frac{2^{2/3} \sigma^2}{3 \delta^3 } \alpha_{t}^2 {\overset{(d)}{\leq} \frac{ \sigma^2 }{24 \delta^3 LI} \alpha_{t}}
\end{align*}
where inequality $(a)$ results from the concavity of $x^{1/3}$ as: $(x + y)^{1/3} - x^{1/3} \leq y/3x^{2/3}$, inequality $(b)$ used the fact that $u_t \geq 2\sigma^2$, inequality $(c)$ uses the definition of $\alpha_t$, inequality $(d)$ uses $u \ge 16^3I^3$, so that $\alpha_t \leq \frac{1}{16L_hI}$ for all $t \in [T]$. Since we have $c_{\nu} = \hat{c}_{\nu} + \frac{\sigma^2 }{24 \delta^3 L_hI}$, where $\hat{c}_{\nu} = 320L_h^2$ is some constant. It is straightforward to verify that if we set $\delta = \frac{\sigma^{2/3}}{L_h}$, we have $c_{\nu} \leq 2\hat{c}_{\nu}$. Next, we have:
\begin{align*}
   \frac{A_t}{\alpha_{t-1}} - \frac{A_{t-1}}{\alpha_{t-2}} & \leq  - \hat{c}_{\nu}\alpha_{t-1} A_{t-1} + 4c_{\nu}^2\alpha_{t-1}^3\sigma^2/M + 8c_{\nu}^2\alpha_{t-1}^3G^2 + 8L_h^2c_{\nu}^2\alpha_{t-1}^3 B_{t-1} + 40L_h^2\eta^2\alpha_{t-1}D_{t-1}\\
  & \qquad + 40L_h^2\eta^2\alpha_{t-1}E_{t-1} + 12L_h^2\gamma^2\alpha_{t-1}F_{t-1}
\end{align*}
Then We multiply $\eta/\hat{c}_{\nu}$ on both sides and have:
\begin{align*}
    \frac{\eta}{\hat{c}_{\nu}}\bigg(\frac{A_t}{\alpha_{t-1}} - \frac{A_{t-1}}{\alpha_{t-2}}\bigg) & \leq  - \eta\alpha_{t-1} A_{t-1} + 4c_{\nu}^2\alpha_{t-1}^3\sigma^2/(\hat{c}_{\nu}M) + 8c_{\nu}^2\eta\alpha_{t-1}^3G^2/\hat{c}_{\nu} + 8L_h^2c_{\nu}^2\eta\alpha_{t-1}^3 B_{t-1}/\hat{c}_{\nu} \nonumber \\ &+ 40L_h^2\eta^3\alpha_{t-1}D_{t-1}/\hat{c}_{\nu}  + 40L_h^2\eta^3\alpha_{t-1}E_{t-1}/\hat{c}_{\nu} + 12L_h^2\gamma^2\eta\alpha_{t-1}F_{t-1}/\hat{c}_{\nu} \nonumber \\
\end{align*}
By telescoping from $\bar{t}_{s-1} + 1$ to $\bar{t}_{s}$, we have:
\begin{align}
    \frac{\eta}{\hat{c}_{\nu}}\bigg(\frac{A_{\bar{t}_s}}{\alpha_{\bar{t}_{s}-1}} - \frac{A_{\bar{t}_{s-1}}}{\alpha_{\bar{t}_{s-1} - 1}}\bigg) & \leq  - \sum_{t=\bar{t}_{s-1}}^{\bar{t}_s-1}\eta\alpha_{t} A_{t} + 4\eta c_{\nu}^2\sigma^2/(\hat{c}_{\nu}M) \sum_{t=\bar{t}_{s-1}}^{\bar{t}_s-1}\alpha_{t}^3 + 8\eta c_{\nu}^2G^2/\hat{c}_{\nu} \sum_{t=\bar{t}_{s-1}}^{\bar{t}_s-1}\alpha_{t}^3 + 8L_h^2\eta c_{\nu}^2/\hat{c}_{\nu}\sum_{t=\bar{t}_{s-1}}^{\bar{t}_s-1}\alpha_{t}^3 B_{t}\nonumber\\
    & \qquad + 40L_h^2\eta^3/\hat{c}_{\nu} \sum_{t=\bar{t}_{s-1}}^{\bar{t}_s-1}\alpha_{t}D_{t} + 40L_h^2\eta^3/\hat{c}_{\nu}\sum_{t=\bar{t}_{s-1}}^{\bar{t}_s-1}\alpha_{t}E_{t} + 12L_h^2\eta\gamma^2/\hat{c}_{\nu}\sum_{t=\bar{t}_{s-1}}^{\bar{t}_s-1}\alpha_{t}F_{t} \nonumber \\
\label{eq:A_tele}
\end{align}
Similarly, by the condition $u \ge c_{\omega}^{3/2}\delta^3$, the condition of Lemma~\ref{lemma: inner_est_error_storm} satisfies. For $t \in [\bar{t}_{s-1} + 1, \bar{t}_s]$, we have:
\begin{align*}
    C_t & \leq ( 1 - c_{\omega}\alpha_{t-1}^2)^2C_{t-1} + 2(c_{\omega}\alpha_{t-1}^2)^2\sigma^2 + 4(1 - c_{\omega}\alpha_{t-1}^2)^2L^2\eta^2\alpha_{t-1}^2(D_{t-1} + E_{t-1}) + 2(1 - c_{\omega}\alpha_{t-1}^2)^2L^2\gamma^2\alpha_{t-1}^2F_{t-1}\nonumber \\
\end{align*}
We bound the term $C_t/\alpha_{t-1}  - C_{t-1}/\alpha_{t-2}$ and follow similar derivation as $A_t/\alpha_{t-1}  - A_{t-1}/\alpha_{t-2}$ and since we have  $c_{\omega} = \frac{5\hat{c}_{\omega}}{\mu} + \frac{\sigma^2}{24 \delta^3 L_hI}$, where $\hat{c}_{\omega} = 32L^2$ is some constant.  we get:
\begin{align*}
    \frac{C_t}{\alpha_{t-1}} - \frac{C_{t-1}}{\alpha_{t-2}} & \leq -\frac{5\hat{c}_{\omega}}{\mu}\alpha_{t-1}C_{t-1} + 2c_{\omega}^2\alpha_{t-1}^3\sigma^2 + 4(1 - c_{\omega}\alpha_{t-1}^2)^2L^2\eta^2\alpha_{t-1}(D_{t-1} + E_{t-1}) + 2(1 - c_{\omega}\alpha_{t-1}^2)^2L^2\gamma^2\alpha_{t-1}F_{t-1}\nonumber \\
    & \leq -\frac{5\hat{c}_{\omega}}{\mu}\alpha_{t-1}C_{t-1} + 2c_{\omega}^2\alpha_{t-1}^3\sigma^2 + 4L^2\eta^2\alpha_{t-1}(D_{t-1} + E_{t-1}) + 2L^2\gamma^2\alpha_{t-1}F_{t-1}\nonumber \\
\end{align*}
Divide $\gamma/\hat{c}_{\omega}$ for both sides and then telescope from $\bar{t}_{s-1} + 1$ to $\bar{t}_{s}$, we have:
\begin{align}
    \frac{\gamma}{\hat{c}_{\omega}}\bigg(\frac{C_{\bar{t}_s}}{\alpha_{\bar{t}_{s}-1}} - \frac{C_{\bar{t}_{s-1}}}{\alpha_{\bar{t}_{s-1} - 1}}\bigg)& \leq -\sum_{t=\bar{t}_{s-1}}^{\bar{t}_s-1}\frac{5\gamma}{\mu}\alpha_{t}C_{t} + 2\gamma c_{\omega}^2\sigma^2/\hat{c}_{\omega}\sum_{t=\bar{t}_{s-1}}^{\bar{t}_s-1}\alpha_{t}^3 + 4L^2\gamma\eta^2/\hat{c}_{\omega}\sum_{t=\bar{t}_{s-1}}^{\bar{t}_s-1}\alpha_{t}(D_{t} + E_{t}) + 2L^2\gamma^3/\hat{c}_{\omega}\sum_{t=\bar{t}_{s-1}}^{\bar{t}_s-1}\alpha_{t}F_{t}\nonumber \\
\label{eq:C_tele}
\end{align}
Next, since we have $u \ge \delta^3$, the condition of Lemma~\ref{lemma: inner_drift_storm} satisfies, we write it as follows, first for $t \in [\bar{t}_{s-1} + 1, \bar{t}_s - 1]$, we have:
\begin{align*}
    B_t  & \leq \bigg(1 - \frac{\mu\gamma\alpha_{t-1}}{8}\bigg)B_{t-1} - \frac{3\gamma^2\alpha_{t-1}F_{t-1}}{4}  + \frac{5\gamma\alpha_{t-1}C_{t-1}}{\mu} + \frac{10L^2\eta^2\alpha_{t-1}D_{t-1}}{\mu^3\gamma} + \frac{10L^2\eta^2\alpha_{t-1}E_{t-1}}{\mu^3\gamma}
\end{align*}
and when $t =  \bar{t}_s$, we have:
\begin{align*}
    B_t  & \leq \bigg(1 - \frac{\mu\gamma\alpha_{t-1}}{8}\bigg)B_{t-1} - \frac{3\gamma^2\alpha_{t-1}F_{t-1}}{4}  + \frac{5\gamma\alpha_{t-1}C_{t-1}}{\mu} + \frac{10L^2\eta^2\alpha_{t-1}D_{t-1}}{\mu^3\gamma} \\
    & \qquad + \frac{10L^2\eta^2\alpha_{t-1}E_{t-1}}{\mu^3\gamma} + (1+\frac{8}{\mu\gamma\alpha_{t - 1}})I\eta^2\sum_{\ell = \bar{t}_{s-1}}^{t-1} \alpha_l^2D_l 
\end{align*}
We telescope from $\bar{t}_{s-1} + 1$ to $\bar{t}_{s}$ and have:
\begin{align*}
    B_{\bar{t}_s} - B_{\bar{t}_{s-1}}  & \leq  - \frac{\mu\gamma}{8} \sum_{t=\bar{t}_{s-1}}^{\bar{t}_s-1}\alpha_{t}B_{t} - \frac{3\gamma^2}{4} \sum_{t=\bar{t}_{s-1}}^{\bar{t}_s-1}\alpha_{t}F_{t}  + \frac{5\gamma}{\mu} \sum_{t=\bar{t}_{s-1}}^{\bar{t}_s-1}\alpha_{t}C_{t} + \frac{10L^2\eta^2}{\mu^3\gamma} \sum_{t=\bar{t}_{s-1}}^{\bar{t}_s-1}\alpha_{t}D_{t} \\
    & \qquad + \frac{10L^2\eta^2}{\mu^3\gamma} \sum_{t=\bar{t}_{s-1}}^{\bar{t}_s-1} \alpha_{t}E_{t} + (1+\frac{8}{\mu\gamma\alpha_{\bar{t}_s - 1}})I\eta^2\sum_{t = \bar{t}_{s-1}}^{\bar{t}_s-1} \alpha_t^2D_t \\
    & \leq  - \frac{\mu\gamma}{8} \sum_{t=\bar{t}_{s-1}}^{\bar{t}_s-1}\alpha_{t}B_{t} - \frac{3\gamma^2}{4} \sum_{t=\bar{t}_{s-1}}^{\bar{t}_s-1}\alpha_{t}F_{t}  + \frac{5\gamma}{\mu} \sum_{t=\bar{t}_{s-1}}^{\bar{t}_s-1}\alpha_{t}C_{t} + \frac{10L^2\eta^2}{\mu^3\gamma} \sum_{t=\bar{t}_{s-1}}^{\bar{t}_s-1}\alpha_{t}D_{t} \\
    & \qquad + \frac{10L^2\eta^2}{\mu^3\gamma} \sum_{t=\bar{t}_{s-1}}^{\bar{t}_s-1} \alpha_{t}E_{t} + (1+\frac{8}{\mu\gamma\alpha_{\bar{t}_s - 1}})I\eta^2\sum_{t = \bar{t}_{s-1}}^{\bar{t}_s-1} \alpha_t^2D_t \\   
\end{align*}
Next for $\alpha_t/\alpha_{\bar{t}_s - 1}$, we have:
\begin{align*}
    \frac{\alpha_t}{\alpha_{\bar{t}_s - 1}} & = \frac{(u + \sigma^2 (\bar{t}_s - 1))^{1/3}}{(u + \sigma^2 t)^{1/3}} = \bigg(1 + \frac{u + \sigma^2 (\bar{t}_s - 1) - u - \sigma^2 t}{u + \sigma^2 t}\bigg)^{1/3} \\
    & \leq \bigg(1 + \frac{(I-1)\sigma^2}{u + \sigma^2 t}\bigg)^{1/3} \leq 1 + \frac{(I-1)}{3(t + 2)}  \leq I
\end{align*}
The third inequality is by the fact that  $0 < \bar{t}_s - 1 - t < I - 1$, the fourth inequality uses the concavity of $x^{1/3}$ as: $(x + y)^{1/3} - x^{1/3} \leq y/3x^{2/3}$. The second last inequality uses the fact that $u \ge 2\sigma^2$, while the last inequality holds for $I \ge 1$ and $t \ge 0 $. Next we have:
\begin{align*}
    (1+\frac{8}{\mu\gamma\alpha_{\bar{t}_s - 1}})\alpha_t^2 = \alpha_t^2 + \frac{8\alpha_t^2}{\mu\gamma\alpha_{\bar{t}_s - 1}} < \alpha_t + \frac{8I\alpha_t}{\mu\gamma} \leq \frac{16I\alpha_t}{\mu}
\end{align*}
The last inequality is because we have $\gamma < \frac{1}{15\mu} < \frac{8I}{\mu}$. Finally, we have:
\begin{align}
    B_{\bar{t}_s} - B_{\bar{t}_{s-1}} & \leq  - \frac{\mu\gamma}{8} \sum_{t=\bar{t}_{s-1}}^{\bar{t}_s-1}\alpha_{t}B_{t} - \frac{3\gamma^2}{4} \sum_{t=\bar{t}_{s-1}}^{\bar{t}_s-1}\alpha_{t}F_{t}  + \frac{5\gamma}{\mu} \sum_{t=\bar{t}_{s-1}}^{\bar{t}_s-1}\alpha_{t}C_{t} + \frac{10L^2\eta^2}{\mu^3\gamma} \sum_{t=\bar{t}_{s-1}}^{\bar{t}_s-1}\alpha_{t}D_{t} \nonumber\\
    & \qquad + \frac{10L^2\eta^2}{\mu^3\gamma} \sum_{t=\bar{t}_{s-1}}^{\bar{t}_s-1} \alpha_{t}E_{t} + \frac{16I^2\eta^2}{\mu}\sum_{t = \bar{t}_{s-1}}^{\bar{t}_s-1} \alpha_tD_t \nonumber\\
    & \leq  - \frac{\mu\gamma}{8} \sum_{t=\bar{t}_{s-1}}^{\bar{t}_s-1}\alpha_{t}B_{t} - \frac{3\gamma^2}{4} \sum_{t=\bar{t}_{s-1}}^{\bar{t}_s-1}\alpha_{t}F_{t}  + \frac{5\gamma}{\mu} \sum_{t=\bar{t}_{s-1}}^{\bar{t}_s-1}\alpha_{t}C_{t} \nonumber\\
    & \qquad + \frac{10L^2\eta^2}{\mu^3\gamma} \sum_{t=\bar{t}_{s-1}}^{\bar{t}_s-1} \alpha_{t}E_{t} + \bigg(\frac{16I^2\eta^2}{\mu}+ \frac{10L^2\eta^2}{\mu^3\gamma}\bigg)\sum_{t = \bar{t}_{s-1}}^{\bar{t}_s-1} \alpha_tD_t \nonumber\\ 
\label{eq:B_tele}
\end{align}
Next we rewrite Lemma~\ref{lemma:desent_storm} as follows:
\begin{align*}
    \mathbb{E} \bigg[  h(\bar{x}_{t + 1}) \bigg]  \leq \mathbb{E} \bigg[    h(\bar{x}_{t }) \bigg] -  \left( \frac{\eta\alpha_t}{2} -  \frac{\eta^2\alpha_t^2 L_h}{2}  \right)  E_t - \frac{\eta\alpha_t}{2} \mathbb{E} \bigg[ \|\nabla h(\bar{x}_t) \|^2 \bigg] + L_h^2I\eta^3\alpha_t \sum_{\ell = \bar{t}_{s-1}}^{t-1}  \alpha_l^2D_l + \eta\alpha_t A_t \nonumber\\
\end{align*}
We telescope from $\bar{t}_{s-1}$ to $\bar{t}_{s} - 1$ to have:
\begin{align}
    \mathbb{E} \bigg[  h(\bar{x}_{\bar{t}_{s}}) - h(\bar{x}_{\bar{t}_{s - 1} }) \bigg] & \leq - \sum_{t = \bar{t}_{s-1}}^{\bar{t}_s-1}\left( \frac{\eta\alpha_t}{2} - \frac{\eta^2\alpha_t^2 L}{2}  \right)  E_t -  \sum_{t = \bar{t}_{s-1}}^{\bar{t}_s-1}\frac{\eta\alpha_t}{2} \mathbb{E} \bigg[ \|\nabla h(\bar{x}_t) \|^2 \bigg] \nonumber\\
    & \qquad + L_h^2I\eta^3\sum_{t = \bar{t}_{s-1}}^{\bar{t}_s-1}\alpha_t \sum_{\ell = \bar{t}_{s-1}}^{t-1}  \alpha_l^2D_l + \sum_{t = \bar{t}_{s-1}}^{\bar{t}_s-1}\eta\alpha_t A_t \nonumber\\
    & \leq - \sum_{t = \bar{t}_{s-1}}^{\bar{t}_s-1}\left( \frac{\eta\alpha_t}{2} -  \frac{\eta^2\alpha_t^2 L}{2}  \right)  E_t - \sum_{t = \bar{t}_{s-1}}^{\bar{t}_s-1}\frac{\eta\alpha_t}{2} \mathbb{E} \bigg[ \|\nabla h(\bar{x}_t) \|^2 \bigg] \nonumber\\
    & \qquad + L_h^2I\eta^3\bigg(\sum_{t = \bar{t}_{s-1}}^{\bar{t}_s-1}\alpha_t \bigg)\sum_{t = \bar{t}_{s-1}}^{\bar{t}_s-1}  \alpha_t^2D_t + \sum_{t = \bar{t}_{s-1}}^{\bar{t}_s-1}\eta\alpha_t A_t \nonumber\\
    & \leq - \sum_{t = \bar{t}_{s-1}}^{\bar{t}_s-1}\left( \frac{\eta\alpha_t}{2} -  \frac{\eta^2\alpha_t^2 L}{2}  \right)  E_t - \sum_{t = \bar{t}_{s-1}}^{\bar{t}_s-1}\frac{\eta\alpha_t}{2} \mathbb{E} \bigg[ \|\nabla h(\bar{x}_t) \|^2 \bigg] \nonumber\\
    & \qquad + \frac{\eta^3}{256}\sum_{t = \bar{t}_{s-1}}^{\bar{t}_s-1}  \alpha_t D_t + \sum_{t = \bar{t}_{s-1}}^{\bar{t}_s-1}\eta\alpha_t A_t \nonumber\\
\label{eq:h_tele}
\end{align}
In the last inequality, we use the fact that $\bar{t}_s - \bar{t}_{s-1}  \leq I$  and $\alpha_t < \frac{1}{16L_hI}$.

Recall that Potential function is defined as:
\begin{align*}
    \mathbb{E}[\mathcal{G}_t] = \mathbb{E} [h(\bar{x}_{t})] + \frac{A_t}{\hat{c}_{\nu}\alpha_{t}} + B_t + \frac{C_t}{\hat{c}_{\omega}\alpha_{t}}
\end{align*}
Combine Eq.~(\ref{eq:A_tele}), Eq.~(\ref{eq:B_tele}), Eq.~(\ref{eq:C_tele}) and Eq.~(\ref{eq:h_tele}) and we have:
\begin{align*}
    \mathbb{E}[\mathcal{G}_{\bar{t}_s}] - \mathbb{E}[\mathcal{G}_{\bar{t}_{-1}s}] & \leq - \sum_{t = \bar{t}_{s-1}}^{\bar{t}_s-1}\frac{\eta\alpha_t}{2} \mathbb{E} \bigg[ \|\nabla h(\bar{x}_t) \|^2 \bigg] + \bigg(2c_{\omega}^2\gamma\sigma^2/\hat{c}_{\omega} +  4\eta c_{\nu}^2\sigma^2/(\hat{c}_{\nu}M) + 8\eta c_{\nu}^2G^2/\hat{c}_{\nu}\bigg) \sum_{t=\bar{t}_{s-1}}^{\bar{t}_s-1}\alpha_{t}^3 \nonumber\\
    & \qquad  -\bigg(\frac{3}{4}  - 2L^2\gamma/\hat{c}_{\omega}  - 12\eta L_h^2/\hat{c}_{\nu}\bigg)\gamma^2\sum_{t=\bar{t}_{s-1}}^{\bar{t}_s-1}\alpha_{t}F_{t} \nonumber \\
    & \qquad -\bigg(\frac{\eta}{2} -  \frac{\eta^2\alpha_t L_h}{2} - \frac{10L^2\eta^2}{\mu^3\gamma}  -  4L^2\gamma\eta^2/\hat{c}_{\omega} - 40L_h^2\eta^3/\hat{c}_{\nu} \bigg) \sum_{t=\bar{t}_{s-1}}^{\bar{t}_s-1}\alpha_{t}E_{t}\nonumber\\
    & \qquad   - \sum_{t=\bar{t}_{s-1}}^{\bar{t}_s-1} \bigg(\frac{\mu\gamma}{8} - 12L_h^2c_{\nu}^2\eta\alpha_{t}^2/\hat{c}_{\nu}\bigg) \alpha_{t}B_{t} \nonumber \\
    & \qquad +  \bigg(\frac{\eta}{256} + \frac{16I^2}{\mu} + \frac{10L^2}{\mu^3\gamma} + 4L^2\gamma/\hat{c}_{\omega} + 40\eta L_h^2/\hat{c}_{\nu} \bigg)\eta^2 \sum_{t=\bar{t}_{s-1}}^{\bar{t}_s-1}\alpha_{t}D_{t}
\end{align*}
Since we take $\hat{c}_{\omega} = 32L^2$, $\hat{c}_{\nu} = 320L_h^2$, $\alpha_t < \frac{1}{16L_hI}$, $\frac{1}{\gamma} > max(15L, 1)$, $\frac{1}{\eta} > max(\frac{600L^3}{\mu^3}  + \frac{1}{15L} + 1, \frac{12(1 + \rho^2)}{I^2} + \frac{97}{256} + \frac{1}{120L} + \frac{16I^2}{\mu}, \frac{960L_h^2}{\mu\gamma}, \frac{1}{\gamma}, \frac{40L_h^2}{\mu^3\gamma})$. then we have:
\begin{align}
    \mathbb{E}[\mathcal{G}_{\bar{t}_s}] - \mathbb{E}[\mathcal{G}_{\bar{t}_{s - 1}}] & \leq - \sum_{t = \bar{t}_{s-1}}^{\bar{t}_s-1}\frac{\eta\alpha_t}{2} \mathbb{E} \bigg[ \|\nabla h(\bar{x}_t) \|^2 \bigg] + \bigg(\frac{c_{\omega}^2\sigma^2}{16L^2} +  \frac{3c_{\nu}^2\sigma^2}{80L_h^2}\bigg) \sum_{t=\bar{t}_{s-1}}^{\bar{t}_s-1}\alpha_{t}^3 \nonumber\\
    & \qquad  -\frac{5}{8}\gamma^2\sum_{t=\bar{t}_{s-1}}^{\bar{t}_s-1}\alpha_{t}F_{t} -\frac{\eta}{4} \sum_{t=\bar{t}_{s-1}}^{\bar{t}_s-1}\alpha_{t}E_{t}   - \frac{\mu}{320L_h}\sum_{t=\bar{t}_{s-1}}^{\bar{t}_s-1} \alpha_{t}B_{t} \nonumber \\
    & \qquad +  \bigg(1 - \frac{6(1 + \rho^2)}{I^2}\eta^2\bigg) \sum_{t=\bar{t}_{s-1}}^{\bar{t}_s-1}\alpha_{t}D_{t} \nonumber\\
\label{eq:phi_bound}
\end{align}
For the term related to $F_t$, we have:
\begin{align*}
    \frac{3}{4}  - 2L^2\gamma/\hat{c}_{\omega}  - 12\eta L_h^2/\hat{c}_{\nu} = \frac{3}{4} - \frac{\gamma}{16} - \frac{3\eta}{80} \ge \frac{3}{4} -\frac{\gamma}{10} \ge \frac{1}{2}
\end{align*}
where the first inequality follows $\eta \leq \gamma$; the second inequality follows that $\gamma < 1$; Next for the term related to $E_t$, we have:
\begin{align*}
    &\frac{\eta}{2} -  \frac{\eta^2\alpha_t L_h}{2} - \frac{10L^2\eta^2}{\mu^3\gamma}  -  4L^2\gamma\eta^2/\hat{c}_{\omega} - 40L_h^2\eta^3/\hat{c}_{\nu} \ge \frac{\eta}{2} -  \frac{\eta^2}{32I} - \frac{\eta}{4}  -  \frac{(\gamma+\eta)\eta^2}{8} \\
    &\ge \frac{\eta}{4} -  \frac{\eta^2}{8}\bigg(\frac{1}{4I}  + \gamma + \eta\bigg) \ge  \frac{\eta}{4} - \frac{\eta^2}{8}\bigg(\frac{1}{4I}  + \gamma + \frac{\mu^3\gamma}{40L^2}\bigg) \ge \frac{\eta}{4} - \frac{\eta^2}{8}\bigg(\frac{1}{4I}  + \frac{1}{15L} + \frac{\mu^3}{600L^3}\bigg) \ge \frac{\eta}{8}
\end{align*}
where the first inequality is because $\hat{c}_{\omega} = 32L^2$, $\hat{c}_{\nu} = 320L_h^2$, $\alpha_t < \frac{1}{16L_hI}$, $\eta < \frac{\mu^3\gamma}{40L_h^2} < \frac{\mu^3\gamma}{40L^2}$ and the second last inequality is because $\gamma < \frac{1}{15L}$ and $\frac{1}{\eta} \ge \frac{600L^3}{\mu^3}  + \frac{1}{15L} + 1$. Next for the term related to $B_t$, we have:
\begin{align*}
    \frac{\mu\gamma}{8} - 12L_h^2c_{\nu}^2\eta\alpha_{t}^2/\hat{c}_{\nu} \ge  \frac{\mu\gamma}{8} - \frac{48L_h^2\eta\hat{c}_{\nu}}{16^2L_h^2I^2} = \frac{\mu\gamma}{8} - \frac{60L_h^2\eta}{I^2} \ge \frac{\mu\gamma}{8} - \frac{\mu\gamma}{16I^2} \ge \frac{\mu\gamma}{16}
\end{align*}
The first inequality is by $c_\nu \leq 2\hat{c}_{\nu}$ and $\alpha_t < \frac{1}{16LI}$; the second last inequality is by $\eta < \frac{\mu\gamma}{960L_h^2}$. Lastly, for the term related to $D_t$, we have:
\begin{align*}
    & \bigg(\frac{\eta^2}{256} + \frac{16I^2\eta}{\mu} + \frac{10L^2\eta}{\mu^3\gamma} + 4L^2\gamma\eta/\hat{c}_{\omega} + 40L_h^2\eta^2/\hat{c}_{\nu}\bigg)\eta \leq \bigg(\frac{\eta^2}{256} + \frac{16I^2}{\mu} + \frac{1}{4} + \frac{\gamma\eta}{8} + \frac{\eta^2}{8}\bigg)\eta \nonumber \\
    & \qquad \leq \bigg(\frac{97}{256} + \frac{1}{120L} + \frac{16I^2}{\mu} \bigg)\eta \leq  1 - \frac{12(1 + \rho^2)}{I^2}\eta
\end{align*}
The first inequality is by $\hat{c}_{\omega} = 32L^2$, $\hat{c}_{\nu} = 320L_h^2$ and $\eta < \frac{\mu^3\gamma}{40L^2}$; the second inequality is by $\eta < 1$ and $\gamma < \frac{1}{15L}$; The last inequality is by $\frac{1}{\eta} \ge \frac{12(1 + \rho^2)}{I^2} + \frac{97}{256} + \frac{1}{120L} + \frac{16I^2}{\mu}$. Next, by Lemma~\ref{lemma:d_bound}, we have:
\begin{align}
    \bigg(1 - \frac{3\eta^2 c_{\nu}^2(1 + \rho^2)}{16^3*32I^2L_h^4}\bigg)\sum_{t = \bar{t}_{s-1}}^{\bar{t}_s-1} \alpha_{t} D_{t} \leq \frac{3\eta^2}{32}\sum_{t=\bar{t}_{s-1}}^{\bar{t}_s-1}\alpha_{t}E_{t} + \frac{3\gamma^2}{64}\sum_{t=\bar{t}_{s-1}}^{\bar{t}_s-1}\alpha_{t}F_{t} + \bigg(\frac{3c_{\nu}^2\sigma^2}{32L_h} + \frac{3c_{\nu}^2G^2}{2L_h} +  \frac{ 3c_{\nu}^2\zeta^2}{16L_h}\bigg)\sum_{t=\bar{t}_{s-1}}^{\bar{t}_s-1}\alpha_{t}^3 \nonumber\\
\label{eq:d_bound}
\end{align}
since we have:
\begin{align*}
    \frac{c_{\nu}^2}{16^3*32L_h^4} \leq \frac{4\hat{c}_{\nu}^2}{16^3*32L_h^4} \leq \frac{4*320^2L_h^4}{16^3*32L_h^4} = 25/8 < 4
\end{align*}
The first inequality is by $c_{\nu} \leq  2\hat{c}_{\nu}$. So we have:
\begin{align*}
  \frac{3\eta^2 c_{\nu}^2(1 + \rho^2)}{16^3*64I^2L_h^4} \leq  \frac{12\eta^2(1 + \rho^2)}{I^2} \leq \frac{12\eta(1 + \rho^2)}{I^2}
\end{align*}
Combine Eq.~(\ref{eq:phi_bound}) and Eq.~(\ref{eq:d_bound}) and use $\gamma < 1$ and $\eta < 1$, we have:
\begin{align}
    \mathbb{E}[\mathcal{G}_{\bar{t}_s}] - \mathbb{E}[\mathcal{G}_{\bar{t}_{s - 1}}] & \leq - \sum_{t = \bar{t}_{s-1}}^{\bar{t}_s-1}\frac{\eta\alpha_t}{2} \mathbb{E} \bigg[ \|\nabla h(\bar{x}_t) \|^2 \bigg] + \bigg(\frac{c_{\omega}^2\sigma^2}{16L^2} +  \frac{c_{\nu}^2\sigma^2}{80L_h^2} + \frac{c_{\nu}^2G^2}{40L_h^2}\bigg) \sum_{t=\bar{t}_{s-1}}^{\bar{t}_s-1}\alpha_{t}^3 \nonumber \\ & \qquad +  \bigg(\frac{3c_{\nu}^2\sigma^2}{32L_h} + \frac{3c_{\nu}^2G^2}{2L_h} +  \frac{ 3c_{\nu}^2\zeta^2}{16L_h}\bigg)\sum_{t=\bar{t}_{s-1}}^{\bar{t}_s-1}\alpha_{t}^3  \nonumber\\
    & \qquad  -\frac{1}{4}\gamma^2\sum_{t=\bar{t}_{s-1}}^{\bar{t}_s-1}\alpha_{t}F_{t} -\frac{\eta}{32} \sum_{t=\bar{t}_{s-1}}^{\bar{t}_s-1}\alpha_{t}E_{t}   - \frac{\mu\gamma}{16}\sum_{t=\bar{t}_{s-1}}^{\bar{t}_s-1} \alpha_{t}B_{t}\nonumber \\
    & \leq - \sum_{t = \bar{t}_{s-1}}^{\bar{t}_s-1}\frac{\eta\alpha_t}{2} \mathbb{E} \bigg[ \|\nabla h(\bar{x}_t) \|^2 \bigg] + \bigg(\frac{c_{\omega}^2\sigma^2}{16L^2} +  \frac{c_{\nu}^2\sigma^2}{5L_h^2} + \frac{2c_{\nu}^2G^2}{L_h^2} + \frac{ 3c_{\nu}^2\zeta^2}{16L_h^2}\bigg) \sum_{t=\bar{t}_{s-1}}^{\bar{t}_s-1}\alpha_{t}^3\nonumber\\
\label{eq:phi_bound_overall}
\end{align}
which completes the proof.

\begin{theorem}
\label{theorem:FedBiOAcc}
Suppose $\frac{1}{\gamma} > max(15L, 1)$, $\frac{1}{\eta} > max(\frac{600L^3}{\mu^3}  + \frac{1}{15L} + 1, \frac{12(1 + \rho^2)}{I^2} + \frac{97}{256} + \frac{1}{120L} + \frac{16I^2}{\mu}, \frac{960L_h^2}{\mu\gamma}, \frac{1}{\gamma}, \frac{40L_h^2}{\mu^3\gamma})$, $c_{\nu} = 320L_h^2 + \frac{\sigma^2 }{24 \delta^3 L_hI}$, $c_{\omega} = \frac{160L^2}{\mu} + \frac{\sigma^2}{24 \delta^3 L_hI}$, $u = max(2\sigma^2, \delta^3, c_{\nu}^{3/2}\delta^3, c_{\omega}^{3/2}\delta^3, 16^3I^3)$, $\delta = \frac{\sigma^{2/3}}{16IL_h}$, then we have:
\begin{align*}
    \frac{1}{T}\sum_{t = 1}^{T-1} \mathbb{E} \bigg[ \|\nabla h(\bar{x}_t) \|^2 \bigg]  & \leq \bigg(\frac{2(h(\bar{x}_{1}) - h^{\ast})}{\eta} + \frac{2 \sigma^2u^{1/3}}{\hat{c}_{\nu}\delta} + \frac{2 \sigma^{8/3}}{\hat{c}_{\nu}\delta} + \frac{2C_{y,1}^2}{\eta} + \frac{2\gamma u^{1/3} \sigma^2}{\eta\hat{c}_{\omega}\delta} + \frac{2\gamma \sigma^{8/3}}{\eta\hat{c}_{\omega}\delta} \nonumber\\& \qquad + \bigg(\frac{c_{\omega}^2\sigma^2}{16L^2} +  \frac{c_{\nu}^2\sigma^2}{5L_h^2} + \frac{2c_{\nu}^2G^2}{L_h^2} + \frac{ 3c_{\nu}^2\zeta^2}{16L_h^2}\bigg)\frac{ 2\delta^3\ln(T) }{\eta\sigma^2}\bigg)\bigg(\frac{u^{1/3}}{\delta T} + \frac{\sigma^{2/3}}{\delta T^{2/3}}\bigg)  \nonumber\\
\end{align*}
where the expectation is w.r.t the stochasticity of the algorithm.
\end{theorem}
First, based on Lemma~\ref{lemma:potential_descent}, we have:
\begin{align*}
    \mathbb{E}[\mathcal{G}_{\bar{t}_s}] - \mathbb{E}[\mathcal{G}_{\bar{t}_{s - 1}}] 
    & \leq - \sum_{t = \bar{t}_{s-1}}^{\bar{t}_s-1}\frac{\eta\alpha_t}{2} \mathbb{E} \bigg[ \|\nabla h(\bar{x}_t) \|^2 \bigg] + \bigg(\frac{c_{\omega}^2\sigma^2}{16L^2} +  \frac{c_{\nu}^2\sigma^2}{5L_h^2} + \frac{2c_{\nu}^2G^2}{L_h^2} + \frac{ 3c_{\nu}^2\zeta^2}{16L_h^2}\bigg) \sum_{t=\bar{t}_{s-1}}^{\bar{t}_s-1}\alpha_{t}^3\nonumber\\
\end{align*}
Next we sum for all $s \in [S]$ and assume $T = SI + 1$, we have:
\begin{align*}
    \mathbb{E}[\mathcal{G}_{T}] - \mathbb{E}[\mathcal{G}_{1}] 
    & \leq - \sum_{t = 1}^{T-1}\frac{\eta\alpha_t}{2} \mathbb{E} \bigg[ \|\nabla h(\bar{x}_t) \|^2 \bigg] + \bigg(\frac{c_{\omega}^2\sigma^2}{16L^2} +  \frac{9c_{\nu}^2\sigma^2}{L_h^2} + \frac{ 3c_{\nu}^2\zeta^2}{16L_h^2}\bigg) \sum_{t=1}^{T-1}\alpha_{t}^3\nonumber\\
\end{align*}
So we have:
\begin{align*}
    \sum_{t = 1}^{T-1}\frac{\eta\alpha_t}{2} \mathbb{E} \bigg[ \|\nabla h(\bar{x}_t) \|^2 \bigg] 
    & \leq \mathbb{E}[\mathcal{G}_{1}] - \mathbb{E}[\mathcal{G}_{T}]  + \bigg(\frac{c_{\omega}^2\sigma^2}{16L^2} +  \frac{c_{\nu}^2\sigma^2}{5L_h^2} + \frac{2c_{\nu}^2G^2}{L_h^2} + \frac{ 3c_{\nu}^2\zeta^2}{16L_h^2}\bigg) \sum_{t=1}^{T-1}\alpha_{t}^3\nonumber\\
    & \leq h(\bar{x}_{1}) - h^{\ast} + \frac{\eta A_1}{\hat{c}_{\nu}\alpha_{1}} + B_1 + \frac{\gamma C_1}{\hat{c}_{\omega}\alpha_{1}} + \bigg(\frac{c_{\omega}^2\sigma^2}{16L^2} +  \frac{c_{\nu}^2\sigma^2}{5L_h^2} + \frac{2c_{\nu}^2G^2}{L_h^2} + \frac{ 3c_{\nu}^2\zeta^2}{16L_h^2}\bigg) \sum_{t=1}^{T-1}\alpha_{t}^3\nonumber\\
\end{align*}
where we use $\mathcal{G}_T \ge h^{\ast})$ and $h^{\ast}$ denotes the optimal value of $h$. Then for the last term:
\begin{align}
     \sum_{t=1}^T \alpha_t^3 & =    \sum_{t = 1}^{T} \frac{\delta^3 }{u + \sigma^2 t} \leq  \sum_{t = 1}^{T} \frac{\delta^3  }{\sigma^2 + \sigma^2 t} = \frac{ \delta^3}{\sigma^2}   \sum_{t = 1}^{T} \frac{1}{1 +   t} \leq \frac{  \delta^3 }{\sigma^2}   \ln(T+1).
     \label{Eq: Sum_OverT_LastTerm}
\end{align}
where the first inequality follows $u_t \ge 2\sigma^2 > \sigma^2$, the last inequality follows Proposition~\ref{Lem: AD_Sum_1overT}. Next use the fact that $\alpha_t$ is non-increasing, we have:
\begin{align*}
    \frac{\eta\alpha_T}{2}\sum_{t = 1}^{T-1} \mathbb{E} \bigg[ \|\nabla h(\bar{x}_t) \|^2 \bigg] & \leq h(\bar{x}_{1}) - h^{\ast} + \frac{\eta A_1}{\hat{c}_{\nu}\alpha_{1}} + B_1 + \frac{\gamma C_1}{\hat{c}_{\omega}\alpha_{1}} + \bigg(\frac{c_{\omega}^2\sigma^2}{16L^2} +  \frac{c_{\nu}^2\sigma^2}{5L_h^2} + \frac{2c_{\nu}^2G^2}{L_h^2} + \frac{ 3c_{\nu}^2\zeta^2}{16L_h^2}\bigg)\frac{  \delta^3 }{\sigma^2}   \ln(T)\nonumber\\
\end{align*}
Divide both sides by $2T/\eta\alpha_T$, we have:
\begin{align*}
    \frac{1}{T}\sum_{t = 1}^{T-1} \mathbb{E} \bigg[ \|\nabla h(\bar{x}_t) \|^2 \bigg] & \leq \frac{2(h(\bar{x}_{1}) - h^{\ast})}{\eta\alpha_T T} + \frac{2 A_1}{\hat{c}_{\nu}\alpha_{1}\alpha_T T} + \frac{2B_1}{\eta\alpha_T T} + \frac{2\gamma C_1}{\eta\hat{c}_{\omega}\alpha_{1}\alpha_T T} + \bigg(\frac{c_{\omega}^2\sigma^2}{16L^2} +  \frac{9c_{\nu}^2\sigma^2}{L_h^2} + \frac{ 3c_{\nu}^2\zeta^2}{16L_h^2}\bigg) \frac{ 2\delta^3\ln(T) }{\eta\sigma^2\alpha_T T}   \nonumber\\
\end{align*}
Next we have 
\begin{align*}
A_1 = \mathbb{E} \bigg[\Big\| \bar{\nu}_{1} - \frac{1}{M} \sum_{m=1}^M  \nabla h(x^{(m)}_1)  \Big\|^2\bigg] = \mathbb{E} \bigg[\Big\| \frac{1}{M} \sum_{m=1}^M \bigg(\mathcal{G}^{(m)} (x^{(m)}_{1}, y^{(m)}_{1}; \mathcal{B}_{x}) -  \nabla h(x^{(m)}_1) \bigg) \Big\|^2\bigg] \leq \sigma^2
\end{align*}
and $B_1 = \frac{1}{M} \sum_{m=1}^M \|y^{(m)}_1 - y^{(m)}_{x^{(m)}_{1}} \|^2 \leq C_{y,1}^2$,  $C_1 = \mathbb{E} [\frac{1}{M} \sum_{m=1}^M \|\omega^{(m)}_1 - \nabla_y g^{(m)}(x^{(m)}_{1}, y^{(m)}_{1} )\|^2] \leq \sigma^2$, 
Then we have:
\begin{align*}
    \frac{1}{T}\sum_{t = 1}^{T-1} \mathbb{E} \bigg[ \|\nabla h(\bar{x}_t) \|^2 \bigg] & \leq \frac{2(h(\bar{x}_{1}) - h^{\ast})}{\eta\alpha_T T} + \frac{2 \sigma^2}{\hat{c}_{\nu}\alpha_{1}\alpha_T T} + \frac{2C_{y,1}^2}{\eta\alpha_T T} + \frac{2\gamma \sigma^2}{\eta\hat{c}_{\omega}\alpha_{1}\alpha_T T} \nonumber \\
    & \qquad+ \bigg(\frac{c_{\omega}^2\sigma^2}{16L^2} +  \frac{c_{\nu}^2\sigma^2}{5L_h^2} + \frac{2c_{\nu}^2G^2}{L_h^2} + \frac{ 3c_{\nu}^2\zeta^2}{16L_h^2}\bigg)\frac{ 2\delta^3\ln(T) }{\eta\sigma^2\alpha_T T}   \nonumber\\
\end{align*}
Note that we have:
\begin{align*}
    \frac{1}{{\alpha_t t}} = \frac{(u + \sigma^2t)^{1/3}}{\delta t} \leq \frac{u^{1/3}}{\delta t} + \frac{\sigma^{2/3}}{\delta t^{2/3}}
\end{align*}
where the inequality uses the fact that $(x + y)^{1/3} \leq x^{1/3} + y^{1/3}$. 
Consider $t=1$ and $t=T$ and we have:
\begin{align*}
    \frac{1}{T}\sum_{t = 1}^{T-1} \mathbb{E} \bigg[ \|\nabla h(\bar{x}_t) \|^2 \bigg]  & \leq \bigg(\frac{2(h(\bar{x}_{1}) - h^{\ast})}{\eta} + \frac{2 \sigma^2u^{1/3}}{\hat{c}_{\nu}\delta} + \frac{2 \sigma^{8/3}}{\hat{c}_{\nu}\delta} + \frac{2C_{y,1}^2}{\eta} + \frac{2\gamma u^{1/3} \sigma^2}{\eta\hat{c}_{\omega}\delta} + \frac{2\gamma \sigma^{8/3}}{\eta\hat{c}_{\omega}\delta} \nonumber\\& \qquad + \bigg(\frac{c_{\omega}^2\sigma^2}{16L^2} +  \frac{c_{\nu}^2\sigma^2}{5L_h^2} + \frac{2c_{\nu}^2G^2}{L_h^2} + \frac{ 3c_{\nu}^2\zeta^2}{16L_h^2}\bigg)\frac{ 2\delta^3\ln(T) }{\eta\sigma^2}\bigg)\bigg(\frac{u^{1/3}}{\delta T} + \frac{\sigma^{2/3}}{\delta T^{2/3}}\bigg)  \nonumber\\
\end{align*}
which completes the proof.

\section{More Experimental Detials}
\label{more-experiments}
Firstly, we use the well known Equal Opportunity as our main fairness metrics in experiments. It is defined as the $max_{z\in[K]} \|\mathbb{P}(\hat{y}| y=1, a=z) - \mathbb{P}(\hat{y}| y=1, a=z)\|$, where $\hat{y}$ is is predication made by the model, $\mathbb{P}$ is the probability notation. Next for the hyper-parameters, we use $I=5$ for FedAvg, FedReg and our algorithms, and $I=1$ for FedMinMax and FCFL. For learning rates, we search over the range of $\{0.001, 0.01, 0.1, 1\}$, most algorithms get best performance at 0.1 or 1. Then for special hyper-parameters of each algorithms: For FedReg the regularization coefficient is set as 0.1; For FedMinMax, the stepsize for the group weights are 0.1; For FCFL, we use hyper-parameters provided from the original paper~\cite{cui2021addressing}. For our FedBiO, we set the outer learning rate $\eta$ as 0.1; For our FedBiOAcc, we set $\delta$ as 0.1, $u$ as 1 and $c_{\nu}$ as 1, $c_{\omega}$ as 1. The results reported in Table~\ref{tb:1} are run for 2000 steps with batchsize 128 for the Adult Dataset and 32 for the Credit Dataset.

The two datasets we used in experiments are widely used benchmarks for fair machine learning. More specifically, the Adult dataset aims to predict income level based on around 200 features, and the race/gender are sensitive groups. We use the race as the sensitive attribute in experiments, which include 5 different racial groups. Next the (German) Credit dataset aims to predict good and bad credit risks. the gender and marital status are sensitive attributes.

\end{appendices}

\end{onecolumn}

\end{document}